\documentclass{article}

\usepackage{arxiv}

\usepackage{microtype}
\usepackage{graphicx}
\usepackage{subfigure}
\usepackage{booktabs} 

\usepackage[utf8]{inputenc}
\usepackage[T1]{fontenc}
\usepackage[colorlinks=true,linkcolor=blue, citecolor=blue]{hyperref}
\usepackage{url} 
\usepackage{booktabs}
\usepackage{amsfonts} 
\usepackage[export]{adjustbox}
\usepackage{amsmath}
\usepackage{amsthm}
\usepackage{nicefrac} 
\usepackage{microtype} 
\usepackage[flushleft]{threeparttable}
\usepackage{array, multirow}
\usepackage{placeins}
\usepackage{mathtools}
\usepackage{algorithm, algorithmic}
\usepackage{enumerate}
\usepackage[square, sort]{natbib}
\usepackage{bm}

\setcitestyle{numbers,square,citesep={,},aysep={,},yysep={,}}

\begin{document}

\title{Solving NP-Hard Problems on Graphs with Extended AlphaGo Zero}

\makeatletter
\newcommand{\printfnsymbol}[1]{%
  \textsuperscript{\@fnsymbol{#1}}%
}
\makeatother

\author{
    Kenshin Abe \textsuperscript{\rm{1}} \thanks{equal contribution}, 
    Zijian Xu \textsuperscript{\rm{1}} \printfnsymbol{1},
    Issei Sato \textsuperscript{\rm{1, 2}},
    Masashi Sugiyama \textsuperscript{\rm{2, 1}} \\
    \textsuperscript{\rm{1}} The University of Tokyo, \textsuperscript{\rm{2}} RIKEN \\
    \texttt{abe@ms.k.u-tokyo.ac.jp}, \texttt{xuzijian@ms.k.u-tokyo.ac.jp}, \\
    \texttt{sato@k.u-tokyo.ac.jp}, \texttt{sugi@k.u-tokyo.ac.jp} \\
}

\date{}
\maketitle

\begin{abstract}
There have been increasing challenges to solve combinatorial optimization problems by machine learning.
\citet{S2V-DQN} proposed an end-to-end reinforcement learning framework, S2V-DQN, which automatically learns graph embeddings to construct solutions to a wide range of problems.
To improve the generalization ability of their Q-learning method, we propose a novel learning strategy based on AlphaGo Zero \cite{alphago-zero} which is a Go engine that achieved a superhuman level without the domain knowledge of the game.
Our framework is redesigned for combinatorial problems, where the final reward might take any real number instead of a binary response, win/lose.
In experiments conducted for five kinds of NP-hard problems including {\sc MinimumVertexCover} and {\sc MaxCut}, our method is shown to generalize better to various graphs than S2V-DQN.
Furthermore, our method can be combined with recently-developed graph neural network (GNN) models such as the \emph{Graph Isomorphism Network} \cite{GIN}, resulting in even better performance.
This experiment also gives an interesting insight into a suitable choice of GNN models for each task.
\end{abstract}

\section{Introduction}
There is no polynomial-time algorithm found for NP-hard problems \cite{Cook71}, but they often arise in many real-world optimization tasks.
Therefore, a variety of algorithms have been developed, including approximation algorithms \cite{fvs-2opt, maxcut-878}, meta-heuristics based on local searches such as simulated annealing and evolutionary computation \cite{MetaHeuristics, maxcut-heuristic}, general-purpose exact solvers such as CPLEX \footnote{www.cplex.com} and Gurobi \cite{gurobi}, and problem-specific exact solvers.
Some problem-specific exact solvers are fast and practical for certain well-studied combinatorial problems such as {\sc MinimumVertexCover} \citep{iwata, mis-sugoi}.
These algorithms can handle sparse graphs of millions of nodes.
However, they usually require problem-specific highly sophisticated reduction rules and heavy implementation and are hard to generalize for other combinatorial problems. Moreover, real-world NP-hard problems are usually much more complex than these problems and thus it is difficult to construct ad-hoc efficient algorithms for them.

Recently, machine learning approaches have been actively investigated to solve combinatorial optimization, with the expectation that the combinatorial structure of the problem can be automatically learned without complicated hand-crafted heuristics.
In the early stage, many of these approaches focused on solving specific problems \cite{NeuralComputation, NeuralCombinatorial} such as the traveling salesperson problem ({\sc TSP}).
Recently, \citet{S2V-DQN} proposed a framework to solve combinatorial problems by a combination of reinforcement learning and graph embedding, which attracted attention for the following two reasons:
It does not require any knowledge on graph algorithms other than greedy selection based on network outputs.
Furthermore, it learns algorithms without any training dataset.
Thanks to these advantages, the framework can be applied to a diverse range of problems over graphs and it also performs much better than previous learning-based approaches.
However, we observed poor empirical performance on some graphs having different characteristics (e.g., synthetic graphs and real-world graphs) than random graphs that were used for training, possibly because of the limited exploration space of their Q-learning method.

In this paper, to overcome this weakness, we propose to replace Q-learning with our novel learning strategy named CombOpt Zero.
CombOpt Zero is inspired by AlphaGo Zero \citep{alphago-fan}, a superhuman engine of Go, which conducts Monte Carlo Tree Search (MCTS) to train deep neural networks.
AlphaGo Zero was later generalized to AlphaZero \cite{alphazero} so that it can handle other games; however, its range of applications is limited to two-player games whose state is win/lose (or possibly draw).
We extend AlphaGo Zero to a bunch of combinatorial problems by a simple normalization technique based on random sampling and show that it successfully works in experiments.
More specifically, we train our networks for five kinds of NP-hard tasks and test on different instances including standard random graphs (e.g., the Erd\H{o}s-Renyi model \cite{ER-model} and the Barab\'{a}si-Albert model \cite{BA-model}) and many real-world graphs.
We show that CombOpt Zero has a better generalization to a variety of graphs than the existing method, which indicates the MCTS based training strengthens the exploration of various actions.
Also, using the MCTS at test time significantly improves performance for a certain problem and guarantees the solution quality.
Furthermore, we propose to combine our framework with several Graph Neural
Network models \cite{GCNsKipf, GIN, PGNN}, and experimentally demonstrate that an appropriate choice of models contributes to improving the performance with a significant margin.

\section{Preliminary}
In this section, we introduce the main ingredients which our work is based on.

\subsection{Notations}
In this paper, we use $G = (V, E)$ to denote an undirected and unlabelled graph, where $V$ is the set of vertices and $E$ is the set of edges. Since $v$ is used to represent a state value in AlphaGo Zero, we often use $x, y, ...$ to denote a node of a graph. $V(G)$ indicates the set of vertices of graph $G$. $\mathcal{N}(x)$ means the set of $1$-hop neighbors of node $x$ and $\mathcal{N}(S) = \bigcup_{x \in S}\mathcal{N}(x)$. We usually use $a, b, \ldots$ for actions.
Bold variables such as $\bm{p}$ and $\bm{\pi}$ are used for vectors.

\newcommand{\argmax}{\mathop{\rm argmax}\limits}
\newcommand{\argmin}{\mathop{\rm argmin}\limits}

\subsection{Machine Learning for Combinatorial Optimization}
Machine learning approaches for combinatorial optimization problems have been studied in the literature, starting from \citet{NeuralComputation}, who applied a variant of neural networks to small instances of Traveling Salesperson Problem ({\sc TSP}).
With the success of deep learning, more and more studies were conducted including \citet{NeuralCombinatorial, AttentionTSP} for {\sc TSP} and \citet{satnet} for {\sc MaxSat}.

\citet{S2V-DQN} proposed an end-to-end reinforcement learning framework S2V-DQN, which attracted attention because of promising results in a wide range of problems over graphs such as {\sc MinimumVertexCover} and {\sc MaxCut}.
It optimizes a deep Q-network (DQN) where the Q-function is approximated by a graph embedding network, called \texttt{structure2vec} (S2V) \cite{structure2vec}.
The DQN is based on their reinforcement learning formulation, where each action is picking up a node and each state represents the ``sequence of actions''.
In each step, a partial solution $S \subset V$, i.e., the current state, is expanded by the selected vertex $v^* = \argmax_{v \in V(h(S))}Q(h(S), v)$ to $(S, v^*)$, where $h(\cdot)$ is a fixed function determined by the problem that maps a state to a certain graph, so that the selection of $v$ will not violate the problem constraint.
For example, in {\sc MaximumIndependentSet}, $h(S)$ corresponds to the induced subgraph of the input graph $G = (V, E)$, where vertices are limited to $V \backslash (S \cup \mathcal{N}(S))$.
The immediate reward is the change in the objective function.
The Q-network, i.e, S2V learns a fixed dimensional embedding for each node.

In this work, we mitigate the issue of S2V-DQN's generalization ability.
We follow the idea of their reinforcement learning setting, with a different formulation, and replace their Q-learning by a novel learning strategy inspired by AlphaGo Zero.
Note that although some studies combine classic heuristic algorithms and learning-based approaches (using dataset) to achieve the state-of-the-art performance \cite{zhuwen}, we stick to learning without domain knowledge and dataset in the same way as S2V-DQN.

\subsection{AlphaGo Zero} \label{prelim-alphago-zero}
AlphaGo Zero \citep{alphago-zero} is a well-known superhuman engine designed for use with the game of Go.
It trains a deep neural network $f_\theta$ with parameter $\theta$ by reinforcement learning.
Given a state (game board), the network outputs $f_\theta(s) = (\bm{p}, v)$, where $\bm{p}$ is the probability vector of each move and $v \in [-1, 1]$ is a scalar denoting the value of the state.
If $v$ is close to $1$, the player who takes a corresponding action from state $s$ is very likely to win.

The fundamental idea of AlphaGo Zero is to enhance its own networks by self-play.
For this self-play, a special version of Monte Carlo Tree Search (MCTS) \cite{MCTS-first}, which we describe later, is used.
The network is trained in such a way that the policy imitates the enhanced policy by MCTS $\bm{\pi}$, and the value imitates the actual reward from self play $z$ (i.e. $z = 1$ if the player wins and $z = -1$ otherwise).
More formally, it learns to minimize the loss
\begin{equation}
\mathcal{L} = (z - v)^2 + \mathrm{CrossEntropy}(\bm{p}, \bm{\pi}) + c_{\mathrm{reg}}\|\theta\|_2^2, \label{alphagoloss}
\end{equation}
where $c_{\mathrm{reg}}$ is a nonnegative constant for $L_2$ regularization.

MCTS is a heuristic search for huge tree-structured data.
In AlphaGo Zero, the search tree is a rooted tree, where each node corresponds to a state and the root is the initial state.
Each edge $(s, a)$ denotes action $a$ at state $s$ and stores a tuple
\begin{equation} \label{alphago-zero-edge-store}
    (N(s, a), W(s, a), Q(s, a), P(s, a)),
\end{equation}
where $N(s, a)$ is the visit count, $W(s, a)$ and $Q(s, a)$ are the total and mean action value respectively, and $P(s, a)$ is the prior probability.
One iteration of MCTS consists of three parts: \emph{select}, \emph{expand}, and \emph{backup}.
First, from the root node, we keep choosing an action that maximizes an upper confidence bound
\begin{equation} \label{alphagoucb}
    Q(s, a) + c_{\mathrm{puct}} P(s, a)\frac{\sqrt{\sum_{a'}N(s, a')}}{1 + N(s, a)},
\end{equation}
where $c_{\mathrm{puct}}$ is a non-negative constant (\emph{select}).
Once it reaches to unexplored node $s$, then the edge values (\ref{alphago-zero-edge-store}) are initialized using the network prediction $(\bm{p}, v) = f_{\theta}(s)$ (\emph{expand}).
After expanding a new node, each visited edge is traversed and its edge values are updated (\emph{backup}) so that $Q$ maintain the mean of state evaluations over simulations:
\begin{equation} \label{alphago-Q-update}
    Q(s,a) = \frac{1}{N(s,a)} \sum_{s' \mid s,a \rightarrow s'} v_{s'},
\end{equation}
where the sum is taken over those states reached from $s$ after taking action $a$.
After some iterations, the probability vector $\pi$ is calculated by
\begin{equation} \label{alphago-zero-pi}
    \bm{\pi}_a = \frac{N(s_0, a)^{1 / \tau}}{\sum_b N(s_0, b)^{1 / \tau}}
\end{equation}
for each $a \in A_{s_0}$, where $\tau$ is a temperature parameter.

AlphaGo Zero defeated its previous engine AlphaGo \citep{alphago-fan} with $100$-$0$ score without human knowledge, i.e., the records of the games of professional players and some known techniques in the history of Go.
We are motivated to take advantage of the AlphaGo Zero technique in our problem setting since we also aim at training deep neural networks for discrete optimization without domain knowledge.
However, we cannot directly apply AlphaGo Zero, which was designed for two-player games, to combinatorial optimization problems.
Section \ref{ext} and \ref{algo} explains how we manage to resolve this issue.

\subsection{Graph Neural Network} \label{GNN}
A Graph Neural Network (GNN) is a neural network that takes graphs as input.
\citet{GCNsKipf} proposed the \emph{Graph Convolutional Network} (GCN) inspired by spectral graph convolutions.
Because of its scalability, many variants of spatial based GNN were proposed.
Many of them can be described as a Message Passing Neural Network (MPNN) \cite{MpNNs}.
They recursively aggregate neighboring feature vectors to obtain node embeddings that capture the structural information of the input graph.
The \emph{Graph Isomorphism Network} (GIN) \cite{GIN} is one of the most expressive MPNNs in terms of graph isomorphism.
Although they have a good empirical performance, some studies point out the limitation of the representation power of MPNNs \cite{GIN, WLNeural}.

\citet{InvariantEquivariant} proposed an Invariant Graph Network (IGN) using tensor representations of a graph and was shown to be universal \cite{InvarintUniversality, KerivenUniversality}.
Since it requires a high-order tensor in middle layers, which is impractical, \citet{PGNN} proposed \emph{2-IGN+}, a scalable and powerful model. 

All of these models, as well as S2V \cite{structure2vec} used in S2V-DQN, are compared in the experiments to test the difference in the performance for combinatorial optimization.
The detail of each model is described in Appendix \ref{detail-GNN}.

\subsection{NP-hard Problems on Graphs} \label{NPhard-problems}
Here, we introduce three problems out of the five NP-hard problems we conducted experiments on.
The other two problems are described in Appendix \ref{Other-NPhard}.

\paragraph{Minimum Vertex Cover}
A subset of nodes $V' \subset V$ is called a vertex cover if all edges are covered by $V'$;
for all $(x, y) \in E$, $x \in V'$ or $y \in V'$ holds.
{\sc MinimumVertexCover} asks a vertex cover whose size is minimum.

\paragraph{Max Cut}
Let $C \subset E$ be a cut set of between $V' \subset V$ and $V \backslash V'$, i.e.,
$C = \{(u, v) \in E \mid u \in V', v \in V \backslash V'\}$.
{\sc MaxCut} asks for a subset $V'$ that maximizes the size of cut set $C$.

\paragraph{Maximum Clique}
A subset of nodes $V' \subset V$ is called a clique if any two nodes in $V'$ are adjacent in the original graph;
for all $x, y \in V'$ ($x \neq y$), $(x, y) \in E$.
{\sc MaximumClique} asks for a largest clique.

Note that for all problems in the experiments, we are focusing on the unweighted case.
In Section \ref{MDP}, we show that all of these problems can be formulated in our framework.

\section{Method}
In this section, we give a detailed explanation of our algorithm to solve combinatorial optimization problems over graphs.
First, we introduce a reinforcement learning formulation of NP-hard problems over graphs and show that the problems introduced in Section \ref{NPhard-problems} can be accommodated under this formulation.
Next, we explain the basic ideas of CombOpt Zero in light of the difference between 2-player games and combinatorial optimization.
Finally, we propose the whole algorithm along with pseudocodes.

\subsection{Reduction to MDP}\label{MDP}
To apply the AlphaGo Zero method to combinatorial optimization, we reduce graph problems into a Markov Decision Process (MDP) \cite{MDP}.
A deterministic MDP is defined as
$(S, A_s, T, R)$, where $S$ is a set of states, $A_s$ is a set of actions from the state $s \in S$, $T: S \times A_s \rightarrow S$ is a deterministic state transition function, and $R: S \times A_s \rightarrow \mathbb{R}$ is an immediate reward function.
In our problem setting on graphs, each state $s \in S$ is represented as a \emph{labeled} graph, a tuple $s = (G, d)$.
$G = (V, E)$ is a graph and $d : V \rightarrow L$ is a node-labeling function, where $L$ is a label space.

For each problem, we have a set of terminal states $S_{\mathrm{end}}$.
Given a state $s$, we repeat selecting an action $a$ from $A_s$ and transiting to the next state $T(s, a)$, until $s \in S_{\mathrm{end}}$ holds.
By this process, we get a sequence of states and actions $[s_0, a_0, s_1, a_1, ..., a_{N-1}, s_N]$ of $N$ steps where $s_N \in S_{\mathrm{end}}$, which we call a trajectory.
For this trajectory, we can calculate the simple sum of immediate rewards \(\sum_{n = 0}^{N - 1} R(s_n, a_n)\), and we define $r^* (s)$ be the maximum possible sum of immediate rewards out of all trajectories from state $s$.
Let $\mathrm{Init}$ be a function that maps the input graph to the initial state.
Our goal is to, given graph $G_0$, obtain the maximum sum of rewards $r^* (\mathrm{Init}(G_0))$.
 
Below, we show that the problems described in Section \ref{NPhard-problems} can be naturally accommodated in this framework, by defining $L$, $A_s$, $T$, $R$, $\mathrm{Init}$, and $S_{\mathrm{end}}$ as follows:

\paragraph{Minimum Vertex Cover}
Since we do not need a label of the graph for this problem, we set $d$ to a constant function on any set $L$ (e.g., $L = \mathbb{R}$, $d(s) = 1$).
Actions are represented by selecting one node ($A_s = V$).
$\mathrm{Init}$ uses the same graph as $G_0$ and $d$ defined above.
$T(s, x)$ returns the next state, corresponding to the graph where edges covered by $x$ and isolated nodes are deleted.
$S_{\mathrm{end}}$ is the states with the empty graphs.
$R(s, x) = -1$ for all $s$ and $x$ because we want to minimize the number of transition steps in the MDP.

\paragraph{Max Cut}
In each action, we color a node by $0$ or $1$, and remove it while each node keeps track of how many adjacent nodes have been colored with each color.
$A_s = \{(x, c) \mid x \in V, c \in \{1, 2\}\}$ denotes a set of possible coloring of a node, where $(x, c)$ means coloring node $x$ with color $c$.
$L = \mathbb{N}^2$, representing the number of colored (and removed) nodes in each color (i.e., $l_0$ is the number of (previously) adjacent nodes of $x$ colored with $0$, and same for $l_1$, where $l = d(x)$).
$\mathrm{Init}$ uses the same graph as $G_0$ and sets both of $d(x)$ be $0$ for all $x \in V$.
$T(s, (x,c))$ increases the $c$-th value of $d(x')$ by one for $x' \in N(x)$ and removes $x$ and neighboring edges from the graph.
$S_{\mathrm{end}}$ is the states with the empty graphs.
$R(s, (x,c))$ is the $(3-c)$-th (i.e., $2$ if $c=1$ and $1$ if $c=2$) value of $d(x)$, meaning the number of edges in the original graph which have turned out to be included in the cut set (i.e., colors of the two nodes are different).

\paragraph{Maximum Clique}
$d$, $A_s$, $\mathrm{Init}$, and $S_{\mathrm{end}}$ are the same as {\sc MinimumVertexCover}.
$T(s, x)$ return the next state whose corresponding graphs is the induced subgraph of $\mathcal{N}(x)$; 1-hop neighbors of $x$.
$R(s, x) = 1$ because we want to maximize the number of transition steps of the MDP.

\begin{figure}[t]
    \centering
    \includegraphics[width=13cm]{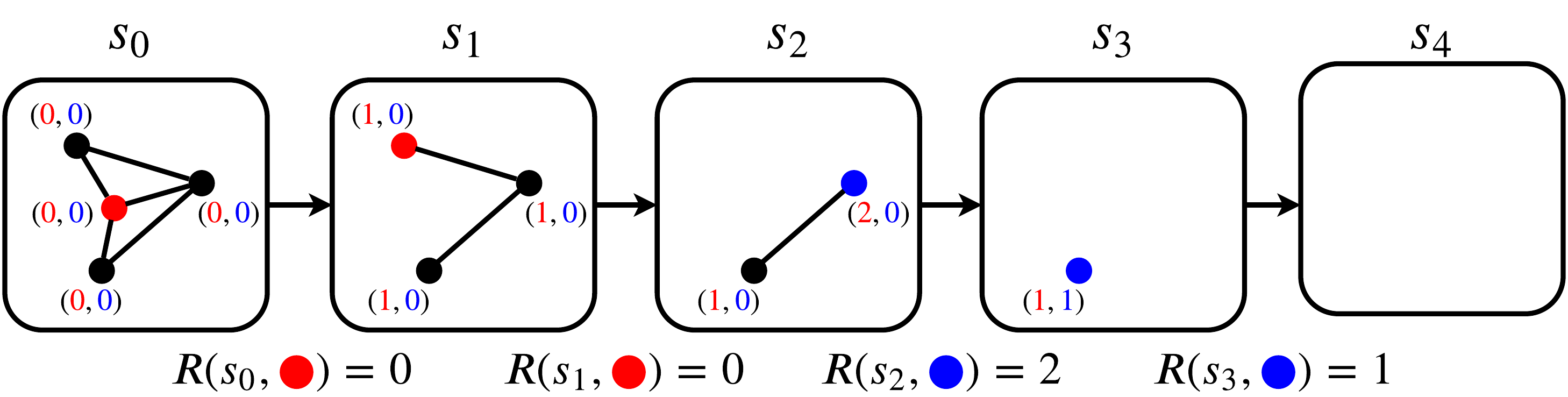}
    \caption{\textbf{Example of {\sc MaxCut} MDP trainsitions.} In the transition sequence, the upper left and center node are colored in {\color{red} $\bullet$}, the right and lower left node are colored in {\color{blue} $\bullet$}, resulting in the cumulative reward $3$.}
    \label{fig:maxcut-mdp}
\end{figure}

Figure \ref{fig:maxcut-mdp} shows some possible MDP transitions for {\sc MaxCut}.
It is easy to check that, for all reduction rules described above, finding a sequence that maximizes the sum of rewards is equivalent to solving the original problems.
One of the differences from the formulation of \citet{S2V-DQN} is that we do not limit the action space to a set of nodes.
This enables more flexibility and results in different reduction rules in some problems.
For example, in our formulation of {\sc MaxCut}, actions represent \emph{a node coloring}.
See Appendix \ref{compare-reduction-rules} for more examples of the differences.

\subsection{Ideas for CombOpt Zero} \label{ext}
Here, we explain the basic idea behind our extension of AlphaGo Zero into combinatorial optimization problems.

\paragraph{Combinatorial Optimization vs. Go}
When we try to apply the AlphaGo Zero method to our MDP formulation of combinatorial optimization, we need to consider the following two differences from Go.
First, in our problem setting, the states are represented by (labeled) graphs of different sizes, while the boards of Go can be represented by $19 \times 19$ fixed-size matrix.
This can be addressed easily by adopting GNN models instead of convolutional neural networks in AlphaGo Zero, just like S2V-DQN.
The whole state $s = (G, d)$ can be given to GNN, where the coloring $d$ is used as the node feature.
The second difference is the possible answers for a state:
While the final result of Go is only win or lose, the answer to combinatorial problems can take any integers or even real numbers.
This makes it easy to model the state value of Go by ``how likely the player is going to win'', with the range of $[-1, 1]$, meaning that the larger the value is, the more likely the player is to win.
Imitating this, in our problem setting, one can evaluate the state value in such a way that it predicts the maximum possible sum of rewards from the current state.
However, since the answer might grow infinitely usually depending on the graph size, this naive approximation strategy makes it difficult to balance the scale in \eqref{alphagoloss} and \eqref{alphagoucb}.
For example, in \eqref{alphagoucb}, the first term could be too large when the answer is large, causing less focus on $P(s, a)$ and $N(s, a)$.
To mitigate this problem, we propose a reward normalization technique.

Given a state $s$, the network outputs $(\bm{p}, \bm{v}) = f_\theta(s)$.
While $\bm{p}$ is the same as AlphaGo Zero (action probabilities), $\bm{v}$ is a \emph{vector} instead of a scalar value.
$v_a$ predicts the \emph{normalized reward} of taking action $a$ from state $s$, meaning ``how good the reward is compared to the return obtained by random actions''.
Formally, we train the network so that $\bm{v_a}$ predicts 
   $(R(s, a) + r^* (T(s, a)) - \mu_s) / \sigma_s$,
where $\mu_s$ and $\sigma_s$ are the mean and the standard deviation of the cumulative rewards by random plays from state $s$.
When we estimate the state value of $s$, we consider taking action $a$ which maximizes $\bm{v}_a$ and restoring the unnormalized value by $\mu_s$ and $\sigma_s$:
\begin{equation} \label{restore-state-value}
    r_{\mathrm{estim}}(s) = 
    \begin{cases}
        0 & (s \in S_{\mathrm{end}})  \\
        \mu_s + \sigma_s \cdot (\max_{a \in A_s} \bm{v}_a) & (otherwise) .
    \end{cases}
\end{equation}
Similarly, we also let $W(s, a)$ and $Q(s, a)$ hold the sum and mean of \emph{normalized} action value estimations over the MCTS iterations.
Figure \ref{fig:backup} illustrates the idea of how we update these values efficiently in the \emph{backup} phase.
See Section \ref{algo} for the details.

\paragraph{Reward Normalization Technique}
\begin{figure}[t]
    \centering
    \includegraphics[width=6cm]{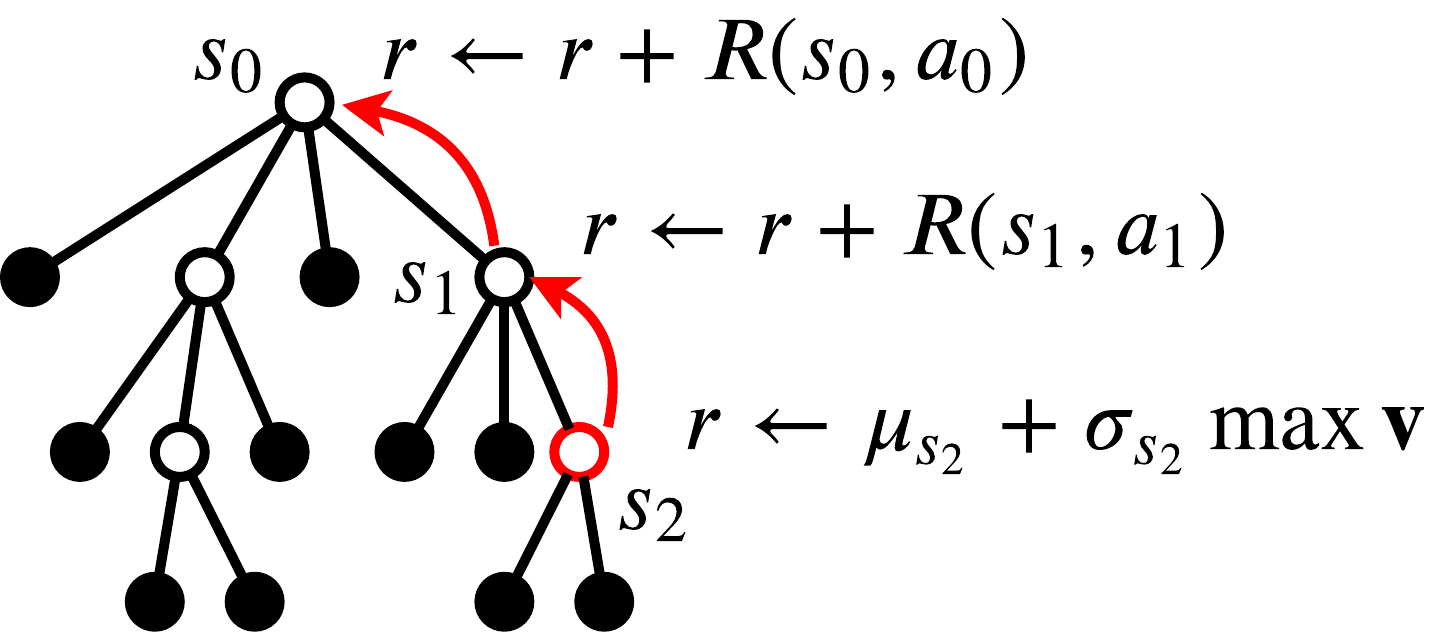}
    \caption{\textbf{Backup.} After estimating the reward of the expanded node, we iteratively update the rewards of its ancestors.}
    \label{fig:backup}
\end{figure}

By virtue of this normalization, we no longer need to care about the difference of scales in \eqref{alphagoloss} and \eqref{alphagoucb}. 
Moreover, it frees the algorithm from problem specification, that is, when we are to maximize some criterion, we always evaluate the action by ``how good it is compared to random actions''.

\subsection{Algorithms} \label{algo}
Based on the discussion so far, now we are ready to introduce the whole algorithm.

\paragraph{MCTS Tree Structure}
Same as AlphaGo Zero, each edge stores values in \eqref{alphago-zero-edge-store}.
Additionally, each node $s$ stores a tuple
$(\mu_s, \sigma_s)$,
the mean and the standard deviation of the results by random plays.

\paragraph{MCTS}
The pseudocode is available in Algorithm \ref{code-MCTS}.
Given an initial state $s_0$, we repeat the iterations of the three parts: \emph{select}, \emph{expand}, and \emph{backup}.
\emph{Select} is same as AlphaGo Zero (keep selecting a child node that maximizes \eqref{alphagoloss}).
Once we reach unexpanded node $s$, we \emph{expand} the node.
$(\bm{p}, \bm{v}) = f_\theta(s)$ is evaluated and each edge value is updated as $(N(s,a) = 0, W(s,a) = 0, Q(s,a) = 0, P(s,a) = \bm{p}_a)$ for $a \in A_s$.
At the same time, $(\mu_s, \sigma_s)$ is calculated by random sampling from $s$.
In \emph{backup}, for each node $s'$ and the corresponding action $a'$ in a backward pass, $N(s',a')$ is incremented by one in the same way as AlphaGo Zero.
The difference is that $Q(s',a')$ is updated to hold the estimation of the \emph{normalized} mean reward from $s'$.
We approximate the non-normalized state value by \eqref{restore-state-value} and calculate the estimated cumulative reward from $s'$ by adding the immediate reward each time we move back to the parent node (Figure \ref{fig:backup} illustrates this process).
$\bm{\pi}$ is calculated in the same way as AlphaGo Zero \eqref{alphago-zero-pi} after sufficient iterations.
While the number of the iterations is fixed in AlphaGo Zero, in CombOpt Zero, since the size of action space differs by states, we make it proportional to the number of actions: $c_{\mathrm{iter}} |A_s|$.
Following AlphaGo Zero, we add Dirichlet noise to the prior probabilities only for the initial state $s_0$ to explore the various initial actions.

\begin{algorithm}[t]
\centering
\small
  \caption{MCTS}
  \label{code-MCTS}
  \begin{algorithmic}
    \REQUIRE Network $f_{\theta}$, root node $s_0$
    \ENSURE Return $\bm{\pi}$, enhanced policy
    
    \WHILE{$\sum_{a \in A_{s_0}}N(s_0, a) \le c_\mathrm{iter} |A_{s_0}|$}
        \STATE $s = s_0$
        \STATE \COMMENT{\emph{select}}
        \WHILE{$s$ is expanded before and $s \not \in S_\mathrm{end}$}
          \STATE $a = \argmax_{b \in A_s} \Bigl(Q(s, b) + c_{\mathrm{puct}} P(s, b)\frac{\sqrt{\sum_{b'}N(s, b')}}{1 + N(s, b)}\Bigr)$
          \STATE $s \leftarrow T(s, a)$
        \ENDWHILE
        
        \STATE \COMMENT{\emph{expand}}
        \IF{$s \not \in S_\mathrm{end}$}
            \STATE $(\bm{p}, \bm{v}) = f_{\theta}(s)$
            \STATE initialize $(N(s, a) = 0, W(s, a) = 0, Q(s, a) = 0, P(s, a) = \bm{p_a})$ for each $a \in A_s$
            \STATE Calculate $(\mu_s, \sigma_s)$ by random sampling
        \ENDIF
        
        \STATE \COMMENT{\emph{backup}}
        \STATE $r = r_{\mathrm{estim}}(s)$
        \WHILE{$s$ is not $s_0$}
          \STATE $a =$ previous action
          \STATE $s \leftarrow $ parent of $s$
          \STATE $r \leftarrow r + R(s, a)$
          \STATE $r' = (r - \mu_s) / \sigma_s$
          \STATE $W(s, a) \leftarrow W(s, a) + r'$
          \STATE $N(s,a) \leftarrow N(s,a) + 1$
          \STATE $Q(s,a) \leftarrow \frac{W(s,a)} {N(s, a)}$
        \ENDWHILE
    \ENDWHILE
    \STATE Compute $\Bigl(\bm{\pi}_a = \frac{N(s_0, a)^{1 / \tau}}{\sum_b N(s_0, b)^{1 / \tau}} \Bigr)$ for each $a$
    \STATE return $\bm{\pi}$
  \end{algorithmic}
\end{algorithm}

\paragraph{Training}
Following AlphaGo Zero, the training of CombOpt Zero is composed of three roles:
\emph{data generators}, \emph{learners}, and \emph{model evaluators}.
The best model that has performed best so far is shared among all of these three components.
The \emph{data generator} repeats generating self-play records for a randomly generated graph input, by the MCTS based on the current best model.
The records are the sequence of
$(s, a, \bm{\pi}, z')$,
which means action $a$ was taken from $s$ depending on the MCTS-enhanced policy $\bm{\pi}$, and $z' = (z - \mu_R(s)) / \sigma_R(s)$ where $z$ is the cumulative reward from $s$ to the terminal state.
The \emph{learner} randomly sample mini-batches from the generator's data and update the parameter of the best model so that it minimizes the loss
\begin{equation} \label{comboptzero-loss}
    \mathcal{L} = (z' - \bm{v}_a)^2 + \mathrm{CrossEntropy}(\bm{p}, \bm{\pi}) + c_{\mathrm{reg}}\|\theta\|_2^2 .
\end{equation}
There are two modification from \eqref{alphagoloss}: since $\bm{v}$ is a vector, $\bm{v}_a$ is used for the loss; since we aim at learning normalized action value function $\bm{v}$, we use the normalized cumulative reward $z'$ instead of $z$.
The \emph{model evaluator} compares the updated model with the best one and stores the better one.
In CombOpt Zero, where we cannot compare the winning rate of two players, the evaluator generates random graph instances each time and compare the mean performance on them.

In CombOpt Zero, the data generator needs to calculate $z'$ for the broader reward function $R$ defined in Section \ref{MDP}.
This can be achieved in the same way as \emph{backup} in MCTS:
after generating the self-play trajectory, it first calculates the cumulative sum of immediate rewards in reverse order and normalizes them with $\mu_s$ and $\sigma_s$.
In this way, we can compute the normalized rewards from the trajectory effectively.
The pseudocode of the data generator is available in Algorithm \ref{code-data-generation}.

\begin{algorithm}[t]
\small
  \caption{Self-play Data Generation}
  \label{code-data-generation}
  \begin{algorithmic}
    \REQUIRE Network $f_{\theta}$, Initial graph $G_0$
    \ENSURE Return self-play data records
    \STATE \COMMENT{self-play}
    \STATE $s = \mathrm{Init}(G_0)$
    \STATE $\mathrm{records} = \{\}$
    \WHILE{$s \not\in S_{\mathrm{end}}$}
      \STATE $\bm{\pi} = \mathtt{MCTS}(s)$
      \STATE $a =$ sampled action according to probability $\bm{\pi}$
      \STATE $s \leftarrow T(s, a)$
      \STATE Add $(s, a, \bm{\pi})$ to $\mathrm{records}$
    \ENDWHILE
    \STATE \COMMENT{calculate $z'$}
    \STATE $z = 0$
    \FORALL{$(s, a, \bm{\pi})$ in $\mathrm{records}$ (reversed order)}
      \STATE $z \leftarrow z + R(s, a)$
      \STATE $z' = (z - \mu_R(s)) / \sigma_R(s)$
      \STATE Replace $(s, a, \bm{\pi})$ with $(s, a, \bm{\pi}, z')$
    \ENDFOR
    \STATE return $\mathrm{records}$
  \end{algorithmic}
\end{algorithm}
\FloatBarrier

\section{Experiments}\label{sec:main-experiment}
In this section, we show some brief results of our experiments on {\sc MinimumVertexCover}, {\sc MaxCut}, and {\sc MaximumClique}.
Refer to Appendix \ref{additional-results} for the full results, including the other two problems and further analyses.

\paragraph{Competitors}
Since we aim at solving combinatorial optimization problems without domain knowledge or training dataset, S2V-DQN is our main competitor.
Additionally, for each problem, we prepared some known heuristics or approximation algorithms to compare with.
For {\sc MinimumVertexCover}, we implemented a simple randomized $2$-approximation algorithm ($2$-approx) with a reduction rule that looks for degree-$1$ nodes. The detailed algorithm is explained in Appendix \ref{appendix:mvc-full}.
For {\sc MaxCut}, we used a randomized algorithm by semidefinite programming \citep{maxcut-878} and two heuristics, non-linear optimization with local search \cite{maxcut-burer} and cross-entropy method with local search \cite{maxcut-laguna} from MQLib \cite{maxcut-heuristic}.

\paragraph{Training and Test}
We trained the models for two hours to make the training of both CombOpt Zero and S2V-DQN converge (Only $2$-IGN+ was trained for four hours due to its slow inference).
Note that CombOpt Zero was trained on $4$ GPUs while S2V-DQN uses a single GPU because of its implementation, which we discuss in Section \ref{section:tradeoff}.
For the hyperparameters of the MCTS, we referred to the original AlphaGo Zero and its reimplementation, ELF OpenGo \cite{ELF-OpenGo}.
See Appendix \ref{experiment-settings} for the detailed environment and hyperparameters.
Since we sometimes observed extremely poor performance for a few models both in S2V-DQN and CombOpt Zero when applied to large graphs, we trained five different models from random parameters and took the best value among them at the test time. 
Throughout the experiments, all models were trained on randomly genrated Erd\H{o}s-Renyi (ER) \cite{ER-model} graphs with $80 \le n \le 100$ nodes and edge probability $p = 0.15$ except for {\sc MaxCut}, where nodes were $40 \le n \le 50$ and for {\sc MaximumClique}, where edge probability was $p=0.5$.
As the input node feature of {\sc MaxCut}, we used a two-dimensional feature vector that stores the number of adjacent nodes of color $1$ and color $2$.
For the other problems, we used a vector of ones.
In tests, to keep the fairness, CombOpt Zero conducted a greedy selection on network policy output $\bm{p}$, which is the same way as S2V-DQN works (except for Section \ref{test-time-MCTS} where we compared the greedy selection and MCTS).

\paragraph{Dataset}
We generated ER and Barab\'{a}si-Albert (BA) graphs \cite{BA-model} of different sizes for testing.
ER100\_15 denotes an ER graph with $100$ nodes and edge probability $p = 0.15$ and BA100\_5 denotes a BA graph with $100$ nodes and $5$ edges addition per node.
Also, we used $10$ real-world graphs from Network Repository \cite{NetworkRepository}, including citation networks, web graphs, bio graphs, and road map graphs, all of which were handled as unlabeled and undirected graphs.
To see all the results for these $10$ graphs, please refer to Appendix \ref{additional-results}.
For {\sc MinimumVertexCover} and {\sc MaximumIndependentSet}, we additionally tested on DIMACS \footnote{https://turing.cs.hbg.psu.edu/txn131/vertex\_cover.html}, difficult artificial instances.
We generated other synthetic instances in Section \ref{section:compare-generalization}.

\begin{table}[t]
\centering
\begin{threeparttable}
\caption{\textbf{Generalization Comparision between CombOpt Zero and S2V-DQN.} COZ and DQN are short for CombOpt Zero and S2V-DQN respectively. Smaller is better for {\sc MinimumVertexCover} and larger is better for {\sc MaxCut}.}
\label{table:comp-coz-dqn} 
\small
\begin{tabular}{lrr|rr}
\toprule
 & \multicolumn{2}{c|}{\sc MVC} & \multicolumn{2}{c}{\sc MaxCut} \\
 & COZ & DQN & COZ & DQN \\
\midrule
er100\_15 & $76$ & $76$ & $494$ & $\mathbf{527}$\\
er1000\_5 & $900$ & $\mathbf{898}$ & $12561$ & $\mathbf{14424}$\\
er5000\_1 & $4484$ & $\mathbf{4482}$ & $63389$ & $\mathbf{71909}$\\
\midrule
ba100\_5 & $63$ & $63$ & $337$ & $\mathbf{343}$\\
ba1000\_5 & $\mathbf{592}$ & $594$ & $\mathbf{3492}$ & $3465$\\
ba5000\_5 & $\mathbf{2920}$ & $2927$ & $\mathbf{17381}$ & $16870$\\
ws100\_k2\_p10 & $49$ & $49$ & $\mathbf{98}$ & $97$\\
ws100\_k10\_p10 & $78$ & $78$ & $335$ & $335$\\
ws1000\_k2\_p10 & $\mathbf{492}$ & $496$ & $\mathbf{999}$ & $973$\\
ws1000\_k4\_p10 & $\mathbf{635}$ & $636$ & $\mathbf{1536}$ & $1327$\\
ws1000\_k10\_p10 & $787$ & $\mathbf{784}$ & $\mathbf{3312}$ & $3287$\\
regular\_100\_d5 & $\mathbf{63}$ & $64$ & $\mathbf{207}$ & $205$\\
regular\_1000\_d5 & $\mathbf{632}$ & $634$ & $\mathbf{2051}$ & $2046$\\
tree100 & $44$ & $44$ & $\mathbf{99}$ & $98$\\
tree1000 & $\mathbf{439}$ & $440$ & $\mathbf{999}$ & $982$\\
\midrule
cora & $1258$ & $1258$ & $\mathbf{4260}$ & $4243$\\
citeseer & $1462$ & $\mathbf{1461}$ & $\mathbf{3933}$ & $3893$\\
web-edu & $1451$ & $1451$ & $\mathbf{4712}$ & $4289$\\
web-spam & $\mathbf{2299}$ & $2319$ & $20645$ & $\mathbf{21027}$\\
road-minnesota & $\mathbf{1324}$ & $1329$ & $\mathbf{3080}$ & $3015$\\
bio-yeast & $\mathbf{456}$ & $457$ & $\mathbf{1769}$ & $1751$\\
bio-SC-LC & $\mathbf{1039}$ & $1053$ & $10893$ & $\mathbf{11890}$\\
rt\_damascus & $369$ & $369$ & $\mathbf{3698}$ & $3667$\\
soc-wiki-vote & $407$ & $\mathbf{406}$ & $\mathbf{2119}$ & $2064$\\
socfb-bowdoin47 & $1796$ & $\mathbf{1793}$ & $\mathbf{42063}$ & $37140$\\
\bottomrule
\end{tabular}
\end{threeparttable}
\end{table}

\begin{table*}[t!]
\centering
\begin{threeparttable}
\caption{\textbf{GNN comparison for {\sc MinimumVertexCover}.} Smaller is better. Since the inference takes $\Theta(n^3)$ per action for $2$-IGN+, it did not finish within $2$ hours for some test instances.}
\label{table:main-mvc} 
\small
\begin{tabular}{lrr|rrrr|r|r}
\toprule
 & & & \multicolumn{4}{c|}{CombOpt Zero} & &\\
 & $|V|$ & $|E|$ & 2-IGN+ & GIN & GCN & S2V & S2V-DQN & $2$-approx \\
\midrule
er100\_15 & $100$ & $783$ & $77$ & $77$ & $\mathbf{76}$ & $\mathbf{76}$ & $\mathbf{76}$ & $\mathit{83}$\\
ba100\_5 & $100$ & $475$ & $63$ & $63$ & $63$ & $63$ & $63$ & $\mathit{69}$\\
cora & $2708$ & $5429$ & - & $\mathbf{1257}$ & $1258$ & $1258$ & $1258$ & $\mathit{1274}$\\
citeseer & $3327$ & $4552$ & - & $\mathbf{1460}$ & $1462$ & $1462$ & $1461$ & $\mathit{1475}$\\
web-edu & $3031$ & $6474$ & - & $1451$ & $1451$ & $1451$ & $1451$ & $\mathit{1451}$\\
web-spam & $4767$ & $37375$ & - & $2305$ & $\mathbf{2298}$ & $2299$ & $2319$ & $\mathit{2420}$\\
soc-wiki-vote & $889$ & $2914$ & $413$ & $\mathbf{406}$ & $\mathbf{406}$ & $407$ & $\mathbf{406}$ & $\mathit{406}$\\
socfb-bowdoin47 & $2252$ & $84387$ & $2187$ & $1793$ & $\mathbf{1792}$ & $1796$ & $1793$ & $\mathit{2052}$\\
dimacs-frb30-15-1 & $450$ & $17827$ & $436$ & $427$ & $\mathbf{426}$ & $\mathbf{426}$ & $427$ & $\mathit{436}$\\
dimacs-frb50-23-1 & $1150$ & $80072$ & $1132$ & $1115$ & $\mathbf{1111}$ & $1116$ & $1114$ & $\mathit{1130}$\\
\bottomrule
\end{tabular}
\end{threeparttable}
\end{table*}

\begin{table*}[t!]
\centering
\begin{threeparttable}
\caption{\textbf{GNN comparison for {\sc MaxCut}.} Larger is better. CombOpt Zero outperformed the SOTA heuristic solvers on several real-world instances.}
\label{table:main-maxcut} 
\small
\begin{tabular}{lrr|rrrr|r|r|rr}
\toprule
 & & & \multicolumn{4}{c|}{CombOpt Zero} & & & \multicolumn{2}{c}{Heuristics}\\
 & $|V|$ & $|E|$ & 2-IGN+ & GIN & GCN & S2V & S2V-DQN & SDP & Burer & Laguna \\
\midrule
er100\_15 & $100$ & $783$ & $516$ & $526$ & $390$ & $494$ & $\mathbf{527}$ & $\mathit{521}$ & $\mathit{528}$ & $\mathit{528}$\\
ba100\_5 & $100$ & $475$ & $324$ & $\mathbf{343}$ & $282$ & $337$ & $\mathbf{343}$ & $341$ & $\mathit{344}$ & $\mathit{344}$\\
cora & $2708$ & $5429$ & - & $\mathbf{4268}$ & $3945$ & $4260$ & $4243$ & - & $\mathit{4394}$ & $\mathit{4383}$\\
citeseer & $3327$ & $4552$ & - & $3929$ & $3477$ & $\mathbf{3933}$ & $3893$ & - & $\mathit{3927}$ & $\mathit{3927}$\\
web-edu & $3031$ & $6474$ & - & $4705$ & $4243$ & $\mathbf{4712}$ & $4289$ & - & $\mathit{4679}$ & $\mathit{4714}$\\
web-spam & $4767$ & $37375$ & - & $\mathbf{23882}$ & $20498$ & $20645$ & $21027$ & - & $\mathit{25001}$ & $\mathit{25070}$\\
soc-wiki-vote & $889$ & $2914$ & $1645$ & $2116$ & $1730$ & $\mathbf{2119}$ & $2064$ & - & $\mathit{2175}$ & $\mathit{2163}$\\
socfb-bowdoin47 & $2252$ & $84387$ & $41741$ & $\mathbf{47426}$ & $20002$ & $42063$ & $37140$ & - & $\mathit{48639}$ & $\mathit{48636}$\\
\bottomrule
\end{tabular}
\end{threeparttable}
\end{table*}

\begin{table*}[t!]
\centering
\small
\begin{threeparttable}
\caption{\textbf{Improvement by test-time MCTS for {\sc MaximumClique}}. Larger is better. Results with test-time MCTS are shown in the parentheses. We broke the lower bounds of maximum clique size for some instances. Stability was also improved.}
\label{table:test-time-mcts} 
\begin{tabular}{lrr|rr|rr|rr|rr|r|r}
\toprule
 & & & \multicolumn{8}{c|}{CombOpt Zero} & &\\
 & $|V|$ & $|E|$ & \multicolumn{2}{c}{2-IGN+} & \multicolumn{2}{c}{GIN} & \multicolumn{2}{c}{GCN} & \multicolumn{2}{c|}{S2V} & S2V-DQN & Best known\\
\midrule
cora & $2708$ & $5429$ & $4$ & $(\mathbf{5})$ & $4$ & $(\mathbf{5})$ & $3$ & $(\mathbf{5})$ & $4$ & $(\mathbf{5})$ & $4$ & $\mathit{5}$\\
citeseer & $3327$ & $4552$ & $4$ & $(6)$ & $5$ & $(6)$ & $4$ & $(6)$ & $4$ & $(6)$ & $4$ & $\mathit{9}$\\
web-edu & $3031$ & $6474$ & $16$ & $(\mathbf{30})$ & $16$ & $(\mathbf{30})$ & $16$ & $(16)$ & $16$ & $(\mathbf{30})$ & $16$ & $\mathit{16}$\\
web-spam & $4767$ & $37375$ & $10$ & $(\mathbf{20})$ & $16$ & $(17)$ & $7$ & $(17)$ & $16$ & $(17)$ & $16$ & $\mathit{14}$\\
soc-wiki-vote & $889$ & $2914$ & $5$ & $(\mathbf{7})$ & $6$ & $(\mathbf{7})$ & $6$ & $(\mathbf{7})$ & $6$ & $(\mathbf{7})$ & $6$ & $\mathit{6}$\\
socfb-bowdoin47 & $2252$ & $84387$ & $14$ & $(\mathbf{23})$ & $15$ & $(\mathbf{23})$ & $7$ & $(22)$ & $13$ & $(\mathbf{23})$ & $14$ & $\mathit{9}$\\
\bottomrule
\end{tabular}
\end{threeparttable}
\end{table*}

\subsection{Comparison of Generalization Ability} \label{section:compare-generalization}
We compared the generalization ability of CombOpt Zero and S2V-DQN to various kinds of graphs.
To see the pure contribution of CombOpt Zero, here CombOpt Zero incorporated the same graph representation model as S2V-DQN, namely S2V.
Table \ref{table:comp-coz-dqn} shows the performance for {\sc MinimumVertexCover} and {\sc MaxCut} on various graph instances (see Appendix \ref{experiment-settings} for the explanation of each graph).
While S2V-DQN had a better performance on ER graphs, which were used for training, CombOpt Zero showed a better generalization ability to the other synthetic graphs such as BA graphs, Watts-Strogatz graphs \cite{WA-model}, and sparse regular graphs (i.e., graphs with the same degree of nodes).
It was interesting that CombOpt Zero successfully learned the optimal solution of {\sc MaxCut} on trees (two-coloring of a tree puts all the edges into the cut set), while S2V-DQN does not.
Appendix \ref{visualization} visualizes how CombOpt Zero achieves the optimal solution on trees.
CombOpt Zero also generalized better to real-world graphs although S2V-DQN performed better on a few instances.

\subsection{Combination with Several GNN Models} \label{section:compare-gnn}
Tables \ref{table:main-mvc} and \ref{table:main-maxcut} show the comparison of performance among CombOpt Zero with four different GNN models ($2$-IGN+, GIN, GCN, and S2V) and S2V-DQN.
In both {\sc MinimumVertexCover} and {\sc MaxCut}, the use of recent GNN models enhanced the performance, but the best models were different across the problems and test instances.
While GIN had the best performance in {\sc MaxCut}, GCN performed slightly better in {\sc MinimumVertexCover}.
See Appendix \ref{additional-results} for all the results of the $5$ problems.
One interesting insight was that GCN performed significantly worse in {\sc MaxCut}, while it did almost the best for the other four NP-hard problems.
Overall, CombOpt Zero outperformed S2V-DQN by properly selecting a GNN model.

\subsection{Test-time MCTS} \label{test-time-MCTS}
The greedy selection in Sections \ref{section:compare-generalization} and \ref{section:compare-gnn} does not make full use of CombOpt Zero.
When more computational time is allowed in test-time, CombOpt Zero can explore better solutions using the MCTS.
Since the MCTS on large graphs takes a long time, we chose {\sc MaximumClique} as a case study because solution sizes and the MCTS depths are much smaller than the other problems.
We used the same algorithm as in the training (Algorithm \ref{code-MCTS}) with the same iteration coefficient ($c_{\mathrm{iter}} = 4$) and $\tau = 0$, and selected an action based on the enhanced policy $\bm{\pi}$.
In all the instances, the MCTS finished in a few minutes on a single process and single GPU, thanks to the small depth of the MCTS tree.
The improvements are shown in Table \ref{table:test-time-mcts}.
The test-time MCTS strictly improved performance on $7$ instances out of $10$ and got at least the same on all of the instances.
Surprisingly, our results surpassed the best known solution reported in Network Repository on $6$ out of $10$ instances. For example, we found a clique of size $30$ on web-edu, whose previous bound was $16$.

\subsection{Tradeoffs between CombOpt Zero and S2V-DQN} \label{section:tradeoff}
Here, we summarize some characteristics of CombOpt Zero and S2V-DQN.
During the training, CombOpt Zero used $32$ processes and four GPUs as described in Appendix \ref{experiment-settings}, while S2V-DQN used a single process and GPU because of its implementation.
CombOpt Zero takes a longer time to generate self-play data than S2V-DQN due to the MCTS process.
For this reason, CombOpt Zero needs a more powerful environment to obtain stable training.
On the other hand, since CombOpt Zero is much more sample-efficient than S2V-DQN (see Appendix \ref{appendix:training-convergence}), the bottleneck of the CombOpt Zero training is the data generation by the MCTS.
This means that it can be highly optimized with an enormous GPU or TPU environment as in \citet{alphago-zero}, \citet{alphazero}, and \citet{ELF-OpenGo}.

By its nature, S2V-DQN can be also combined with other GNN models than S2V.
However, S2V-DQN is directly implemented with GPU programming, it is practically laborious to combine various GNN models.
On the other hand, CombOpt Zero is based on PyTorch framework \cite{PyTorch} and it is relatively easy to implement different GNN models.
\section{Conclusion}
In this paper, we presented a general framework, CombOpt Zero, to solve combinatorial optimization on graphs without domain knowledge.
The Monte Carlo Tree Search (MCTS) in training time successfully helped the wider exploration than the existing method and enhanced the generalization ability to various graphs.
Combined with the recently-designed powerful GNN models, CombOpt Zero achieved even better performance.
We also observed that the test-time MCTS significantly enhanced its performance and stability.

\section*{Acknowledgements}
MS was supported by the International Research Center for Neurointelligence (WPI-IRCN) at The University of Tokyo Institutes for Advanced Study.

\bibliography{combopt_rl}
\bibliographystyle{plainnat}

\clearpage
\appendix
\section{Other NP-hard Problems} \label{Other-NPhard}
Here, we introduce two more NP-hard problems we used in the experiments.

\subsection{Maximum Independent Set}
A subset of nodes $V' \subset V$ is called an independent set if no two nodes in $V'$ are adjacent;
for all $(x, y) \in E$, $x \notin V'$ or $y \notin V'$ holds.
{\sc MaximumIndependentSet} asks for an independent set whose size is maximum.

\paragraph{MDP Formulation}
$d$, $A_s$, $\mathrm{Init}$, and $S_{\mathrm{end}}$ are the same as {\sc MinimumVertexCover}.
$T(s, x)$ returns the next state, corresponding to the graph where $x$ and its adjacent nodes are deleted.
$R(s, x) = 1$ because we want to maximize the number of transition steps of the MDP.

\subsection{Minimum Feedback Vertex Set}
A subset of nodes $V' \subset V$ is called a feedback vertex set if the induced graph $G[V \backslash V']$ is cycle-free.
{\sc MinimumFeedbackVertexSet} asks a feedback vertex set whose size is minimum.

\paragraph{MDP Formulation}
$d$, $A_s$, $T(s, x)$, $R(s, x)$ and $\mathrm{Init}$ are the same as {\sc MinimumVertexCover}.
$S_{\mathrm{end}}$ is the states with cycle-free graphs.

\section{Graph Neural Networks} \label{detail-GNN}
In this section, we review several Graph Neural Network (GNN) models and explain some modifications for our problem setting.
Given a graph and input node feature $H^{(0)} \in \mathbb{R}^{n \times C_0}$, we aim at obtaining $y \in \mathbb{R} ^ {n \times C_{\mathrm{out}}}$ which represents node feature as an output.
$L$ denotes a number of layers.

\subsection{structure2vec}\label{detail-GNN-S2V}
In \texttt{structure2vec} (S2V) \cite{structure2vec}, the feature is propagated by
\begin{align*}
    H^{(l + 1)}_v = \mathrm{relu}\Bigl(\theta_1 H^{(0)}_v + \theta_2 \sum_{u \in \mathcal{N}(v)}H_u^{(l)} \bigr),
\end{align*}
where $\theta_1 \in \mathbb{R}^{p \times C_0}$, $\theta_2 \in \mathbb{R}^{p \times p}$ for some fixed integer $p$.
The edge propagation term is ignored since we don't handle weighted edges in this work.
In the last layer, the features of every node are aggregated to encode the whole graph:
\begin{align*}
    y_v = \theta_3^T \mathrm{relu}\bigl(\bigl[\theta_4\sum_{u \in V} H_u^{(L)}, \theta_5 H_v^{(L)}\bigr]\bigr),
\end{align*}
where $\theta_3 \in \mathbb{R}^{C_{\mathrm{out}} \times 2p}$, $\theta_4, \theta_5 \in \mathbb{R}^{p \times p}$ and $[\cdot, \cdot]$ is the concatenation operator.

\subsection{Graph Convolutional Network}\label{detail-GCN}
Graph Convolutional Network (GCN) \cite{GCNsKipf} follows the layer-wise propagation rule:
\begin{equation}
    H^{(l + 1)} = \sigma(\tilde{D}^{-1/2} \tilde{A} \tilde{D}^{-1/2} H^{(l)} \theta^{(l)}) .
\end{equation}
$\tilde{A} = A + I_n$ where $I_n$ is the identity matrix of the size $n$ and expresses an adjacency matrix with self-connections.
Multiplying $H^{(l)}$ by $\tilde{A}$ means passing each node's feature vector to its neighbors' feature vectors in the next layer.
$\tilde{D}$ is the degree matrix of $\tilde{A}$, such that
$\tilde{D_{ii}} = \sum_j \tilde{A_{ij}}$ for diagonal elements and $0$ for the other elements.
$\tilde{D}^{-1/2}$ is multiplied to normalize $\tilde{A}$.
$\theta^l$ is a trainable weight matrix in $l$-th layer.

\subsection{Graph Isomorphism Network}\label{detail-GIN}
The propagation rule of Graph Isomorphism Network (GIN) \cite{GIN} is as follows:
\begin{equation*}
  H^{(l + 1)} = \mathrm{MLP}^{(l)} (\tilde{A} H^{(l)}),
\end{equation*}
where $\mathrm{MLP}^l$ refers to the multi-layer perceptron in the $l$-th layer.
It takes the simple sum of features among the neighbors and itself by multiplying $\tilde{A}$.
We adopt a similar suffix as the original paper:
\begin{equation}
    y_v = \mathrm{MLP}(\mathrm{CONCAT}(H^{(l)}_v \mid l = 0, 1, \cdots , L))
\end{equation}

\subsection{2-IGN+}\label{detail-PGNN}
Different from message passing GNNs, each block of 2-IGN+ \cite{PGNN} takes a tensor as an input.
It follows the propagation rule:
\begin{equation}
    {\bf X}^{(l + 1)} = B_l({\bf X}^{(l)}).
\end{equation}
${\bf X}^{(l)} \in \mathbb{R}^{n \times n \times C_l}$ is a hidden tensor of the $l$-th block.
${\bf X}^{(0)} \in \mathbb{R}^{n \times n \times (C_0 + 1)}$ is initialized as the concatenation of the adjacency matrix and a tensor with node feature in its diagonal elements (other elements are $0$).
$B_l$ is the $l$-th block of 2-IGN+ and consists of three MLPs: $m_1^{(l)}, m_2^{(l)} \in \mathbb{R}^{C_l} \rightarrow \mathbb{R}^{C'_l}$ and $m_3^{(l)} \in \mathbb{R}^{C_l + C'_l} \rightarrow \mathbb{R}^{C'_l}$.
After applying the two MLPs independently, we perform feature-wise matrix multiplication:
${\bf W}^{(l)}_{:,:,j} = m_1^{(l)}({\bf X}^{(l)})_{:,:,j} \cdot m_2^{(l)}({\bf X}^{(l)})_{:,:,j}$.
The output of the block is the last MLP over the concatenation with ${\bm X}^{(l)}$: ${\bm X}^{(l+1)} = m_3^{(l)}([{\bf X}^{(l)}, {\bf W}^{(l)}])$.
In our problem setting, we adopt equivariant linears layer instead of invariant linear layers to obtain the output tensor $y \in \mathbb{R} ^ {n \times C_{\mathrm{out}}}$.
We also adopt the suffix mentioned in the original paper:
\begin{equation}
    y = \sum_{l = 1}^L h^{(l)}(X^{(l)}),
\end{equation}
where $h^{(l)} \in \mathbb{R}^{n \times C_l} \rightarrow \mathbb{R}^{n \times C_{\mathrm{out}}}$ is an equivariant linear layer of $l$-th block.

\section{Experiment Settings}\label{experiment-settings}
In this section, we explain the detailed settings of experiments.

\subsection{Environment}
The experiments were run on Intel Xeon E5-2695 v4, with four NVIDIA Tesla P100 GPUs. Libraries were compiled with GCC 5.4.0 and CUDA 9.2.148.

\subsection{Competitors}
Our main competitor was S2V-DQN, the state-of-the-art reinforcement learning solver of NP-hard problems.
We also tested some heuristics and approximation algorithms for specific problems.
Note that a fair setting of the running environment is difficult since they do not use GPUs nor training time.
Therefore, they were referred to just for checking how successfully CombOpt Zero learned combinatorial structure and algorithms for each problem.
For {\sc MinimumVertexCover}, {\sc MaximumIndependentSet}, {\sc MinimumFeedbackVertexSet} and {\sc MaxCut}, we compared our algorithm to randomized algorithms.
The results of the randomized algorithm were the best objective among $100$ runs.
For {\sc MaxCut}, we tried two heuristics solvers from MQLib \cite{maxcut-heuristic}.
We set the time limit of $10$ minutes for these algorithms and used the best found solution as the results.
We also used CPLEX to solve the integer programming formulation of {\sc MinimumVertexCover} and {\sc MaximumIndependentSet}.
Only CPLEX was run on MacBook Pro 2.4 GHz Quad-Core Intel Core i5.
It was executed on $8$ threads with the time limit of $10$ minutes.

\subsection{Hyperparameters}
Some of the hyperparameters are summarized in Table \ref{table:hyperparameters}.

\paragraph{Graph Neural Networks}
For S2V, we used the same hyperparameters as the ones used in S2V-DQN \citep{S2V-DQN}:
we set embedding dimension $p=64$ and the number of iterations $L=5$.
For GCN, we used a $5$-layer network with a hidden dimension of size $32$.
For GIN, the network consisted of $5$ layers of hidden dimension of size $32$.
Each MLP consisted of $5$ layers also and the size of a hidden dimension was $16$.
Lastly, for $2$-IGN+, we used a $2$-block network with a hidden dimension of size $8$. Each MLP had $2$ layers with a hidden dimension of size $8$.

\paragraph{MCTS}
We followed the implementation of AlphaGo Zero \citep{alphago-zero} and ELF \citep{ELF-OpenGo}. We set $c_\mathrm{puct} = 1.5$. In training phase and test-time MCTS, we added Dirichlet noise $\bm{\eta} \sim \mathrm{Dir}(\bm{0.03})$ to the first move to explore a variety of actions. More specifically, the policy $\bm{p}$ for the first action is modified to $\bm{p} \leftarrow (1 - \varepsilon) \bm{p} + \varepsilon \bm{\eta}$, where $\varepsilon = 0.25$. We set $c_\mathrm{iter} = 4$, i.e., ran MCTS for $c_\mathrm{iter} |A_{s_0}|$ times, for {\sc MaxCut}, {\sc MaximumClique}, {\sc MaximumIndependentSet}. We set $c_\mathrm{iter} = 3$ for {\sc MinimumVertexCover} and {\sc MinimumFeedbackVertexSet} because they took relatively longer time due to larger solution sizes. We set the temperature as $\tau = 1$ during the training phase, and $\tau = 0$ in the test-time MCTS. $\tau = 0$ means to take an action whose visit count is the maximum (if there are multiple possible actions, choose one at random). When a new node is visited in MCTS, we calculated the mean and standard deviation of the reward from its corresponding state. We approximated these value from $20$ random plays.

\paragraph{Training}
Each \emph{learner} sampled $20$ trajectories from the self-play records and ran stochastic gradient descent by Adam, where the learning rate is $0.001$, weight decay is $0.0001$, and batch size is $16$.
After repeating this routine for $15$ times, the \emph{learner} saved the new model.
The number of \emph{data generators} was $24$ except when $2$-IGN+ was used.
In that case, the number was $20$.
The number of \emph{learners} and \emph{model evaluators} were set to $6$ and $2$, respectively.
Each \emph{model evaluator} generated $50$ test ER graphs with $n=100$ and $p=0.15$ ($n=50$ for {\sc MaxCut} and $p=0.5$ for {\sc MaximumClique}) each time. It compared the performance of the best model and a newly generated model by the cummulative objective and managed the best model.
Each trajectory was removed after $5$ minutes ($10$ minutes for {\sc MinimumVertexCover} and {\sc MinimumFeedbackVertexSet}) since it was generated.

\begin{table}[t!]
\centering
\begin{threeparttable}
\caption{\textbf{Part of hyperparameters.} For each of the five problems, some important hyperparameters are summarized. $\mathrm{min}\ n$, $\mathrm{max}\ n$, $p$ regard ER graphs, $c_\mathrm{iter}$ is for MCTS, and $\mathrm{keep}$ denotes the duration (minutes) to keep trajectories.}
\label{table:hyperparameters} 
\small
\begin{tabular}{lccccc}
\toprule
 & $\mathrm{min}\ n$ & $\mathrm{max}\ n$ & $p$ & $c_\mathrm{iter}$ & $\mathrm{keep}$\\
\midrule
{\sc MVC} & $80$ & $100$ & $0.15$ & $3$ & $10$\\
{\sc MaxCut} & $40$ & $50$ & $0.15$ & $4$ & $5$\\
{\sc MaxClique} & $80$ & $100$ & $0.5$ & $4$ & $5$\\
{\sc MIS} & $80$ & $100$ & $0.15$ & $4$ & $5$\\
{\sc FVS} & $80$ & $100$ & $0.15$ & $3$ & $10$\\
\bottomrule
\end{tabular}
\end{threeparttable}
\end{table}

\section{Performance Comparison of CombOpt Zero and Other Approaches} \label{additional-results}
In this section, we show the full results of our experiments. We compared CombOpt Zero to some simple randomized algorithms and state-of-the-art solvers, in addition to S2V-DQN. For each of the five problems, we first explain the characteristics of the problem and some famous approaches, then we show the comparison of the performances.

\subsection{Minimum Vertex Cover}\label{appendix:mvc-full}
\paragraph{Approximability}
There is a simple 2-approximation algorithm for {\sc MinimumVertexCover}. It greedily obtains a maximal (not necessarily maximum) matching and outputs the nodes in the matching as a solution.
Under Unique Games Conjecture \citep{ugc}, {\sc MinimumVertexCover} cannot be approximated better than this.

\paragraph{Randomized Heuristics}
Regardless of the hardness of approximation of {\sc MinimumVertexCover}, there are some practical algorithms to solve this (and equivalently, {\sc MaximumIndependentSet}) \citep{mis-sugoi, iwata}. These state-of-the-art algorithms usually iteratively \emph{kernelize} the graph, i.e., reduce the size of the problem, and search better solutions, either in an exact way or in an approximated way.
We adopted the easiest reduction rule to design a simple randomized algorithm (Algorithm \ref{mvc-randomized}) for {\sc MinimumVertexCover}; if the input graph has a degree-1 vertex $v$, there is always a vertex cover that does not contain $v$.
Although theoretically, this reduction rule does not improve the approximation ratio, i.e., the approximation is still $2$, it is effective because we can cut off trivial solutions and make the size of the problem smaller.

\begin{algorithm}[t!]
  \small
  \caption{A Simple Randomized Algorithm for {\sc MinimumVertexCover}}
  \label{mvc-randomized}
  \begin{algorithmic}
    \REQUIRE Graph $G = (V, E)$
    \ENSURE Size of a vertex cover $r$
    
    \STATE $r \leftarrow 0$
    \WHILE{$G$ is not empty}
        \IF{$G$ has a degree-1 vertex}
            \STATE $v \leftarrow$ uniformly randomly chosen degree-1 vertex
        \ELSE
            \STATE $v \leftarrow$ uniformly randomly chosen vertex
        \ENDIF
        \STATE $r \leftarrow r + 1$
        \STATE $G \leftarrow$ delete $v$ and its neighbor(s)
    \ENDWHILE
  \end{algorithmic}
\end{algorithm}

\paragraph{Full results}
Table \ref{table:mvc-full} shows the full results for {\sc MinimumVertexCover}. Although both CombOpt Zero and S2V-DQN failed to find an optimal solution for the ER graph of $200$ nodes (er200\_10), both of them found near-optimal solutions for large cases even though they were trained on small random graphs. Remarkably, some models of CombOpt Zero found better solutions than CPLEX for large random graphs. However, for some large sparse real-world networks, they performed worse than a simple $2$-approximation algorithm. A hybrid approach that mixes the reduction rules and machine learning, as done in \cite{zhuwen}, may be effective in such sparse networks, but our work focuses on a method without domain knowledge. Theoretically, $2$-IGN+ is stronger than GIN or GCN in terms of discriminative power \citep{PGNN}, its performance on large instances was not good although they successfully learned er100\_15. One reason could be because we reduced the size of the network to fasten the training. Also, since $2$-IGN+ is not a message passing neural network, its tendency of learning could be different from other GNNs. We leave the empirical and theoretical analyses on the characteristics of learning by $2$-IGN+ as future work.

\begin{table*}[t]
\centering
\begin{threeparttable}
\caption{\textbf{Performance comparison on {\sc MinimumVertexCover}}. Smaller is better. Bold values are the best values among reinforcement learning approaches (CombOpt Zero and S2V-DQN). Since the inference takes $\Theta(n^3)$ per action for $2$-IGN+, it did not finish testing within $2$ hours for some instances.}
\label{table:mvc-full}
\small
\begin{tabular}{lrr|rrrr|r|r|r}
\toprule
 & & & \multicolumn{4}{c|}{CombOpt Zero} & &\\
 & $|V|$ & $|E|$ & 2-IGN+ & GIN & GCN & S2V & S2V-DQN & $2$-approx & CPLEX \\
\midrule
er100\_15 & $100$ & $783$ & $77$ & $77$ & $\mathbf{76}$ & $\mathbf{76}$ & $\mathbf{76}$ & $\mathit{83}$ & $\mathit{76}$ \\
er200\_10 & $200$ & $1957$ & $\mathbf{160}$ & $\mathbf{160}$ & $161$ & $\mathbf{160}$ & $\mathbf{160}$ & $\mathit{174}$ & $\mathit{159}$\\
er1000\_5 & $1000$ & $25091$ & $907$ & $900$ & $\mathbf{893}$ & $900$ & $898$ & $\mathit{954}$ & $\mathit{894}$ \\
er5000\_1 & $5000$ & $124804$ & - & $4479$ & $\mathbf{4448}$ & $4484$ & $4482$ & $\mathit{4793}$ & $\mathit{4480}$\\
ba100\_5 & $100$ & $475$ & $63$ & $63$ & $63$ & $63$ & $63$ & $\mathit{69}$ & $\mathit{63}$\\
ba200\_5 & $200$ & $975$ & $121$ & $120$ & $120$ & $120$ & $120$ & $\mathit{134}$ & $\mathit{118}$ \\
ba1000\_5 & $1000$ & $4975$ & $745$ & $592$ & $\mathbf{591}$ & $592$ & $594$ & $\mathit{670}$ & $\mathit{589}$\\
ba5000\_5 & $5000$ & $24975$ & - & $2920$ & $\mathbf{2905}$ & $2920$ & $2927$ & $\mathit{3374}$ & $\mathit{2932}$\\
cora & $2708$ & $5429$ & - & $\mathbf{1257}$ & $1258$ & $1258$ & $1258$ & $\mathit{1274}$ & $\mathit{1257}$\\
citeseer & $3327$ & $4552$ & - & $\mathbf{1460}$ & $1462$ & $1462$ & $1461$ & $\mathit{1475}$ & $\mathit{1460}$\\
web-edu & $3031$ & $6474$ & - & $1451$ & $1451$ & $1451$ & $1451$ & $\mathit{1451}$ & $\mathit{1451}$\\
web-spam & $4767$ & $37375$ & - & $2305$ & $\mathbf{2298}$ & $2299$ & $2319$ & $\mathit{2420}$ & $\mathit{2297}$\\
road-minnesota & $2642$ & $3303$ & - & $\mathbf{1322}$ & $1323$ & $1324$ & $1329$ & $\mathit{1330}$ & $\mathit{1319}$\\
bio-yeast & $1458$ & $1948$ & $702$ & $\mathbf{456}$ & $\mathbf{456}$ & $\mathbf{456}$ & $457$ & $\mathit{456}$ & $\mathit{456}$\\
bio-SC-LC & $2004$ & $20452$ & $1471$ & $1041$ & $1046$ & $\mathbf{1039}$ & $1053$ & $\mathit{1245}$ & $\mathit{1036}$\\
rt\_damascus & $3052$ & $3881$ & $370$ & $369$ & $369$ & $369$ & $369$ & $\mathit{369}$ & $\mathit{369}$\\
soc-wiki-vote & $889$ & $2914$ & $413$ & $\mathbf{406}$ & $\mathbf{406}$ & $407$ & $\mathbf{406}$ & $\mathit{406}$ & $\mathit{406}$\\
socfb-bowdoin47 & $2252$ & $84387$ & $2187$ & $1793$ & $\mathbf{1792}$ & $1796$ & $1793$ & $\mathit{2052}$ & $\mathit{1792}$\\
dimacs-frb30-15-1 & $450$ & $17827$ & $436$ & $427$ & $\mathbf{426}$ & $\mathbf{426}$ & $427$ & $\mathit{436}$ & $\mathit{421}$\\
dimacs-frb50-23-1 & $1150$ & $80072$ & $1132$ & $1115$ & $\mathbf{1111}$ & $1116$ & $1114$ & $\mathit{1130}$ & $\mathit{1107}$\\
\bottomrule
\end{tabular}
\end{threeparttable}
\end{table*}

\subsection{Max Cut}
\paragraph{Approximability}
{\sc MaxCut} has a famous $0.878$-approximation randomized algorithm by SDP (semidefinite programming) \cite{maxcut-878}. Similarly to {\sc MinimumVertexCover}, this approximation ratio is best under Unique Games Conjecture \citep{ugc}.

\paragraph{Competing Algorithms}
Although the randomized algorithm mentioned above has a polynomial-time complexity and the best approximation ratio, they are rarely applied for graphs of thousands of nodes due to the large size of the SDP. We only applied this algorithm to graphs that have no more than $200$ nodes. Instead, we compared the performance with a {\sc MaxCut} solver by \citet{maxcut-heuristic}.

\paragraph{Full results}
Table \ref{table:maxcut-full} shows the full results for {\sc MaxCut}. We compared CombOpt Zero, S2V-DQN, the SDP approximation algorithm and heuristics.
Unlike {\sc MinimumVertexCover}, CombOpt Zero with GCN had poor performance.
CombOpt Zero with GIN had the best performance and it overperformed S2V-DQN. CombOpt Zero even overperformed state-of-the-art heuristics for some instances such as citeseer and rt\_damascus. It is also remarkable that all of the results by CombOpt Zero with GIN overperformed $0.878$-approximation SDP algorithm (note that SDP did not finish for large instances).

\begin{table*}[t!]
\centering
\begin{threeparttable}
\caption{\textbf{Performance comparison on {\sc MaxCut}.} Larger is better. Bold values are the best values among reinforcement learning approaches. Empty cells mean time limit exceeded.}
\label{table:maxcut-full} 
\small
\begin{tabular}{lrr|rrrr|r|r|rr}
\toprule
 & & & \multicolumn{4}{c|}{CombOpt Zero} & & & \multicolumn{2}{c}{Heuristics}\\
 & $|V|$ & $|E|$ & 2-IGN+ & GIN & GCN & S2V & S2V-DQN & SDP & Burer & Laguna \\
\midrule
er100\_15 & $100$ & $783$ & $516$ & $526$ & $390$ & $494$ & $\mathbf{527}$ & $\mathit{521}$ & $\mathit{528}$ & $\mathit{528}$\\
er200\_10 & $200$ & $1957$ & $1194$ & $\mathbf{1269}$ & $974$ & $1221$ & $1262$ & $\mathit{1266}$ & $\mathit{1289}$ & $\mathit{1289}$\\
er1000\_5 & $1000$ & $25091$ & $12278$ & $\mathbf{14787}$ & $8596$ & $12561$ & $14424$ & - & $\mathit{15164}$ & $\mathit{15140}$\\
er5000\_1 & $5000$ & $124804$ & - & $\mathbf{73275}$ & $42860$ & $63389$ & $71909$ & - & $\mathit{75601}$ & $\mathit{75406}$\\
ba100\_5 & $100$ & $475$ & $324$ & $\mathbf{343}$ & $282$ & $337$ & $\mathbf{343}$ & $341$ & $\mathit{344}$ & $\mathit{344}$\\
ba200\_5 & $200$ & $975$ & $639$ & $694$ & $576$ & $687$ & $\mathbf{698}$ & $\mathit{693}$ & $\mathit{703}$ & $\mathit{703}$\\
ba1000\_5 & $1000$ & $4975$ & $2595$ & $3485$ & $2941$ & $\mathbf{3492}$ & $3465$ & - & $\mathit{3589}$ & $\mathit{3580}$\\
ba5000\_5 & $5000$ & $24975$ & - & $\mathbf{17643}$ & $15015$ & $17381$ & $16870$ & - & $\mathit{17997}$ & $\mathit{17911}$\\
cora & $2708$ & $5429$ & - & $\mathbf{4268}$ & $3945$ & $4260$ & $4243$ & - & $\mathit{4394}$ & $\mathit{4383}$\\
citeseer & $3327$ & $4552$ & - & $3929$ & $3477$ & $\mathbf{3933}$ & $3893$ & - & $\mathit{3927}$ & $\mathit{3927}$\\
web-edu & $3031$ & $6474$ & - & $4705$ & $4243$ & $\mathbf{4712}$ & $4289$ & - & $\mathit{4679}$ & $\mathit{4714}$\\
web-spam & $4767$ & $37375$ & - & $\mathbf{23882}$ & $20498$ & $20645$ & $21027$ & - & $\mathit{25001}$ & $\mathit{25070}$\\
road-minnesota & $2642$ & $3303$ & - & $3079$ & $2856$ & $\mathbf{3080}$ & $3015$ & - & $\mathit{3091}$ & $\mathit{3024}$\\
bio-yeast & $1458$ & $1948$ & $714$ & $\mathbf{1769}$ & $1582$ & $\mathbf{1769}$ & $1751$ & - & $\mathit{1770}$ & $\mathit{1761}$\\
bio-SC-LC & $2004$ & $20452$ & $7967$ & $\mathbf{14358}$ & $11589$ & $10893$ & $11890$ & - & $\mathit{14586}$ & $\mathit{14583}$\\
rt\_damascus & $3052$ & $3881$ & - & $3694$ & $3439$ & $\mathbf{3698}$ & $3667$ & - & $\mathit{3617}$ & $\mathit{3683}$\\
soc-wiki-vote & $889$ & $2914$ & $1645$ & $2116$ & $1730$ & $\mathbf{2119}$ & $2064$ & - & $\mathit{2175}$ & $\mathit{2163}$\\
socfb-bowdoin47 & $2252$ & $84387$ & $41741$ & $\mathbf{47426}$ & $20002$ & $42063$ & $37140$ & - & $\mathit{48639}$ & $\mathit{48636}$\\
\bottomrule
\end{tabular}
\end{threeparttable}
\end{table*}

\subsection{Maximum Clique}
For {\sc MaximumClique}, we compared CombOpt Zero and S2V-DQN. For CombOpt Zero, since the solution sizes were relatively small, we also tested test-time MCTS.

Although {\sc MaximumClique} is equivalent to {\sc MaximumIndependentSet} on complement graphs, since the number of edges in the complement graphs of sparse networks becomes too large, it is usually difficult to solve {\sc MaximumClique} by algorithms designed for {\sc MinimumVertexCover}.

\paragraph{Full results}
Full results for {\sc MaximumClique} is given in Table \ref{table:maxclique-full}.
CombOpt Zero with GIN worked the best among four models and outperformed S2V-DQN especially on random graphs.
We can also observe that the performance is strongly enhanced by test-time MCTS.
For example, although the performance of CombOpt Zero with GCN or 2-IGN+ was relatively poor than other GNNs, by test-time MCTS, its performance became much stabler and solutions were also improved. Since test-time MCTS requires $n$ times larger time complexity, where $n$ is the number of nodes, it is hard to apply test-time MCTS on large instances when solution size is too large or time limit is too short.

\begin{table*}[t!]
\centering
\begin{threeparttable}
\caption{\textbf{Performance comparison on {\sc MaximumClique}.} Larger is better. Results with test-time MCTS are shown in the parentheses. Best known values were obtained from Network Repository \citep{NetworkRepository}. Test-time MCTS broke sove lower bounds of maximum clique size for some instances.}
\label{table:maxclique-full} 
\small
\begin{tabular}{lrr|rr|rr|rr|rr|r|r}
\toprule
 & & & \multicolumn{8}{c|}{CombOpt Zero} & &\\
 & $|V|$ & $|E|$ & \multicolumn{2}{c}{2-IGN+} & \multicolumn{2}{c}{GIN} & \multicolumn{2}{c}{GCN} & \multicolumn{2}{c|}{S2V} & S2V-DQN & Best known\\
\midrule
er100\_15 & $100$ & $783$ & $\mathbf{4}$ & $(\mathbf{4})$ & $\mathbf{4}$ & $(\mathbf{4})$ & $\mathbf{4}$ & $(\mathbf{4})$ & $\mathbf{4}$ & $(\mathbf{4})$ & $3$ & -\\
er200\_10 & $200$ & $1957$ & $\mathbf{4}$ & $(\mathbf{4})$ & $\mathbf{4}$ & $(\mathbf{4})$ & $\mathbf{4}$ & $(\mathbf{4})$ & $\mathbf{4}$ & $(\mathbf{4})$ & $3$ & -\\
er1000\_5 & $1000$ & $25091$ & $\mathbf{4}$ & $(\mathbf{4})$ & $\mathbf{4}$ & $(\mathbf{4})$ & $\mathbf{4}$ & $(\mathbf{4})$ & $\mathbf{4}$ & $(\mathbf{4})$ & $3$ & -\\
er5000\_1 & $5000$ & $124804$ & $3$ & $(\mathbf{4})$ & $3$ & $(\mathbf{4})$ & $3$ & $(\mathbf{4})$ & $3$ & $(\mathbf{4})$ & $3$ & -\\
ba100\_5 & $100$ & $475$ & $\mathbf{5}$ & $(\mathbf{5})$ & $\mathbf{5}$ & $(\mathbf{5})$ & $\mathbf{5}$ & $(\mathbf{5})$ &  $\mathbf{5}$ & $(\mathbf{5})$ & $3$ & -\\
ba200\_5 & $200$ & $975$ & $\mathbf{5}$ & $(\mathbf{5})$ & $\mathbf{5}$ & $(\mathbf{5})$ & $4$ & $(\mathbf{5})$ & $\mathbf{5}$ & $(\mathbf{5})$ & $3$ & -\\
ba1000\_5 & $1000$ & $4975$ & $\mathbf{5}$ & $(\mathbf{5})$ & $\mathbf{5}$ & $(\mathbf{5})$ & $\mathbf{5}$ & $(\mathbf{5})$ & $\mathbf{5}$ & $(\mathbf{5})$ & $3$ & -\\
ba5000\_5 & $5000$ & $24975$ & $4$ & $(\mathbf{5})$ & $4$ & $(\mathbf{5})$ & $4$ & $(\mathbf{5})$ & $4$ & $(\mathbf{5})$ & $3$ & -\\
cora & $2708$ & $5429$ & $4$ & $(\mathbf{5})$ & $4$ & $(\mathbf{5})$ & $3$ & $(\mathbf{5})$ & $4$ & $(\mathbf{5})$ & $4$ & $\mathbf{5}$\\
citeseer & $3327$ & $4552$ & $4$ & $(6)$ & $5$ & $(6)$ & $4$ & $(6)$ & $4$ & $(6)$ & $4$ & $\mathbf{9}$\\
web-edu & $3031$ & $6474$ & $16$ & $(\mathbf{30})$ & $16$ & $(\mathbf{30})$ & $16$ & $(16)$ & $16$ & $(\mathbf{30})$ & $16$ & $16$\\
web-spam & $4767$ & $37375$ & $10$ & $(\mathbf{20})$ & $16$ & $(17)$ & $7$ & $(17)$ & $16$ & $(17)$ & $16$ & $14$\\
road-minnesota & $2642$ & $3303$ & $2$ & $(2)$ & $2$ & $(\mathbf{3})$ & $2$ & $(\mathbf{3})$ & $\mathbf{3}$ & $(\mathbf{3})$ & $\mathbf{3}$ & $\mathbf{3}$\\
bio-yeast & $1458$ & $1948$ & $3$ & $(\mathbf{6})$ & $4$ & $(\mathbf{6})$ & $2$ & $(\mathbf{6})$ & $3$ & $(\mathbf{6})$ & $4$ & $5$\\
bio-SC-LC & $2004$ & $20452$ & $12$ & $(\mathbf{29})$ & $\mathbf{29}$ & $(\mathbf{29})$ & $3$ & $(\mathbf{29})$ & $\mathbf{29}$ & $(\mathbf{29})$ & $\mathbf{29}$ & $\mathbf{29}$\\
rt\_damascus & $3052$ & $3881$ & $3$ & $(\mathbf{4})$ & $3$ & $(\mathbf{4})$ & $3$ & $(\mathbf{4})$ & $\mathbf{4}$ & $(\mathbf{4})$ & $3$ & $\mathbf{4}$\\
soc-wiki-vote & $889$ & $2914$ & $5$ & $(\mathbf{7})$ & $6$ & $(\mathbf{7})$ & $6$ & $(\mathbf{7})$ & $6$ & $(\mathbf{7})$ & $6$ & $6$\\
socfb-bowdoin47 & $2252$ & $84387$ & $14$ & $(\mathbf{23})$ & $15$ & $(\mathbf{23})$ & $7$ & $(22)$ & $13$ & $(\mathbf{23})$ & $14$ & $9$\\
\bottomrule
\end{tabular}
\end{threeparttable}
\end{table*}

\subsection{Maximum Independent Set}

\paragraph{Randomized Algorithm}
We also compared our methods and S2V-DQN to a randomized algorithm for {\sc MaximumIndependentSet}. We adopt the randomized algorithm developed in \ref{mvc-randomized} with a slight change: for a degree-1 vertex $v$, there is always a maximum independent set that includes $v$. Therefore, instead, we include $v$ to the current solution and erase $v$ and its neighbors from the original graph. Note that although {\sc MinimumVertexCover} and {\sc MaximumIndependentSet} are theoretically equivalent, when solved by CombOpt Zero or S2V-DQN, their difficulties are different. For example, in the first stage of the training, the sizes of independent sets obtained are usually much greater than the complements of vertex covers (usually, they tend to select almost all nodes as vertex cover at the beginning).

\paragraph{Full results}
Please refer to Table \ref{table:mis-full} for the full results for {\sc MaximumIndependentSet}. Similarly to {\sc MinimumVertexCover}, even a simple randomized algorithm with recursive reductions could obtain near-optimal solutions for large sparse networks, its performance on ER graphs were much weaker than both CombOpt Zero and S2V-DQN. This is because, in ER graphs, there were only a few trivial vertices to be selected.
Notably, CombOpt Zero with GIN reached to a better solution than CPLEX on one of the DIMACS instances.

\begin{table*}[t]
\centering
\begin{threeparttable}
\caption{\textbf{Performance comparison on {\sc MaximumIndependentSet}.} Larger is better. Bold values are the best values among reinforcement learning approaches. Empty cells mean time limit exceeded.}
\label{table:mis-full}
\small
\begin{tabular}{lrr|rrrr|r|r|r}
\toprule
 & & & \multicolumn{4}{c|}{CombOpt Zero} & &\\
 & $|V|$ & $|E|$ & 2-IGN+ & GIN & GCN & S2V & S2V-DQN & randomized & CPLEX \\
\midrule
er100\_15 & $100$ & $783$ & $\mathbf{24}$ & $\mathbf{24}$ & $23$ & $\mathbf{24}$ & $\mathbf{24}$ & $\mathit{23}$ & $\mathit{24}$\\
er200\_10 & $200$ & $1957$ & $40$ & $40$ & $39$ & $\mathbf{41}$ & $40$ & $\mathit{37}$ & $\mathit{41}$\\
er1000\_5 & $1000$ & $25091$ & $106$ & $107$ & $\mathbf{108}$ & $105$ & $106$ & $\mathit{89}$ & $\mathit{107}$\\
er5000\_1 & $5000$ & $124804$ & - & $544$ & $538$ & $544$ & $\mathbf{551}$ & $\mathit{435}$ & $\mathit{544}$\\
ba100\_5 & $100$ & $475$ & $37$ & $37$ & $37$ & $37$ & $37$ & $\mathit{37}$ & $\mathit{37}$\\
ba200\_5 & $200$ & $975$ & $81$ & $81$ & $80$ & $\mathbf{82}$ & $\mathbf{82}$ & $\mathit{79}$ & $\mathit{82}$\\
ba1000\_5 & $1000$ & $4975$ & $400$ & $407$ & $403$ & $408$ & $\mathbf{409}$ & $\mathit{394}$ & $\mathit{411}$\\
ba5000\_5 & $5000$ & $24975$ & - & $2079$ & $2062$ & $\mathbf{2085}$ & $2078$ & $\mathit{1960}$ & $\mathit{2090}$\\
cora & $2708$ & $5429$ & - & $1450$ & $1448$ & $\mathbf{1451}$ & $1448$ & $\mathit{1439}$ & $\mathit{1451}$\\
citeseer & $3327$ & $4552$ & - & $1818$ & $1817$ & $\mathbf{1819}$ & $1817$ & $\mathit{1860}$ & $\mathit{1867}$\\
web-edu & $3031$ & $6474$ & - & $1580$ & $1580$ & $1580$ & $1580$ & $\mathit{1580}$ & $\mathit{1580}$\\
web-spam & $4767$ & $37375$ & - & $\mathbf{2464}$ & $2456$ & $2463$ & $2441$ & $\mathit{2434}$ & $\mathit{2470}$\\
road-minnesota & $2642$ & $3303$ & - & $1318$ & $1300$ & $1316$ & $\mathbf{1321}$ & $\mathit{1313}$ & $\mathit{1323}$\\
bio-yeast & $1458$ & $1948$ & $990$ & $1000$ & $1001$ & $\mathbf{1002}$ & $\mathbf{1002}$ & $\mathit{1002}$ & $\mathit{1002}$\\
bio-SC-LC & $2004$ & $20452$ & $945$ & $959$ & $953$ & $\mathbf{964}$ & $948$ & $\mathit{936}$ & $\mathit{968}$\\
rt\_damascus & $3052$ & $3881$ & - & $2673$ & $\mathbf{2683}$ & $2679$ & $\mathbf{2683}$ & $\mathit{2683}$ & $\mathit{2683}$\\
soc-wiki-vote & $889$ & $2914$ & $481$ & $\mathbf{483}$ & $482$ & $\mathbf{483}$ & $482$ & $\mathit{483}$ & $\mathit{483}$\\
socfb-bowdoin47 & $2252$ & $84387$ & $445$ & $456$ & $\mathbf{457}$ & $443$ & $426$ & $\mathit{392}$ & $\mathit{461}$\\
dimacs-frb30-15-1 & $450$ & $17827$ & $26$ & $26$ & $26$ & $\mathbf{27}$ & $26$ & $\mathit{24}$ & $\mathit{28}$\\
dimacs-frb50-23-1 & $1150$ & $80072$ & $41$ & $\mathbf{44}$ & $43$ & $40$ & $41$ & $\mathit{39}$ & $\mathit{43}$\\
\bottomrule
\end{tabular}
\end{threeparttable}
\end{table*}

\subsection{Minimum Feedback Vertex Set}\label{appendix:fvs}
\paragraph{Randomized Algorithm}
Similarly to {\sc MinimumVertexCover}, a 2-approximation algorithm is known for {\sc MinimumFeedbackVertexSet} \citep{fvs-2opt}. In this experiment, we implemented a parameterized algorithm that runs in time $4^kn^{O(1)}$ , where $k$ is the size of the solution, instead, which would work well for real-world sparse networks.
The implemented algorithm is based on the following theorem.
\newtheorem{theorem}{Theorem}[section]
\newtheorem{corollary}{Corollary}[section]
\begin{theorem}[\citet{parameterized}, page 101]
Let $G = (V, E)$ be a multigraph whose minimum degree is at least $3$. Then, for any feedback vertex set $V' \subset V$, more than half of edges have at least one endpoint in $V'$.
\end{theorem}
\begin{corollary}\label{fvs-select-rule}
Let $G = (V, E)$ be a multigraph whose minimum degree is at least $3$. If we choose a vertex $v$ with probability proportional to its degree, $v$ is included in a minimum feedback vertex set of $G$ with probability greater than $1/4$.
\end{corollary}
\begin{proof}
This follows from the fact that a uniformly randomly chosen edge in $G$ has at least one endpoint in a minimum feedback vertex set with probability more than $1/2$.
\end{proof}

We iteratively reduce the input graph by following rules so that the minimum degree is at least $3$.

\begin{enumerate}
    \item If multigraph $G$ has a self loop at $v$, remove vertex $v$ and include $v$ to feedback vertex set.
    \item If multigraph $G$ has multiedges whose multiplicity is greater than $2$, reduce the multiplicity to $2$.
    \item If multigraph $G$ has degree $2$ vertex $v$, erase $v$ and connect two neighbors of $v$ by a new edge.
    \item If multigraph $G$ has degree-1 vertex $v$, remove vertex $v$.
\end{enumerate}

\begin{algorithm}[t!]
  \small
  \caption{A Simple Randomized Algorithm for {\sc MinimumFeedbackVertexSet}}
  \label{fvs-randomized}
  \begin{algorithmic}
    \REQUIRE Graph $G = (V, E)$
    \ENSURE Size of a feedback vertex set $r$
    \STATE $r \leftarrow 0$
    \WHILE{$G$ has cycle(s)}
        \STATE $(G, r) \leftarrow$ $\mathrm{reduce}(G, r)$
        \STATE $v \leftarrow$ select a random node by Corollary \ref{fvs-select-rule}
        \STATE $r \leftarrow r + 1$
        \STATE $G \leftarrow$ remove vertex $v$
    \ENDWHILE
  \end{algorithmic}
\end{algorithm}

Therefore, we can construct the following randomized parameterized algorithm (Algorithm \ref{fvs-randomized}) for {\sc MinimumFeedbackVertexSet}.
This algorithm provides a minimum feedback vertex set with a probability greater than $1/4^k$, where $k$ is the size of an optimal solution. Note that we do not have to recursively reduce the graph once the minimum degree becomes at least $3$ to obtain this bound. However, practically, we can get much better solutions through recursive reductions.

\paragraph{Full results}
Full results for {\sc MinimuMFeedbackVertexSet} is given in Table \ref{table:fvs-full}. 
For the randomized parameterized algorithm, we ran it for $100$ times and took the best objective as the score. Although in this problem, we could not observe a significant superiority of CombOpt Zero against S2V-DQN, we can see a significant difference of solution sizes for some of the datasets such as web-edu, bio-SC-LC, and socfb-bowdoin47. This is probably because all of the five S2V-DQN models converged into local optima. Interestingly even the performances were significantly poor on certain instances, they still sometimes overperformed CombOpt Zero on other instances.
The randomized algorithm effectively worked on large real-world networks where the degrees were biased. On the other hand, on ER graphs, since the degrees of nodes were not as variant as social networks, CombOpt Zero and S2V-DQN performed much better than the randomized algorithm.

\begin{table*}[t!]
\centering
\begin{threeparttable}
\caption{\textbf{Performance comparison on {\sc MinimumFeedbackVertexSet}.} Smaller is better. Bold values are the best values among reinforcement learning approaches. Empty cells mean time limit exceeded.}
\label{table:fvs-full}
\small
\begin{tabular}{lrr|rrrr|r|r}
\toprule
 & & & \multicolumn{4}{c|}{CombOpt Zero} & &\\
 & $|V|$ & $|E|$ & 2-IGN+ & GIN & GCN & S2V & S2V-DQN & randomized\\
\midrule
er100\_15 & $100$ & $783$ & $64$ & $\mathbf{63}$ & $69$ & $\mathbf{63}$ & $\mathbf{63}$ & $\mathit{72}$\\
er200\_10 & $200$ & $1957$ & $143$ & $137$ & $150$ & $137$ & $\mathbf{136}$ & $\mathit{156}$\\
er1000\_5 & $1000$ & $25091$ & $917$ & $\mathbf{844}$ & $868$ & $848$ & $874$ & $\mathit{909}$\\
er5000\_1 & $5000$ & $124804$ & - & $\mathbf{4201}$ & $4339$ & $4229$ & $4359$ & $\mathit{4567}$\\
ba100\_5 & $100$ & $475$ & $44$ & $40$ & $44$ & $38$ & $\mathbf{37}$ & $\mathit{46}$\\
ba200\_5 & $200$ & $975$ & $91$ & $70$ & $77$ & $\mathbf{69}$ & $\mathbf{69}$ & $\mathit{88}$\\
ba1000\_5 & $1000$ & $4975$ & $828$ & $341$ & $367$ & $\mathbf{328}$ & $330$ & $\mathit{445}$\\
ba5000\_5 & $5000$ & $24975$ & - & $1753$ & $1879$ & $1692$ & $\mathbf{1678}$ & $\mathit{2265}$\\
cora & $2708$ & $5429$ & - & $527$ & $571$ & $545$ & $\mathbf{509}$ & $\mathit{565}$\\
citeseer & $3327$ & $4552$ & - & $\mathbf{447}$ & $471$ & $457$ & $450$ & $\mathit{430}$\\
web-edu & $3031$ & $6474$ & - & $525$ & $\mathbf{479}$ & $540$ & $859$ & $\mathit{247}$\\
web-spam & $4767$ & $37375$ & - & $1514$ & $1587$ & $1486$ & $\mathbf{1461}$ & $\mathit{1601}$\\
road-minnesota & $2642$ & $3303$ & - & $\mathbf{401}$ & $523$ & $478$ & $415$ & $\mathit{316}$\\
bio-yeast & $1458$ & $1948$ & $794$ & $153$ & $168$ & $186$ & $\mathbf{144}$ & $\mathit{122}$\\
bio-SC-LC & $2004$ & $20452$ & $1596$ & $856$ & $859$ & $\mathbf{796}$ & $1591$ & $\mathit{895}$\\
rt\_damascus & $3052$ & $3881$ & $1424$ & $\mathbf{120}$ & $150$ & $132$ & $286$ & $\mathit{98}$\\
soc-wiki-vote & $889$ & $2914$ & $677$ & $205$ & $204$ & $204$ & $\mathbf{201}$ & $\mathit{227}$\\
socfb-bowdoin47 & $2252$ & $84387$ & $2066$ & $1773$ & $1741$ & $\mathbf{1619}$ & $2175$ & $\mathit{1763}$\\
\bottomrule
\end{tabular}
\end{threeparttable}
\end{table*}

\subsection{Training Convergence} \label{appendix:training-convergence}
Here, we show the learning curves of CombOpt Zero and S2V-DQN for {\sc MaxCut}. For CombOpt Zero, we chose S2V as the network. Figure \ref{fig:coz-learning-curve} and Figure \ref{fig:dqn-learning-curve} are the learning curves for CombOpt Zero and S2V-DQN respectively. As explained in Section \ref{sec:main-experiment}, we trained both CombOpt Zero and S2V-DQN for $2$ hours. The horizontal axes show the number of trajectories generated during this $2$ hours. The vertical axes correspond to the average cut size found on $100$ fixed random ER graphs of size $50$. Since CombOpt Zero starts to log the best models after $15$ minutes from the beginning, and updates the log each $5$ minutes after that, data for the first $15$ minutes is not shown on the chart. Unlike S2V-DQN, if the newly generated models do not perform better than the current one, the best model is not updated and hence the learning curve remains flat for some intervals.

While S2V-DQN generated about $225000$ data in two hours on a single GPU, CombOpt Zero generated only $5000$ data on four GPUs.
This is because the MCTS, which CombOpt Zero executes to generate self-play data, takes time.
On the other hand, S2V-DQN requires about $50000$ data for training to converge, while CombOpt Zero requires only about $2000$ data, meaning that the training of CombOpt Zero is much more sample-efficient.

\begin{figure}[t]
\centering
\includegraphics[width=7cm]{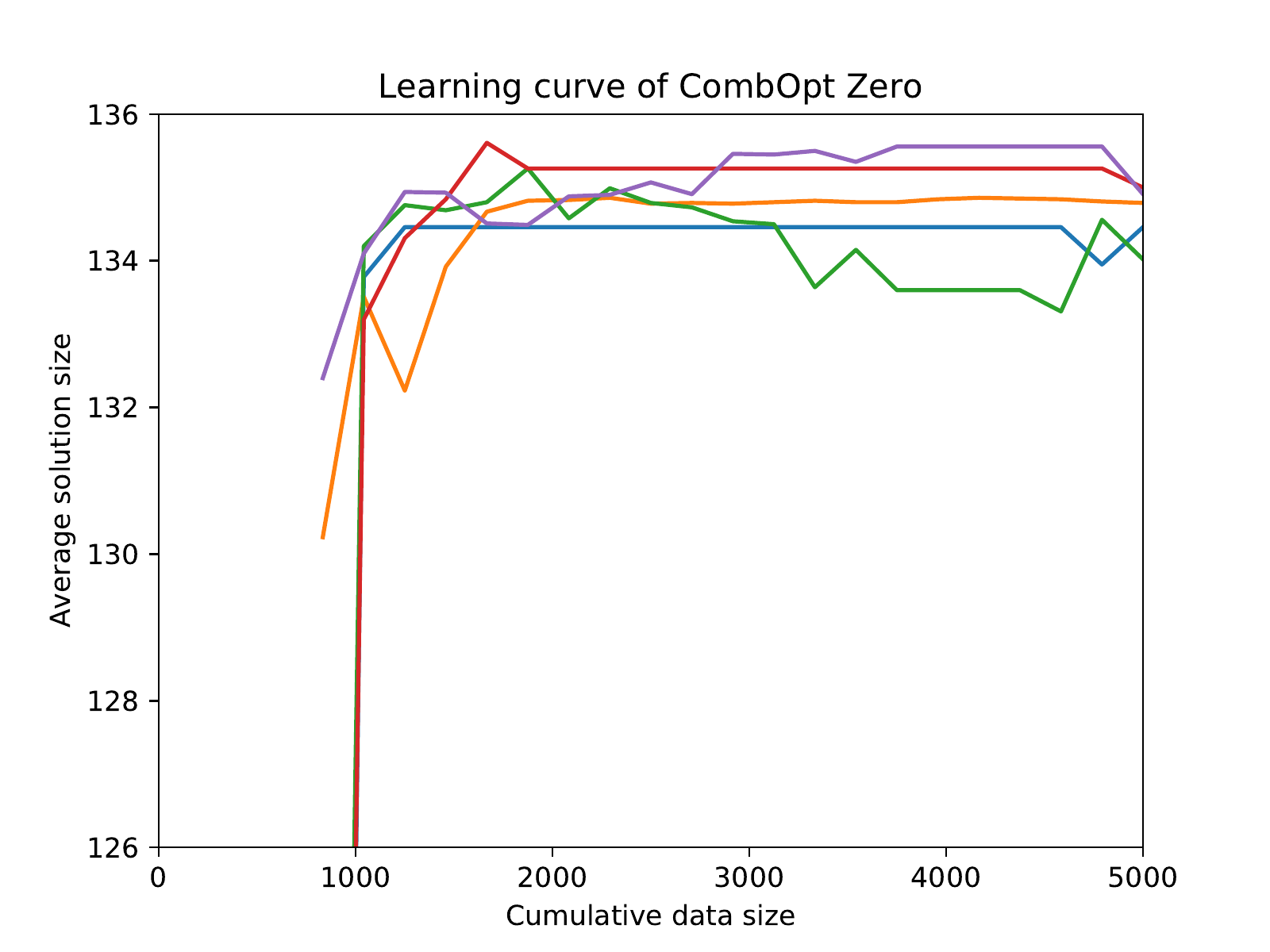}
\caption{\textbf{Training curve for CombOpt Zero on {\sc MaxCut}.}
Data for five models with the same hyperparameters are shown.
The horizontal axis is the cumulative number of generated trajectories during $2$ hours.
Since, CombOpt Zero saves the best models after $15$ minutes from start, its learning curve is not shown for the first $15$ minutes.}
\label{fig:coz-learning-curve}
\end{figure}

\begin{figure}[t]
\centering
\includegraphics[width=7cm]{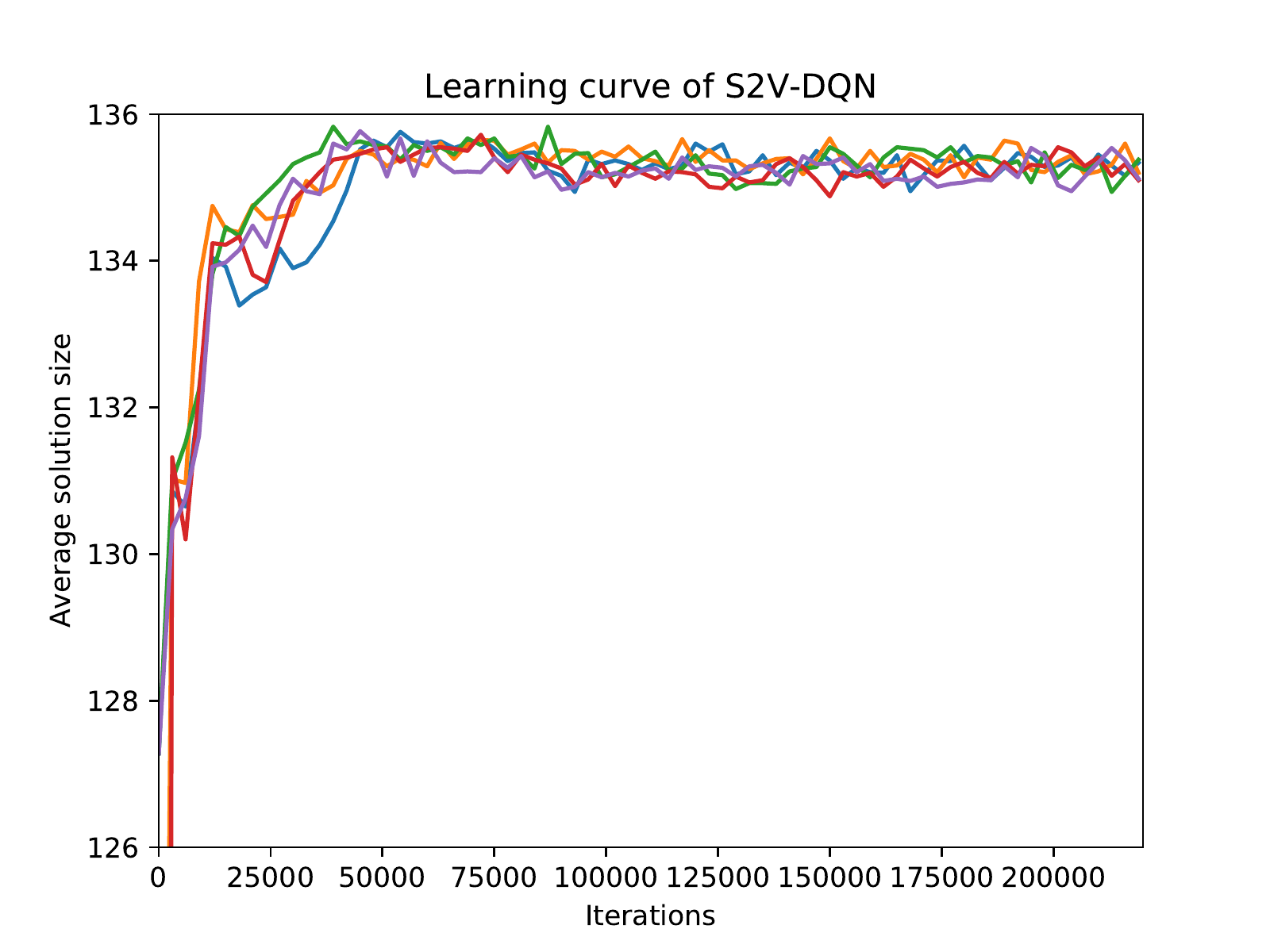}
\caption{\textbf{Training curve for S2V-DQN on {\sc MaxCut}.} Data for five models with the same hyperparameters are shown. The horizontal axis shows the total number of generated trajectories during $2$ hours (on average, S2V-DQN generates one trajectory per iteration).}
\label{fig:dqn-learning-curve}
\end{figure}

\section{Comparison of Reduction Rules} \label{compare-reduction-rules}
Here, we discuss some differences of reinforcement learning formulation between S2V-DQN and CombOpt Zero.
The first difference is their states.
While S2V-DQN's state keeps both the original graph and selected nodes, CombOpt Zero's state is a single current (labeled) graph and its size usually gets smaller as the state proceeds.
Since the GNN inference of smaller graphs is faster, it's more efficient to give small graphs as an input.
Although S2V-DQN implicitly reduces the graph size in most of its implementation, our formulation is more explicit.

The second difference is that while S2V-DQN limits action space to a selection of one vertex, CombOpt Zero does not.
Thanks to this, as stated in \ref{MDP}, {\sc MaxCut} can be formulated as a node coloring by two colors with more intuitive termination criteria: finish coloring all the nodes.
This also allows the application to a wider range of problems.
For example, {\sc K-Coloring} by defining the action space by $\{(x, c) \mid x \in V, c \in \mathbb{N}, 0 \le c \le K \}$, where $(x, c)$ represents the coloring of node $x$ by color $c$.

\section{Visualization} \label{visualization}
In this section, we illustrate how CombOpt Zero finds solutions for {\sc MaxCut}.
As described in Section \ref{sec:main-experiment}, CombOpt Zero found an optimal solution for trees and its overall performance was comparable to the state-of-the-art heuristic solvers.
Figure \ref{fig:maxcut-er-vis} shows the sequence of actions by CombOpt Zero on an ER graph.
Starting from one node, CombOpt Zero colors surrounding nodes one by one with the opposite color.
Figure \ref{fig:maxcut-tree-vis} shows the sequence of actions on a tree.
The order of actions was similar to the order of visiting nodes in the depth-first-search.
Since an optimal two coloring for {\sc MaxCut} on a tree can be obtained by the depth-first-search, we can say that CombOpt Zero successfully learned the effective algorithm for {\sc MaxCut} on a tree from the training on random graphs.

It is also interesting that CombOpt Zero sometimes skips the neighbors and colors a two-hop node by the same color as the current node.
This is possible because the $L$-layer ($5$ in our case) message passing GNN can catch information of the $L$-hop neighbors.
This flexibility possibly affected the good performance on other graph instances than trees.

\begin{figure*}[t]
    \centering
    \begin{tabular}{@{}cccc@{}}
        \includegraphics[width=0.2\textwidth]{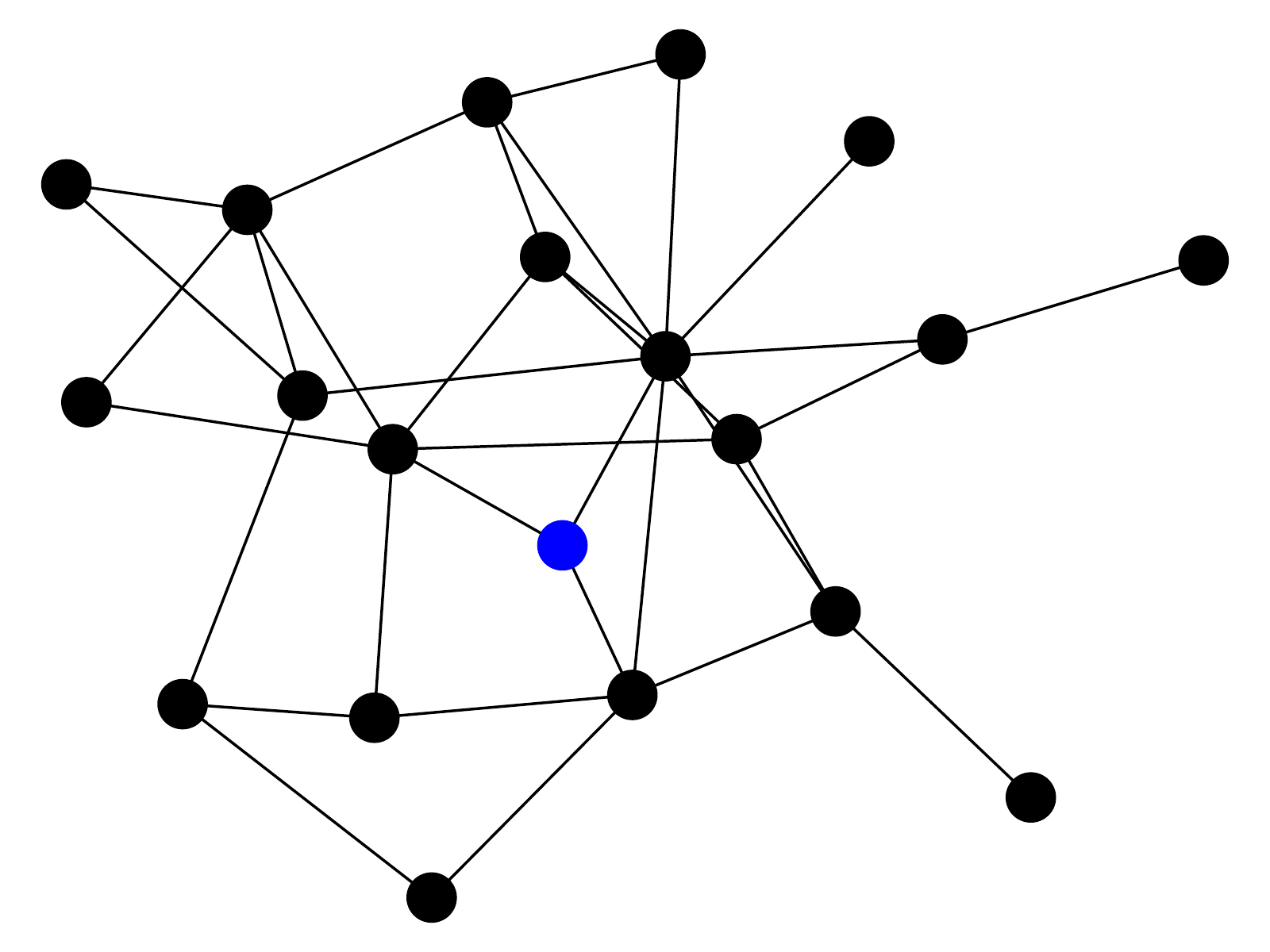} &
        \includegraphics[width=0.2\textwidth]{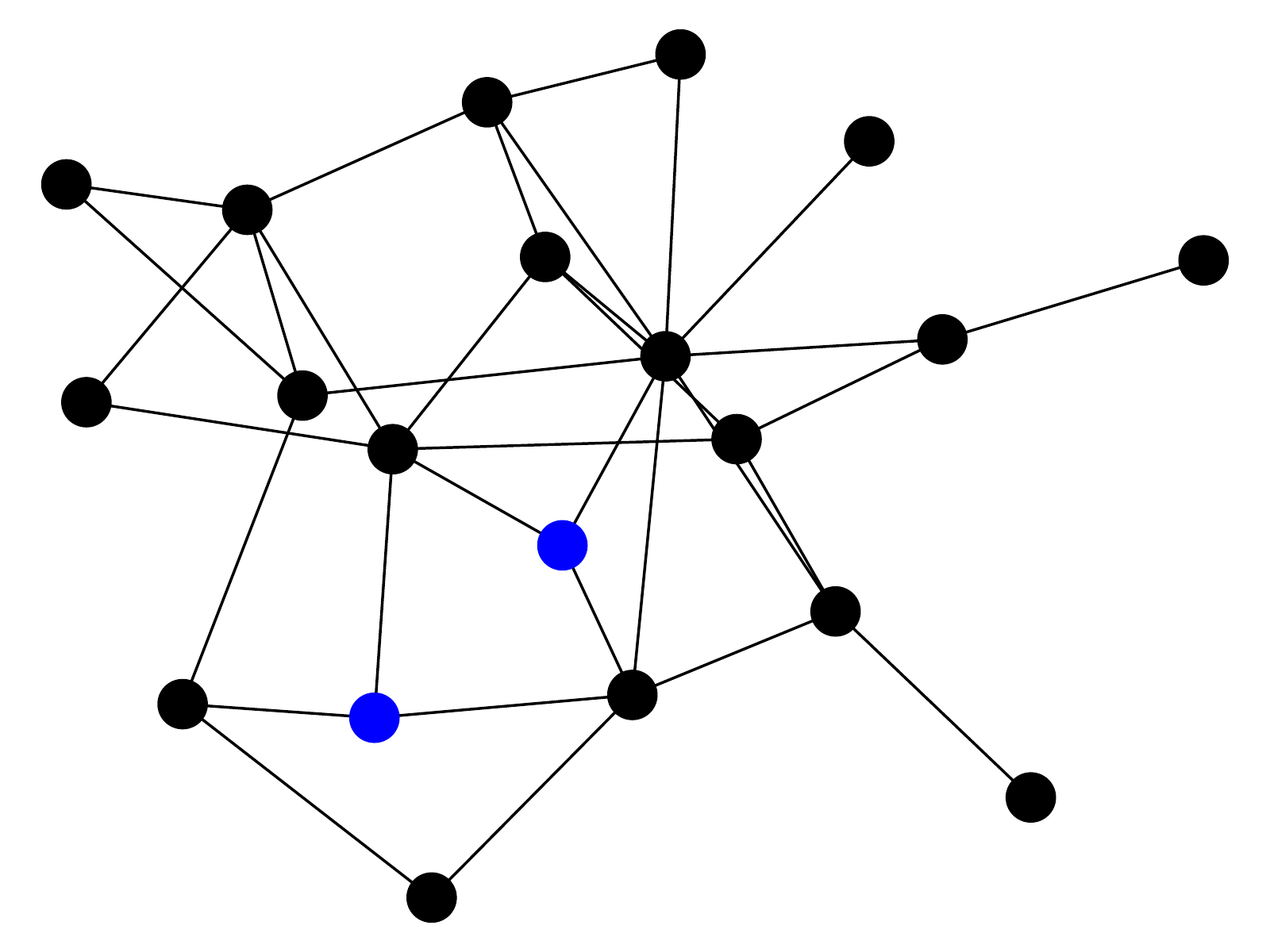} &
        \includegraphics[width=0.2\textwidth]{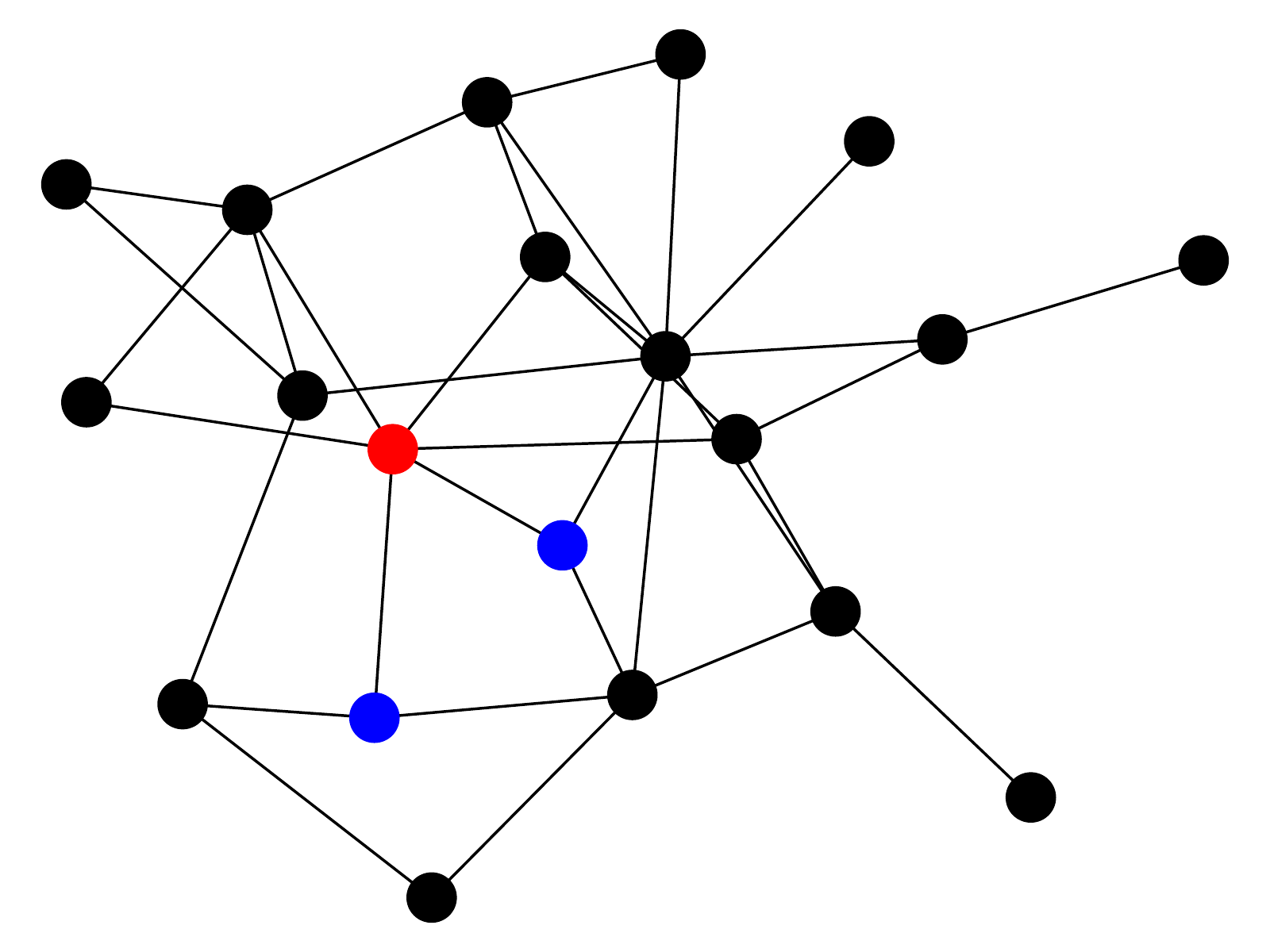} &
        \includegraphics[width=0.2\textwidth]{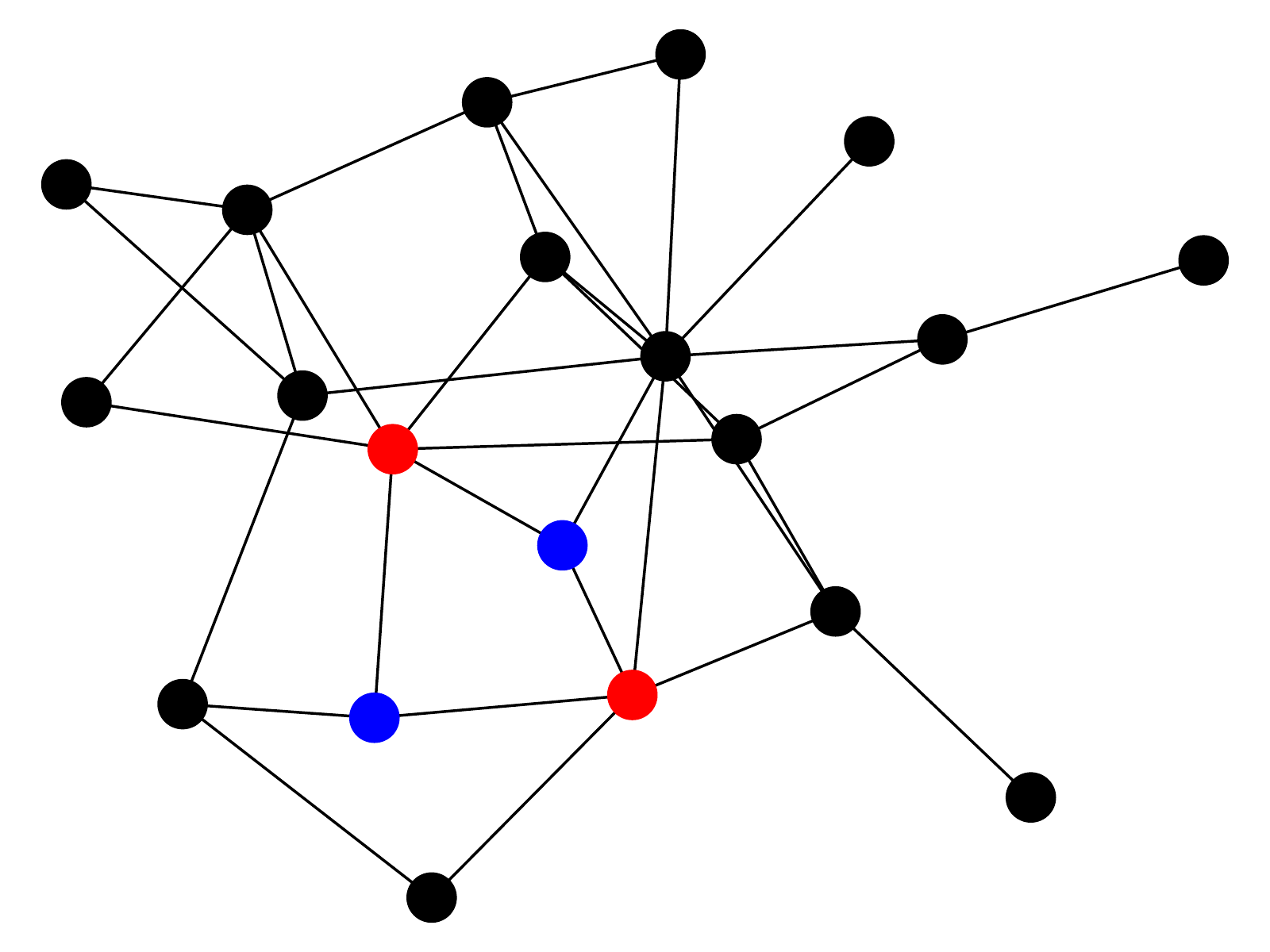} \\
        \includegraphics[width=0.2\textwidth]{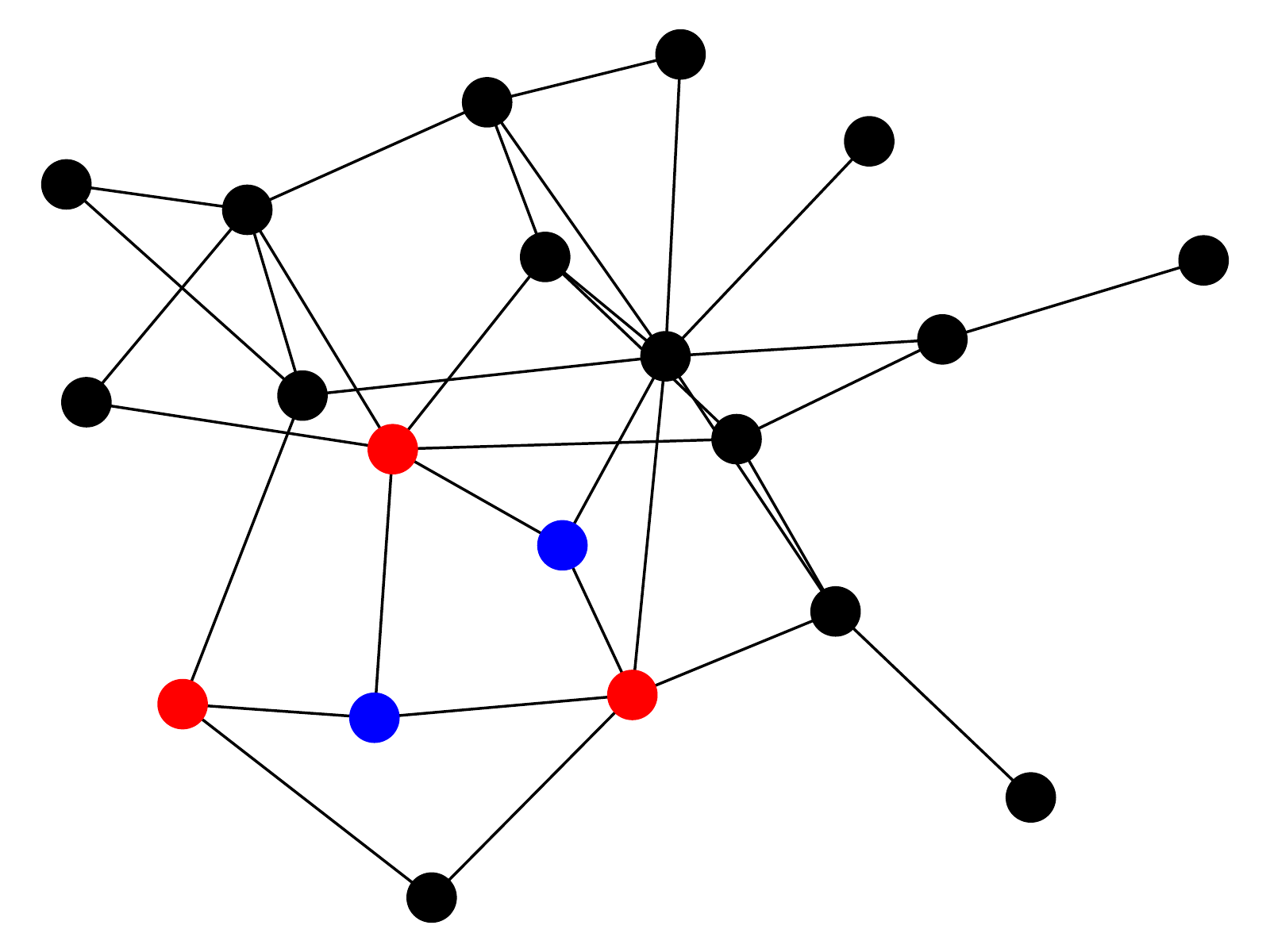} &
        \includegraphics[width=0.2\textwidth]{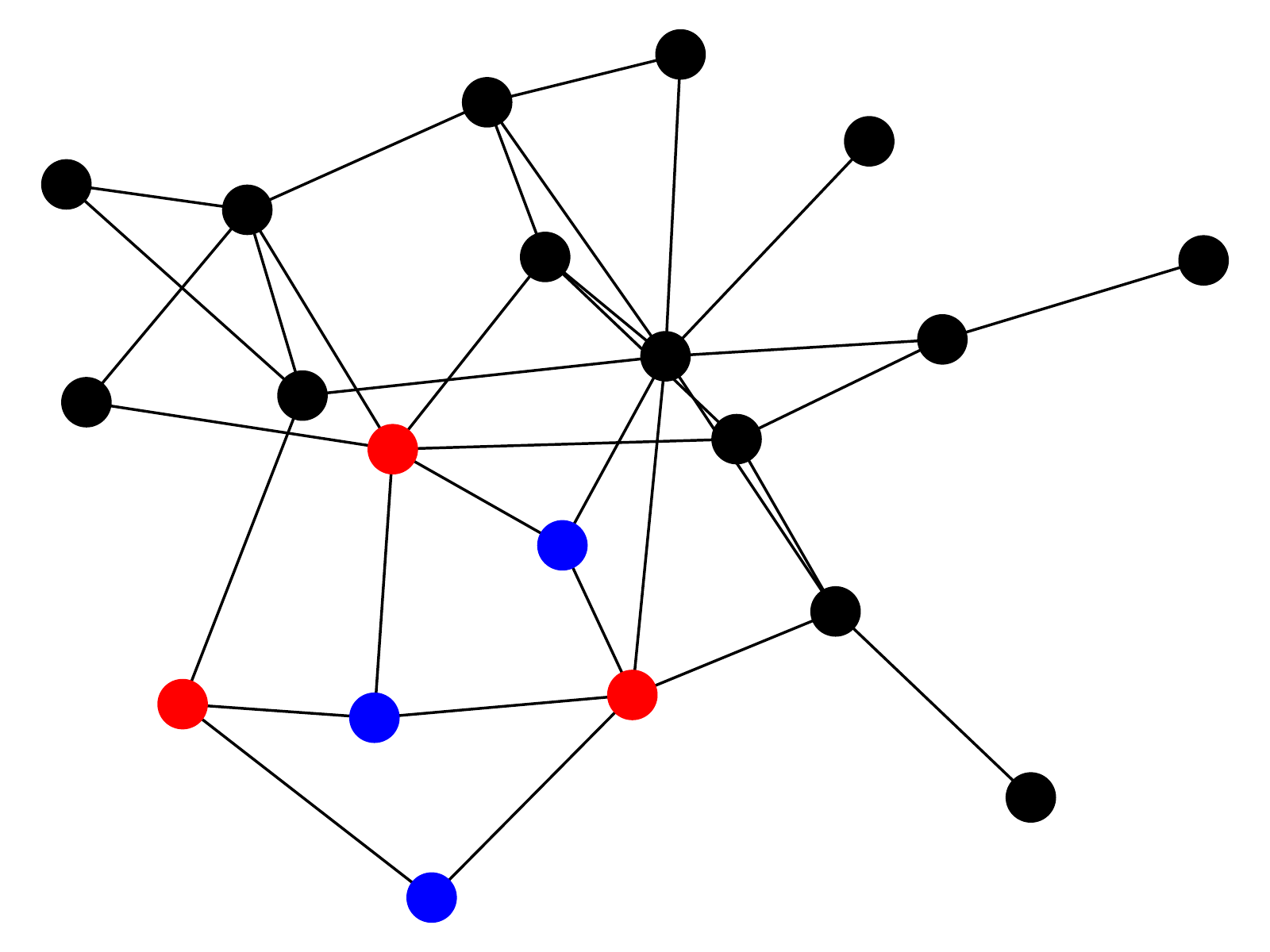} &
        \includegraphics[width=0.2\textwidth]{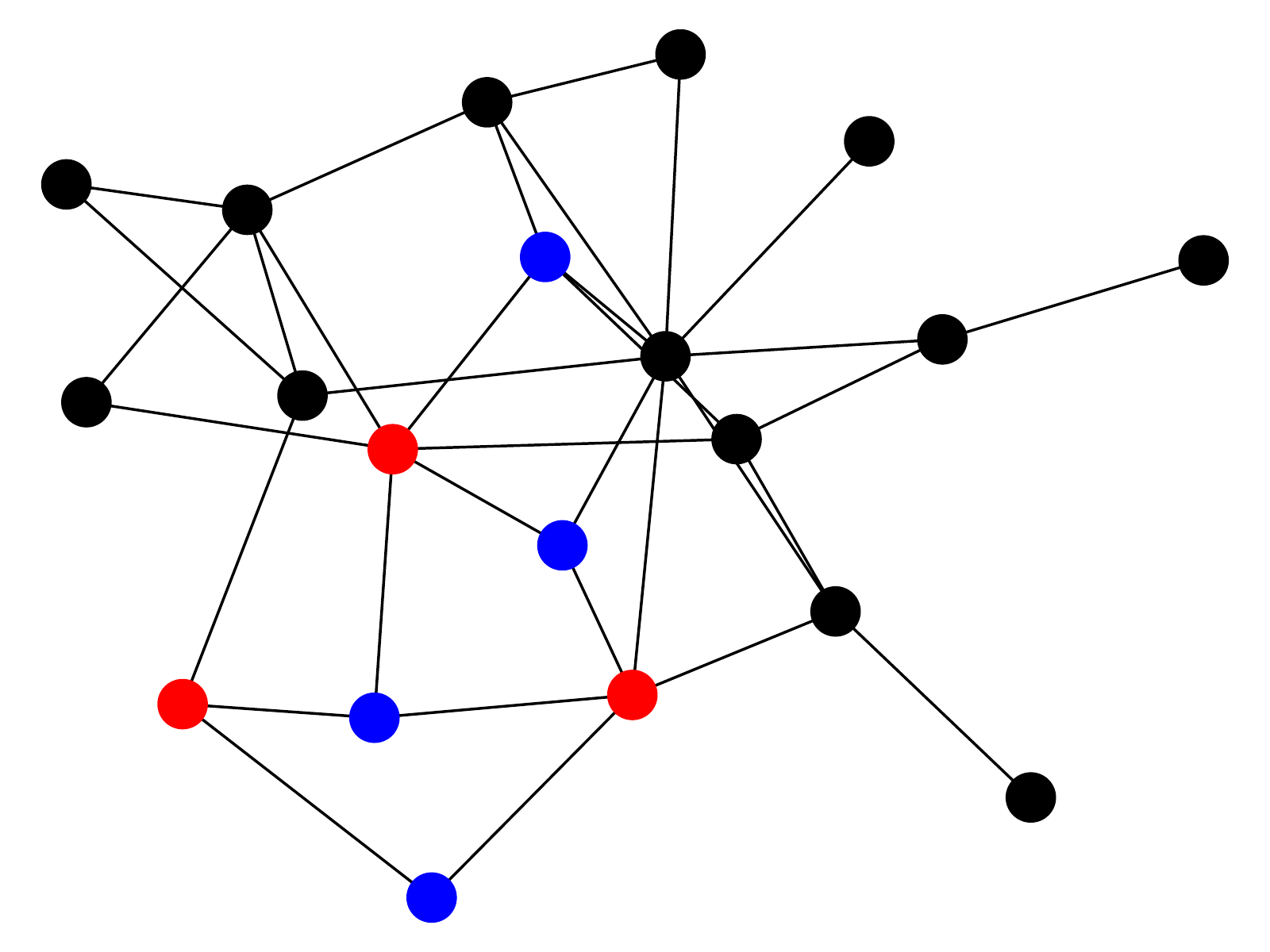} &
        \includegraphics[width=0.2\textwidth]{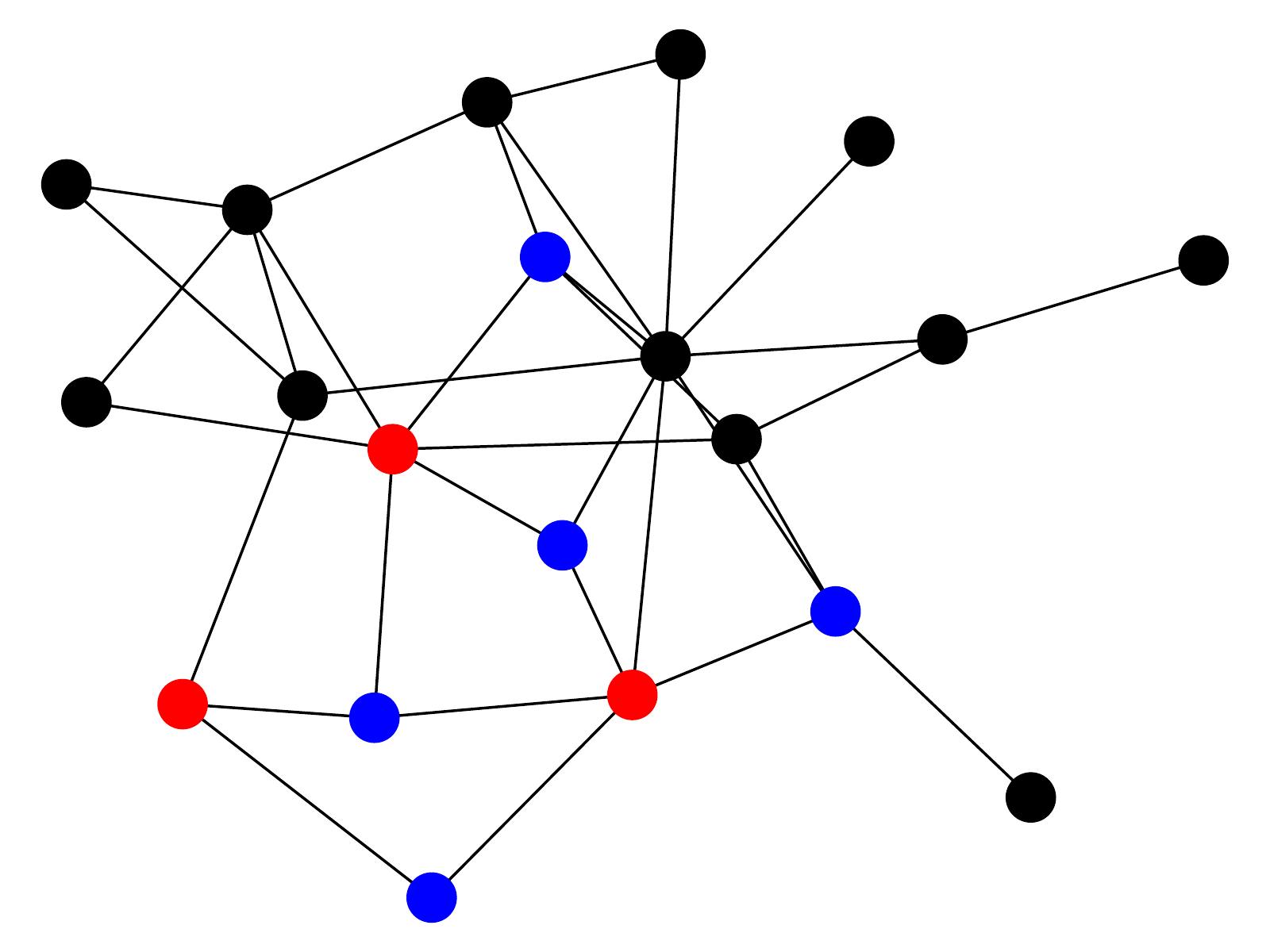} \\
        \includegraphics[width=0.2\textwidth]{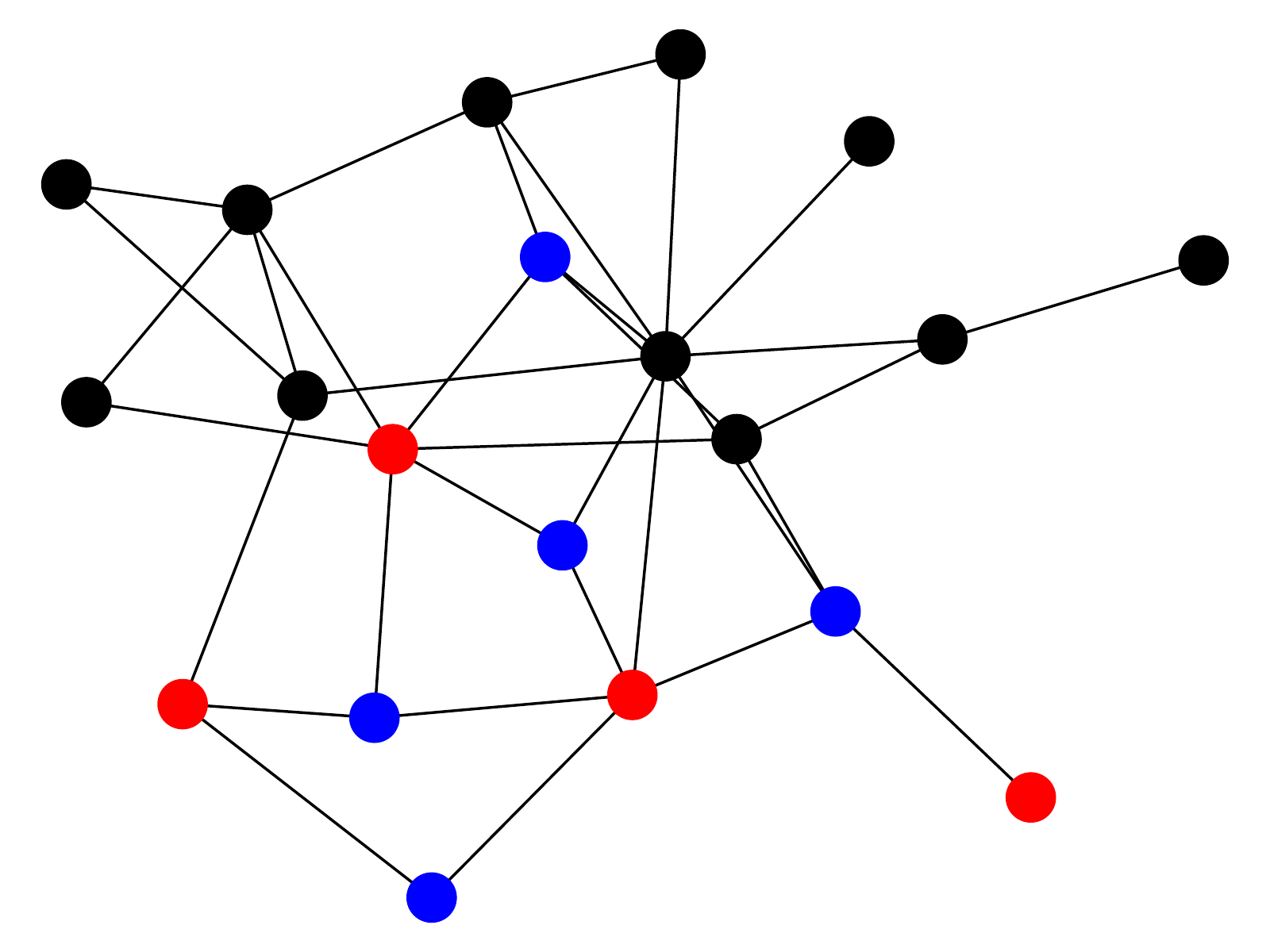} &
        \includegraphics[width=0.2\textwidth]{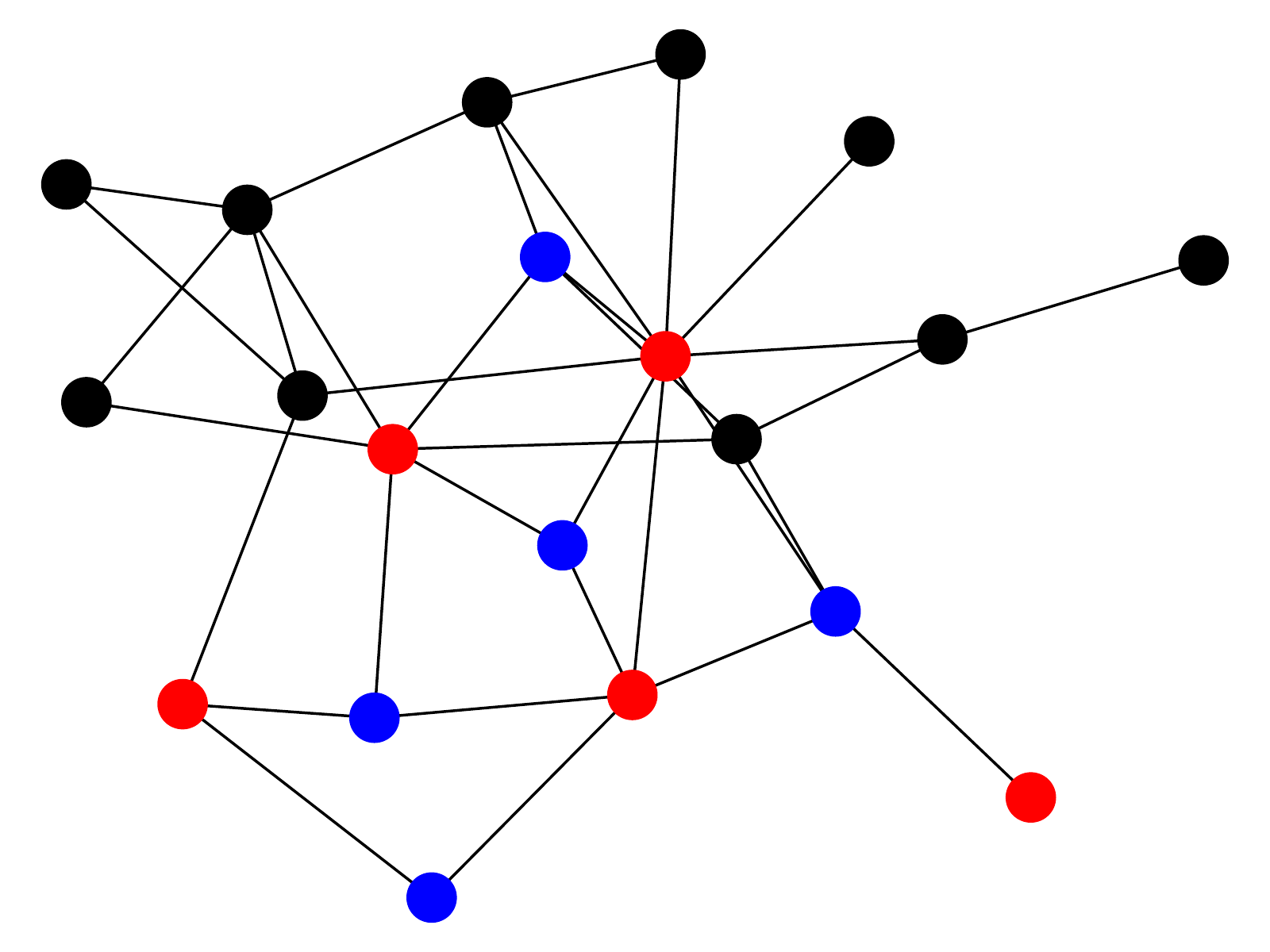} &
        \includegraphics[width=0.2\textwidth]{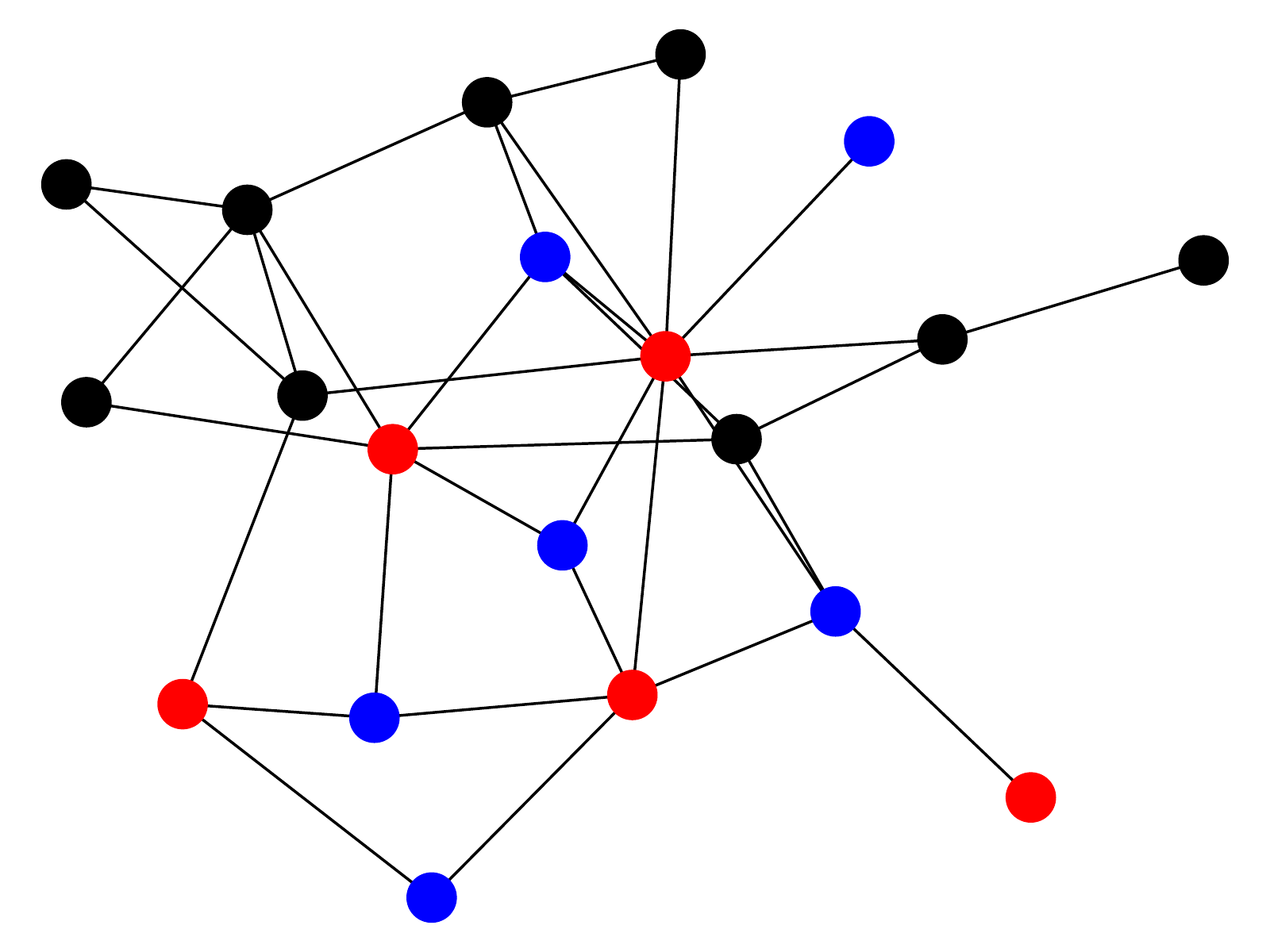} &
        \includegraphics[width=0.2\textwidth]{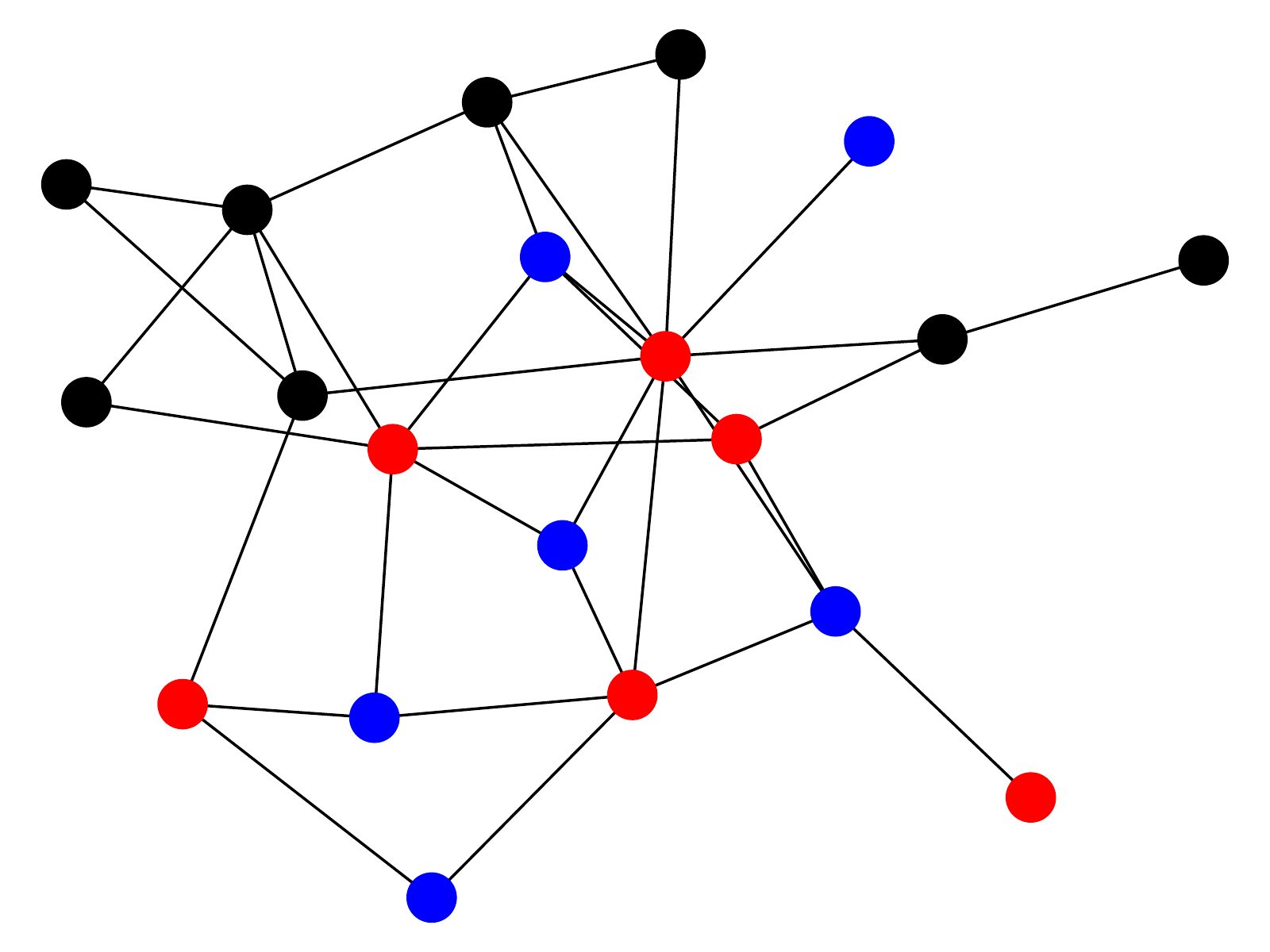} \\
        \includegraphics[width=0.2\textwidth]{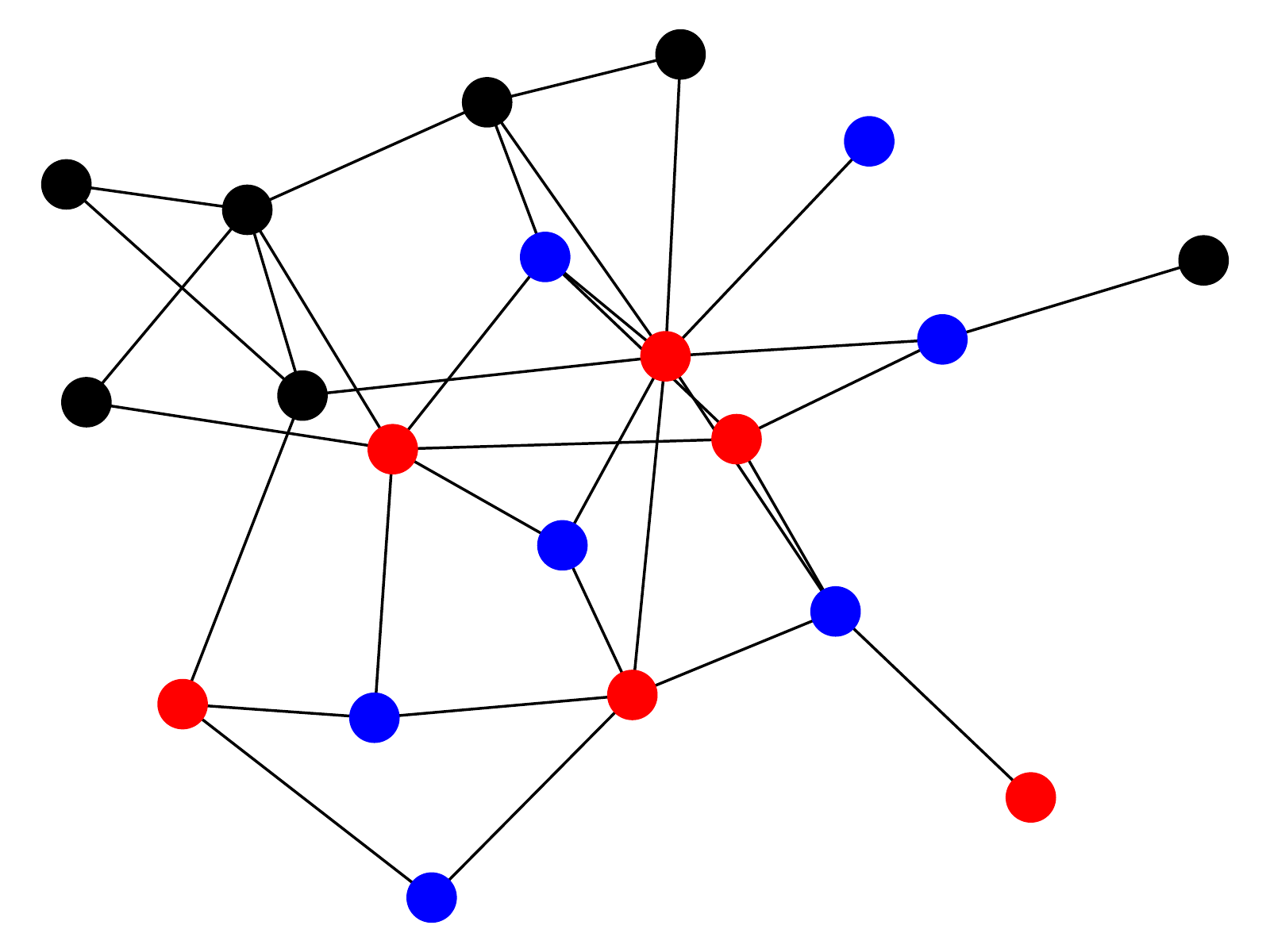} &
        \includegraphics[width=0.2\textwidth]{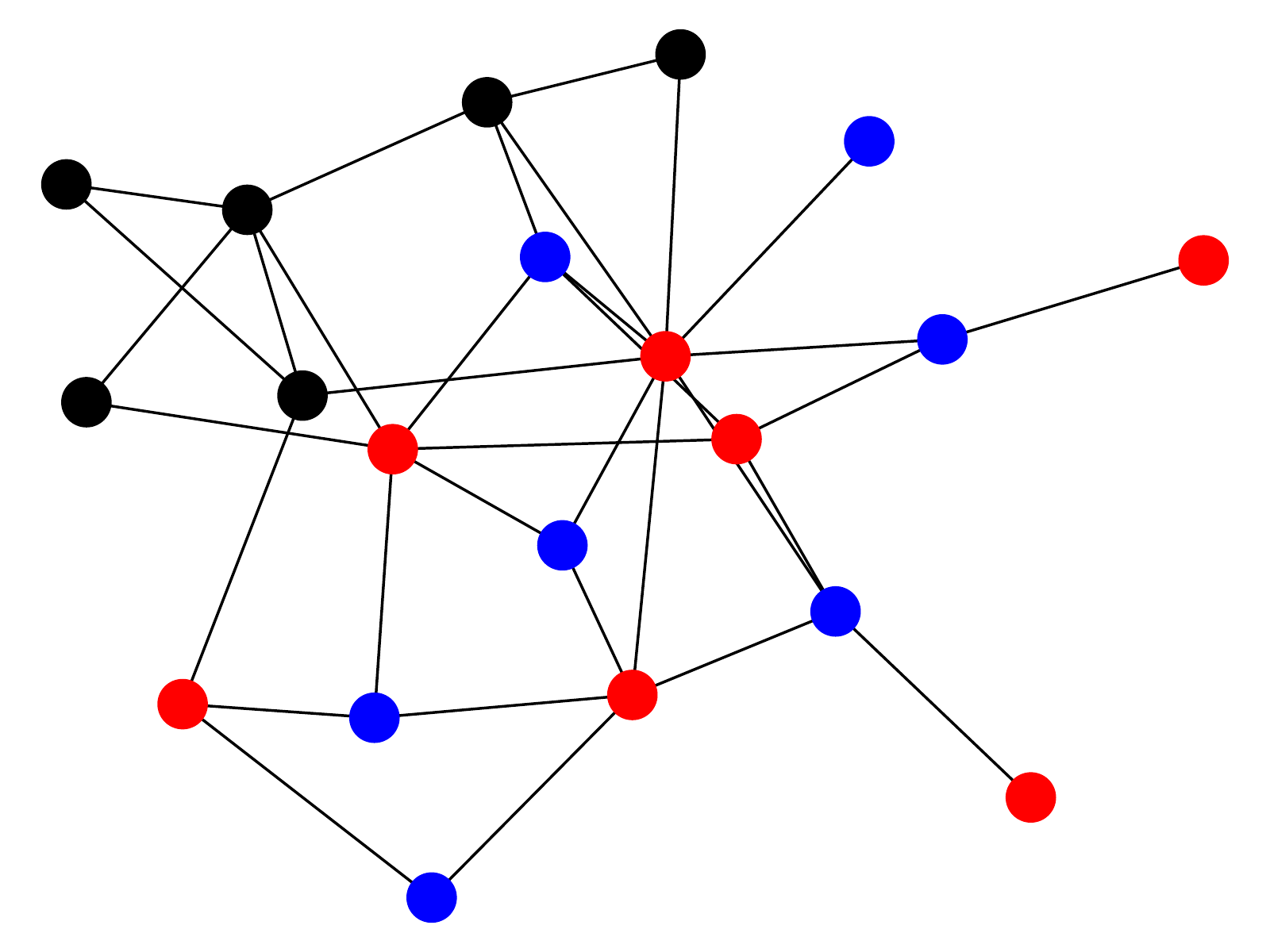} &
        \includegraphics[width=0.2\textwidth]{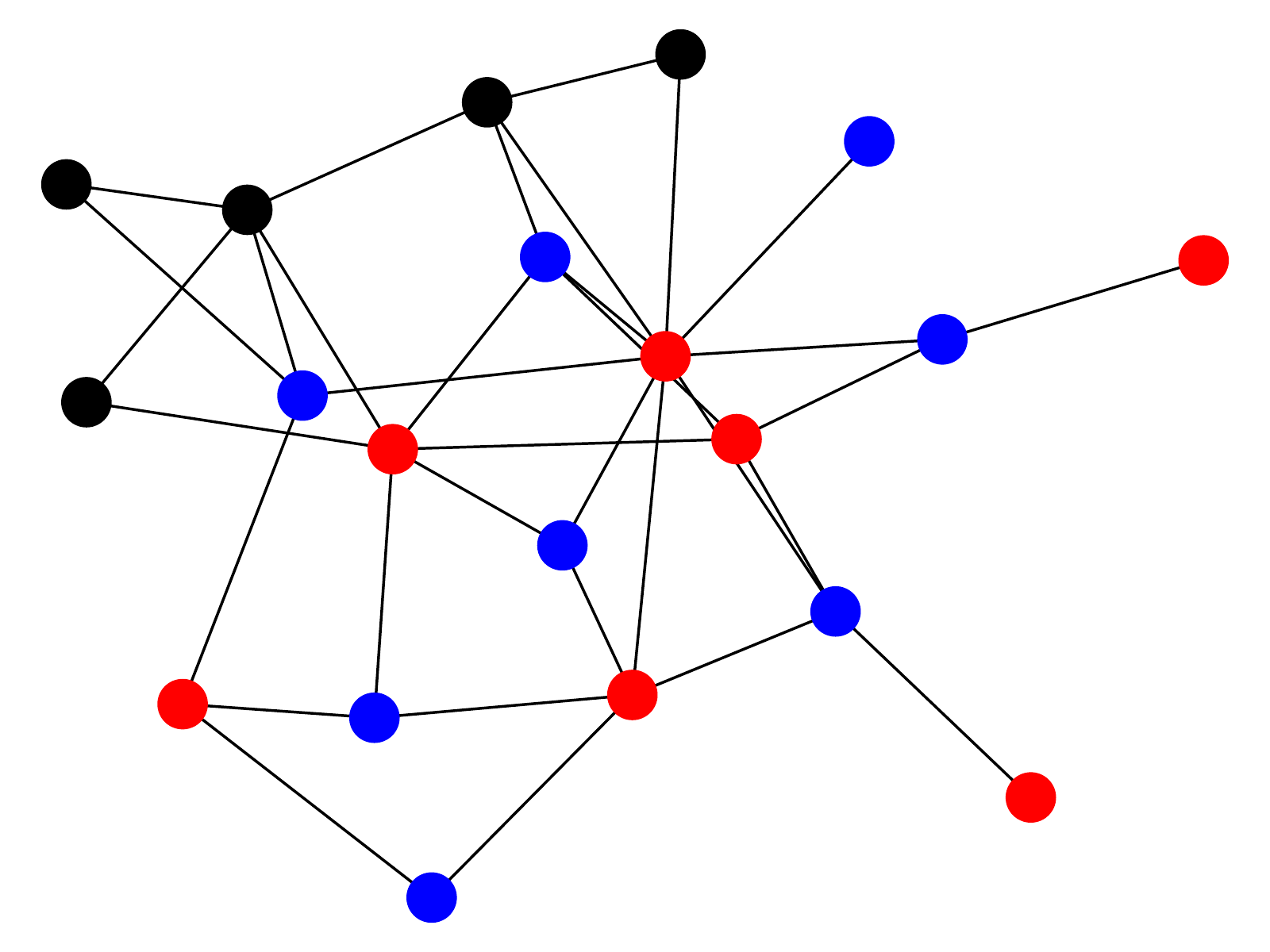} &
        \includegraphics[width=0.2\textwidth]{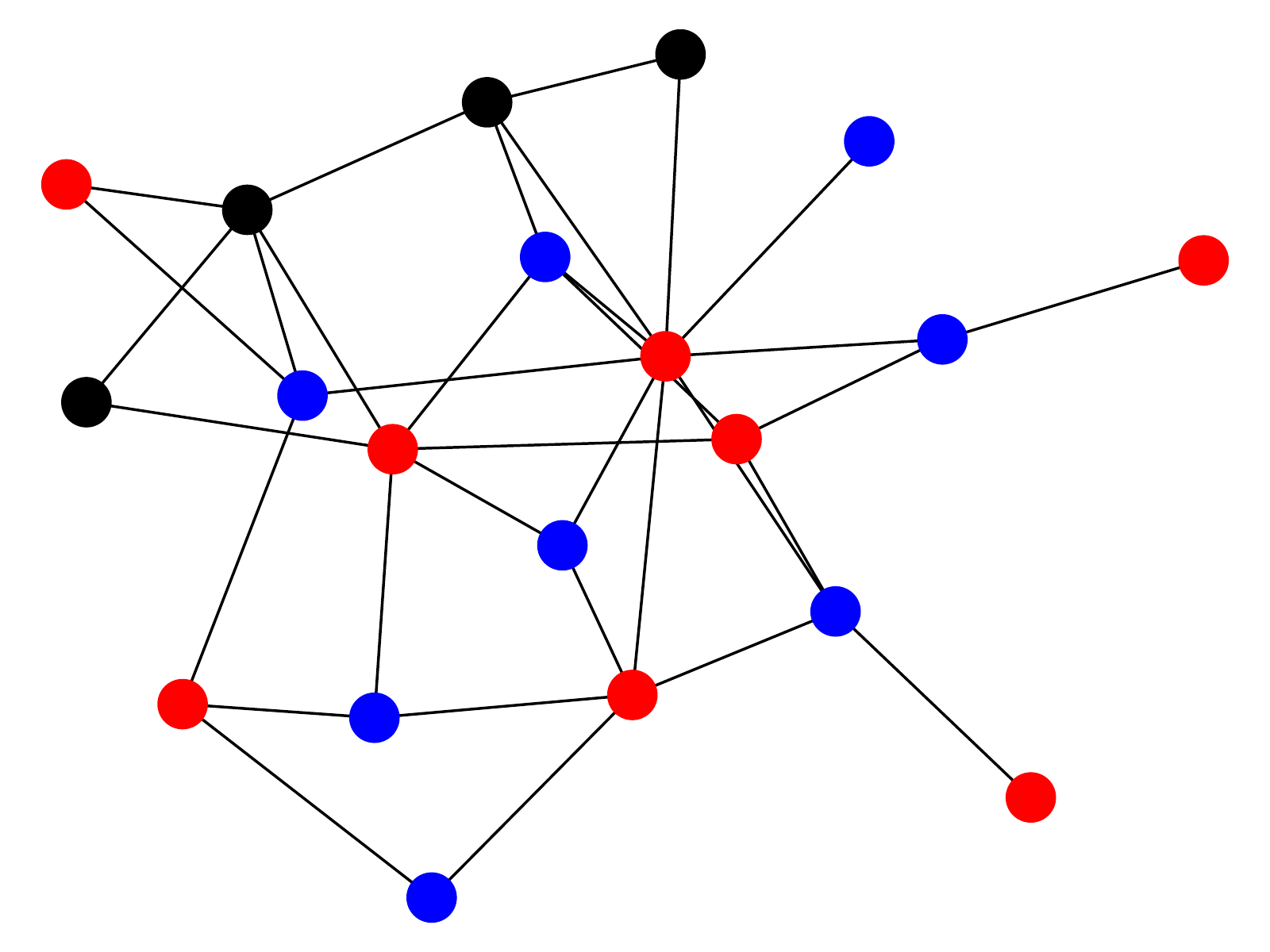} \\
        \includegraphics[width=0.2\textwidth]{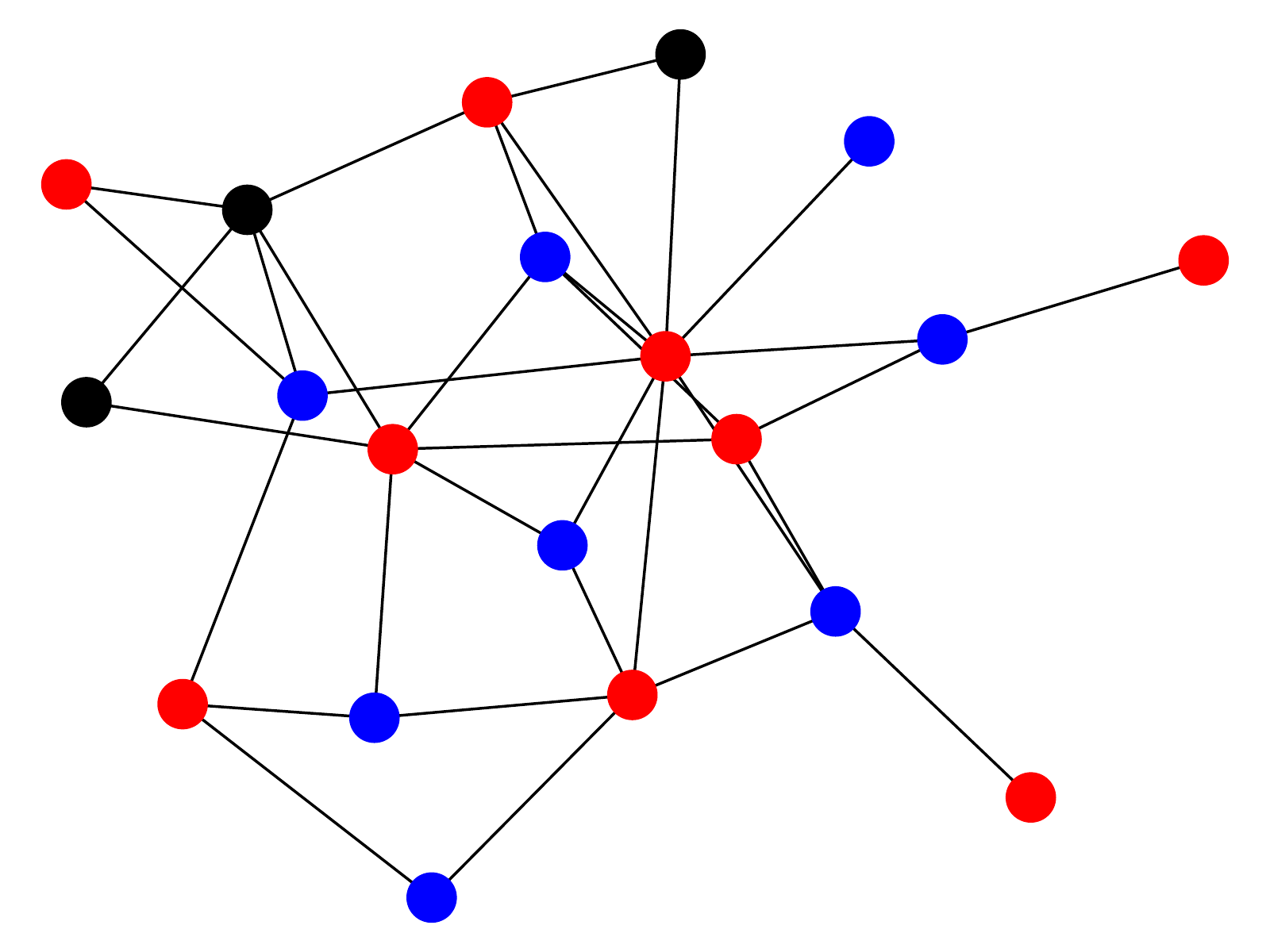} &
        \includegraphics[width=0.2\textwidth]{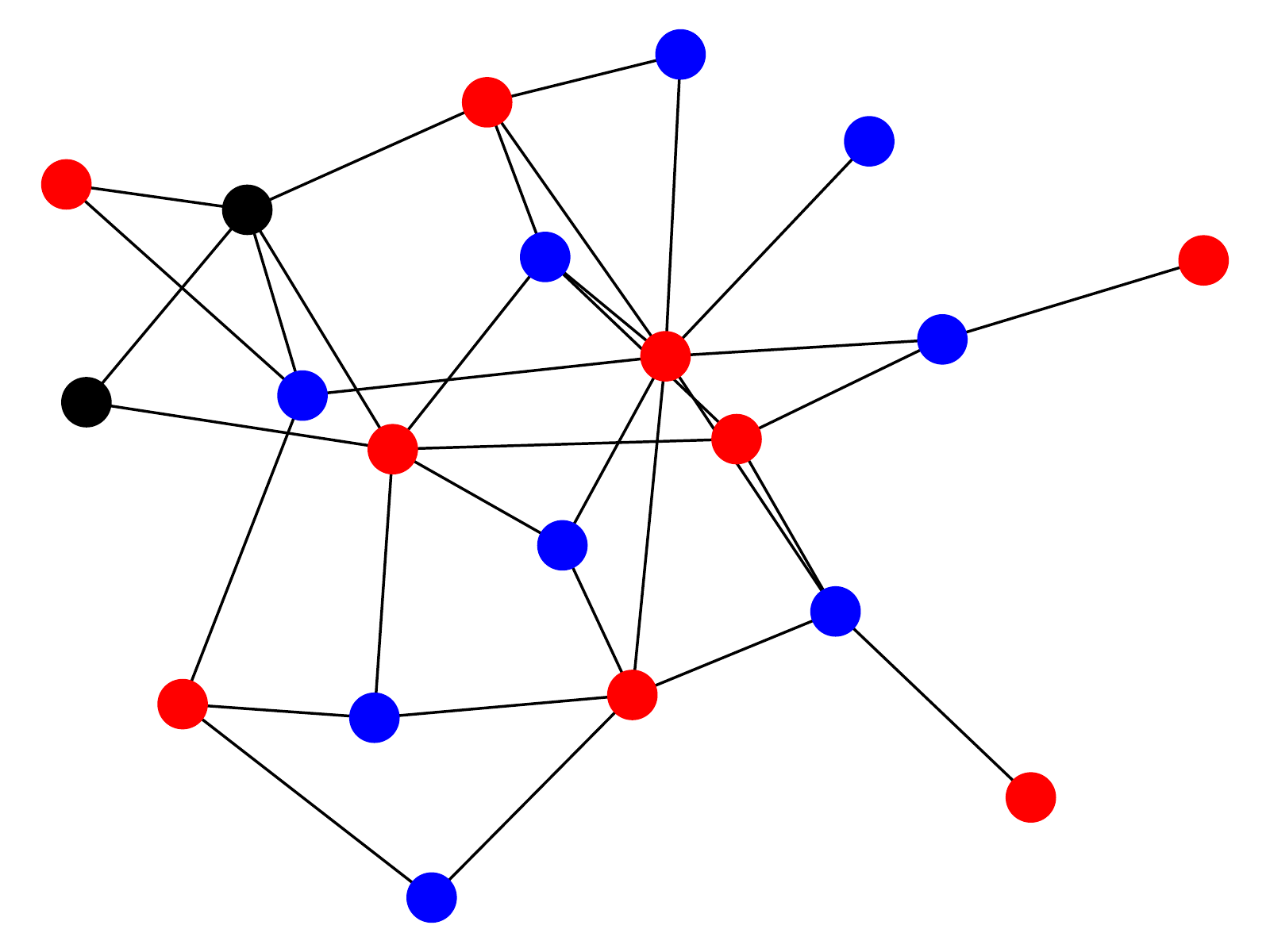} &
        \includegraphics[width=0.2\textwidth]{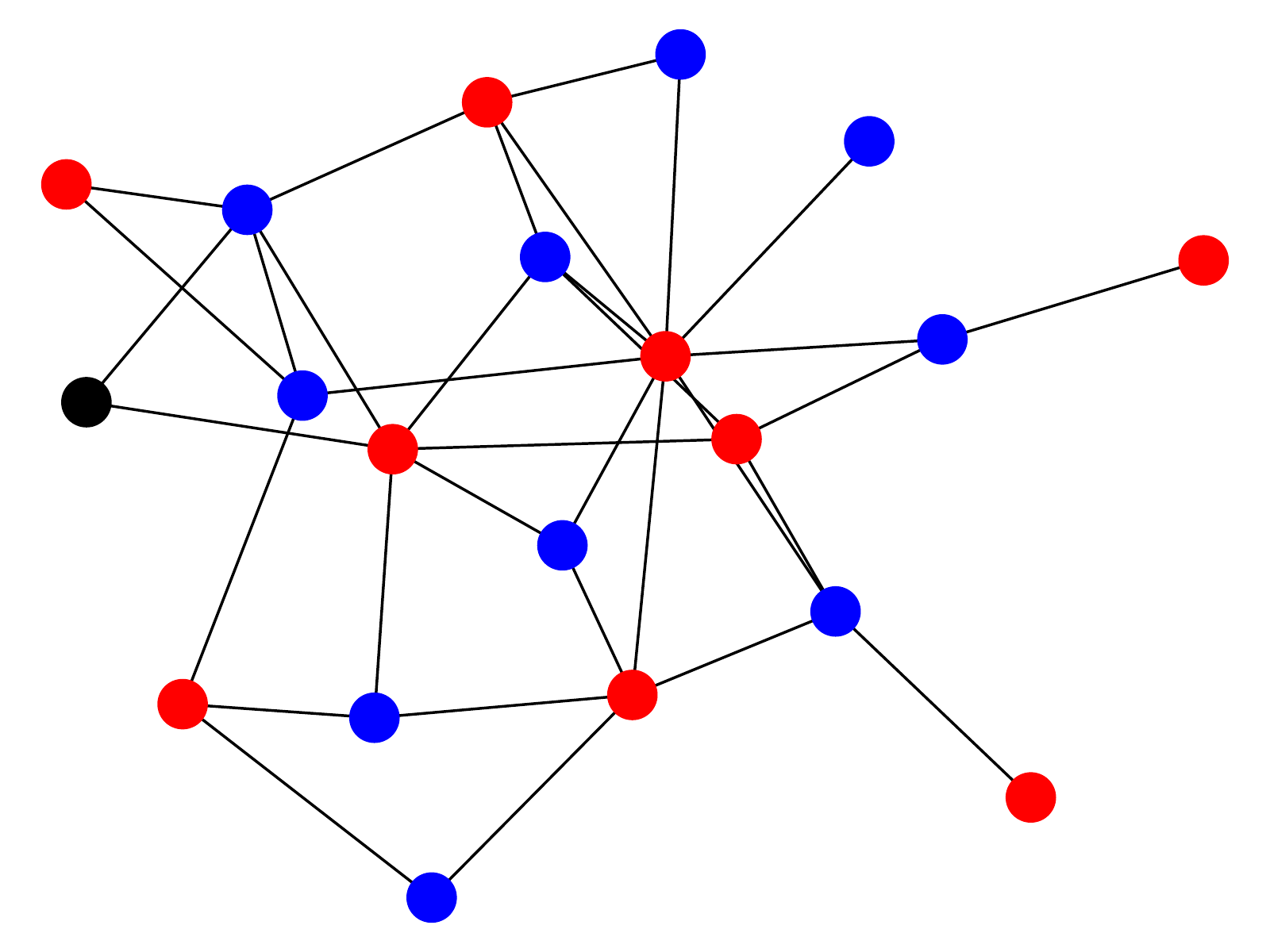} &
        \includegraphics[width=0.2\textwidth]{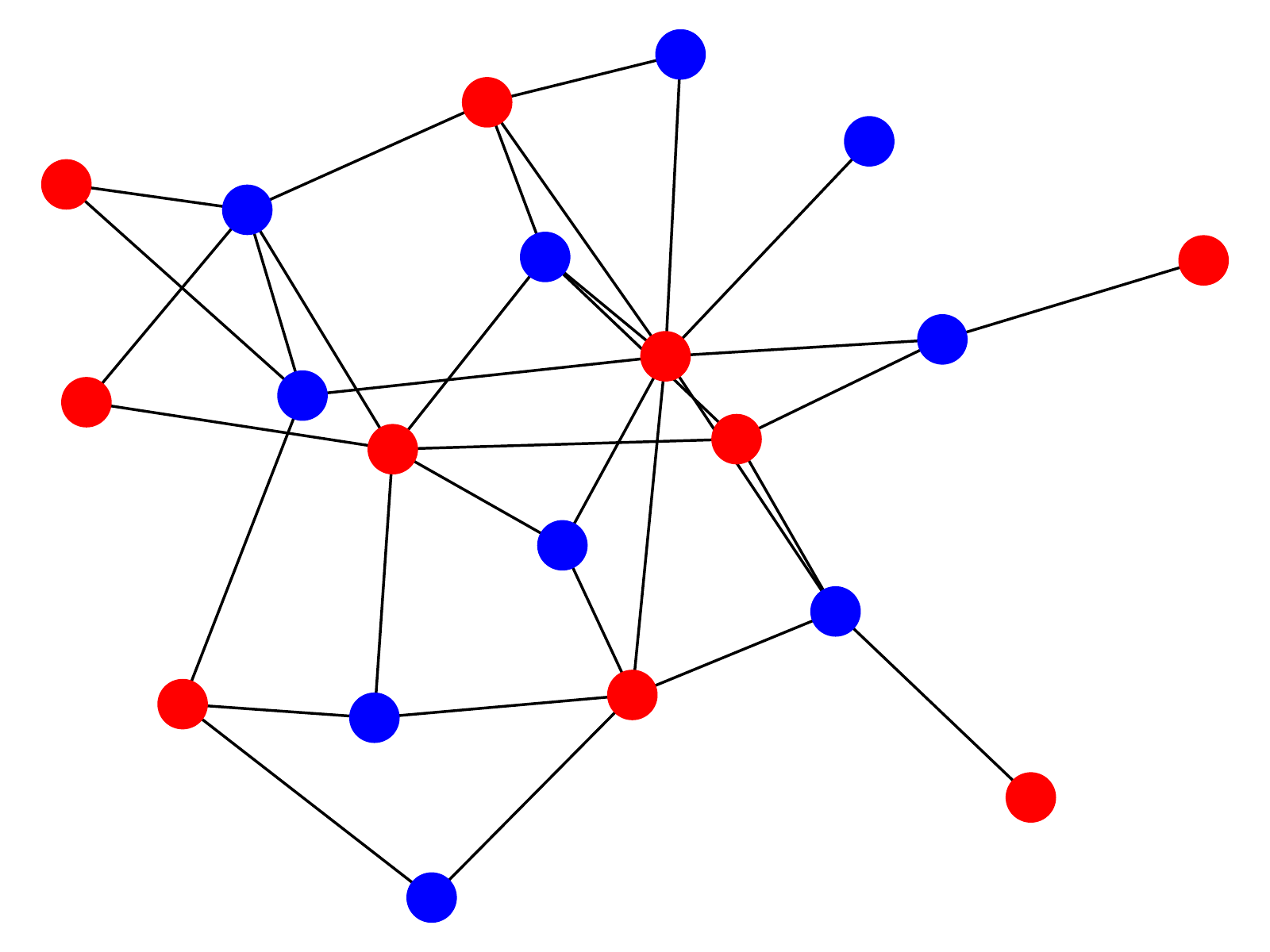} \\
    \end{tabular}
    \caption{\textbf{Order of actions of CombOpt Zero for {\sc MaxCut} on a ER graph.} Although the order of selecting nodes and the coloring is arbitrary, CombOpt Zero learned to color neighbors one by one with the opposite color.}
    \label{fig:maxcut-er-vis}
\end{figure*}

\begin{figure*}[t]
    \centering
    \begin{tabular}{@{}cccc@{}}
        \includegraphics[width=0.2\textwidth]{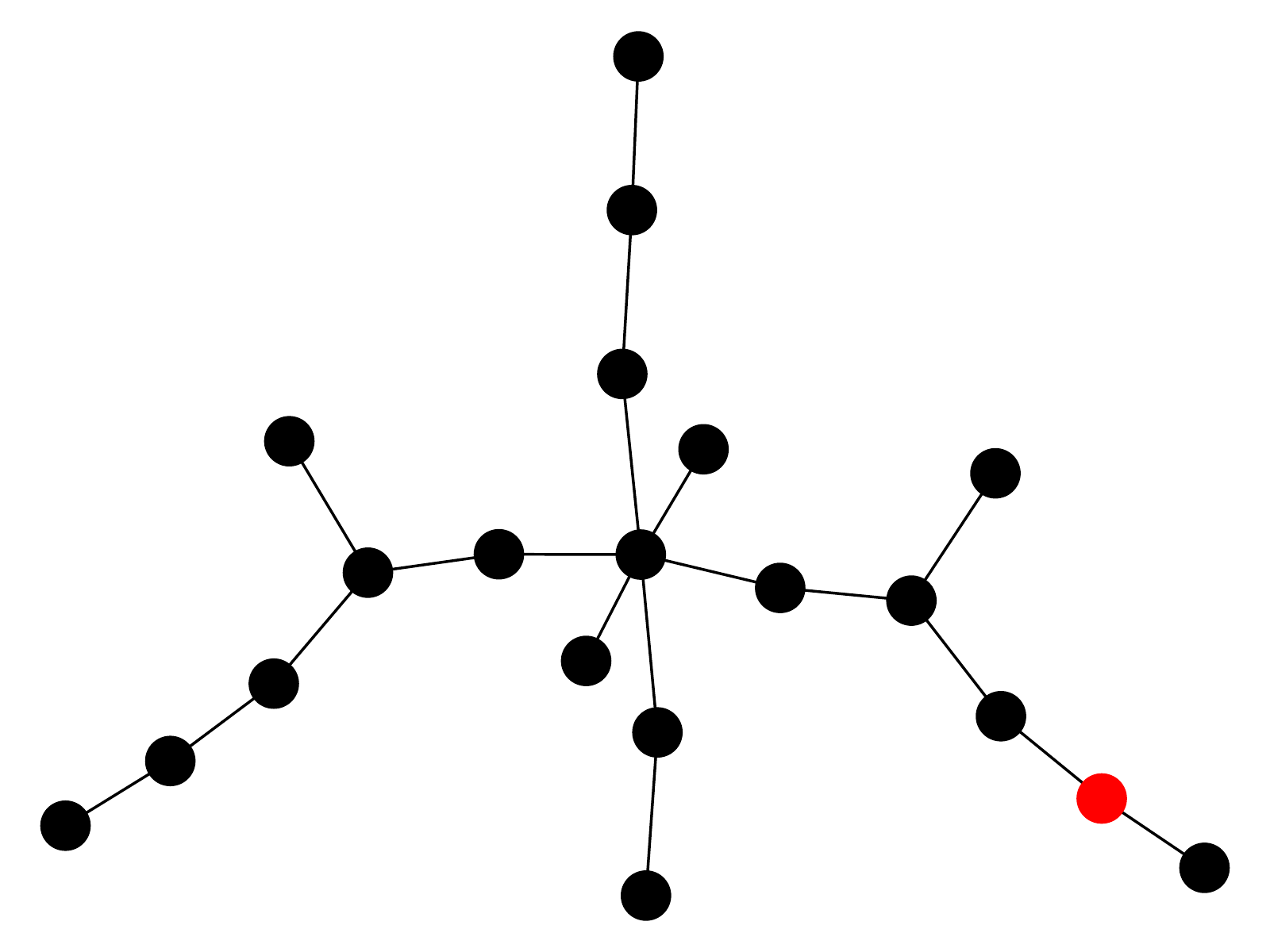} &
        \includegraphics[width=0.2\textwidth]{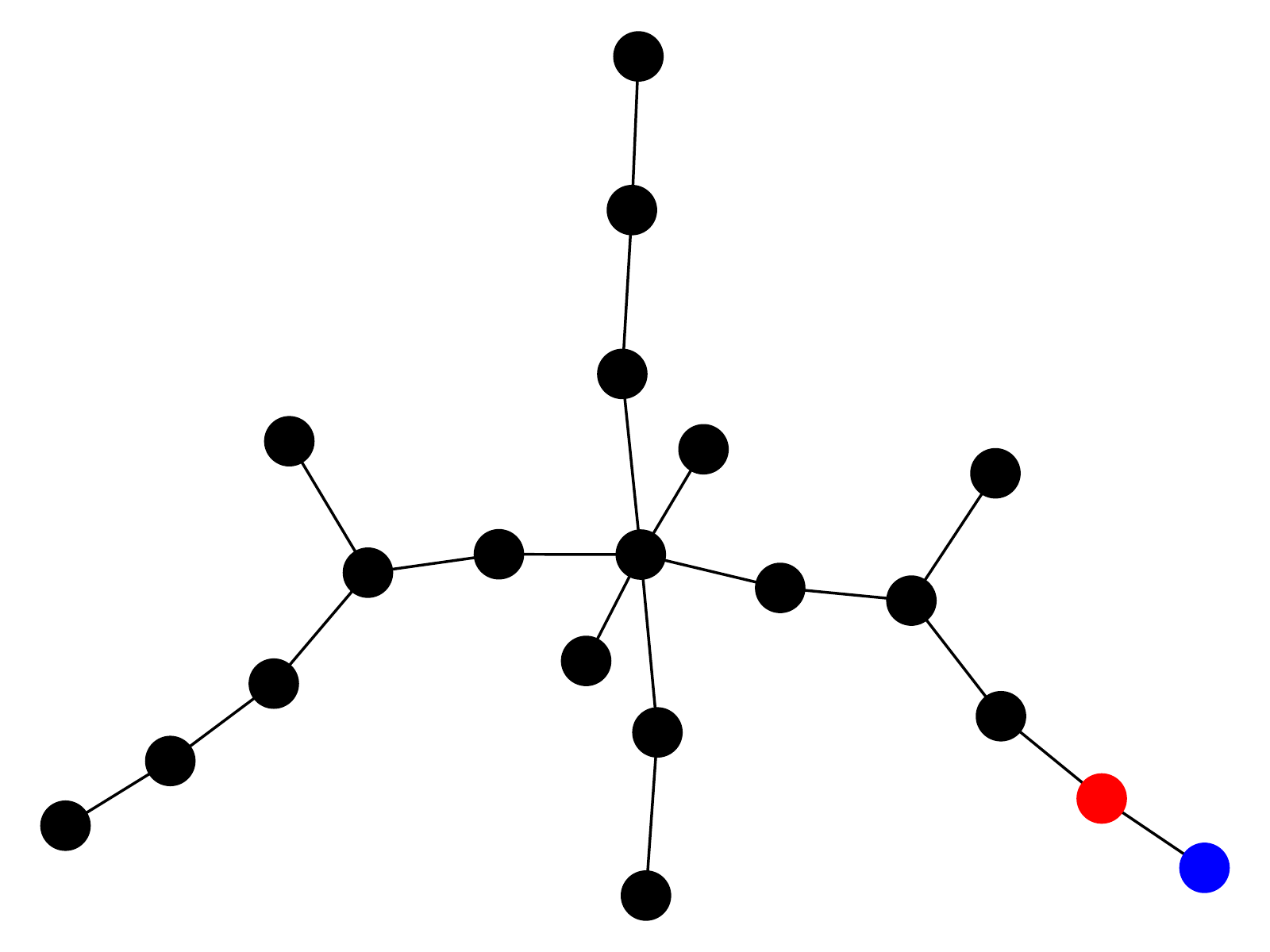} &
        \includegraphics[width=0.2\textwidth]{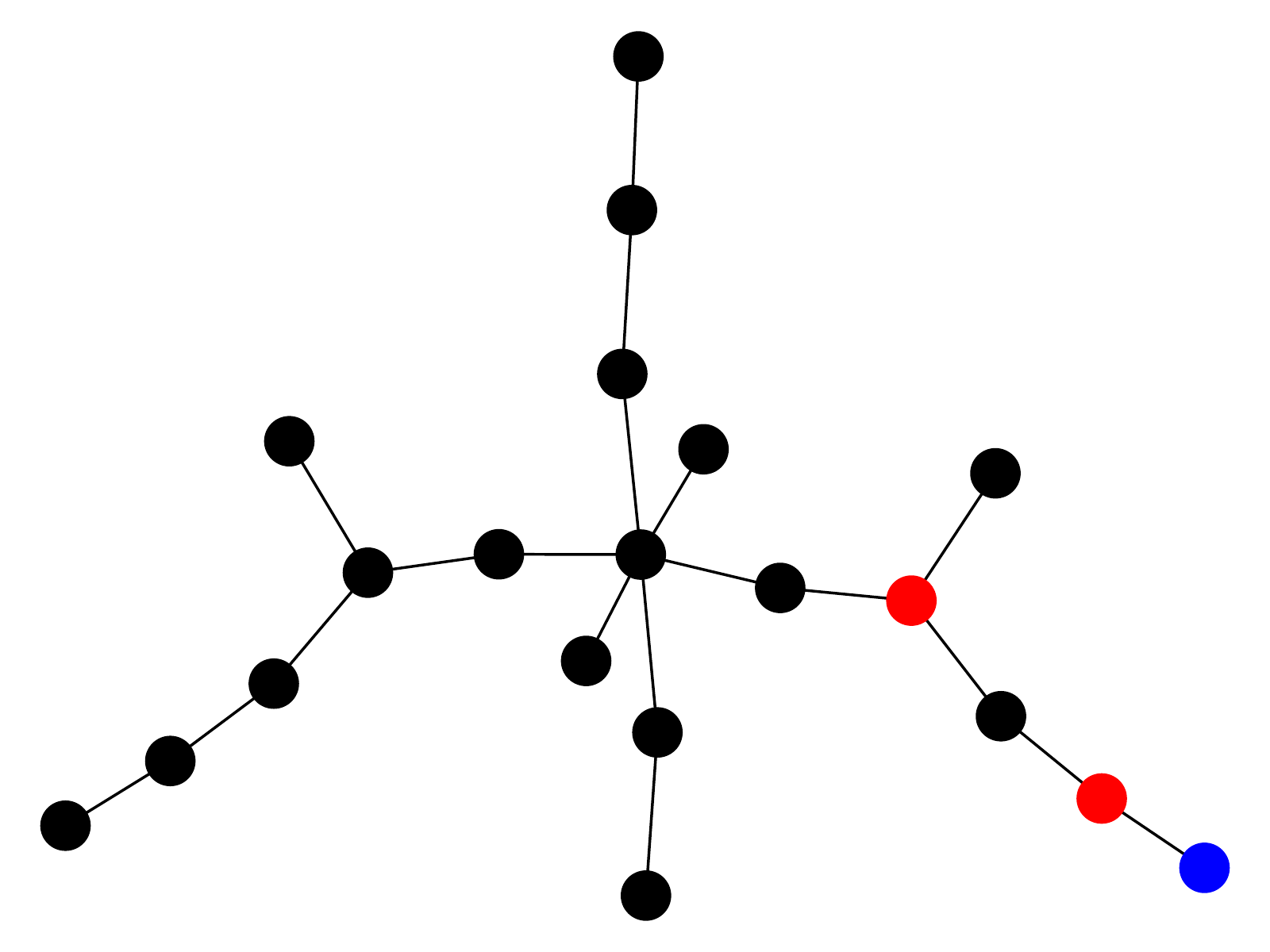} &
        \includegraphics[width=0.2\textwidth]{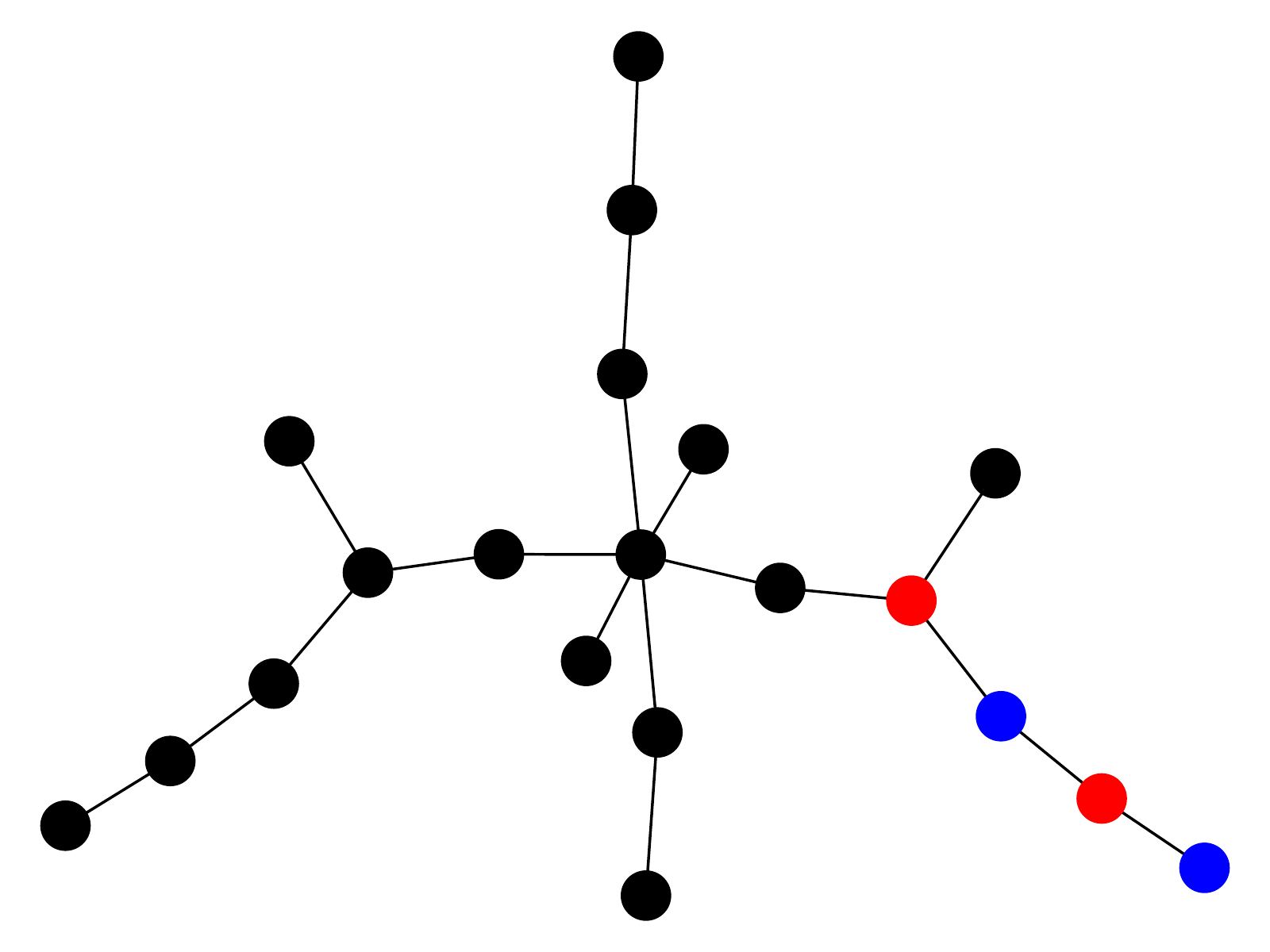} \\
        \includegraphics[width=0.2\textwidth]{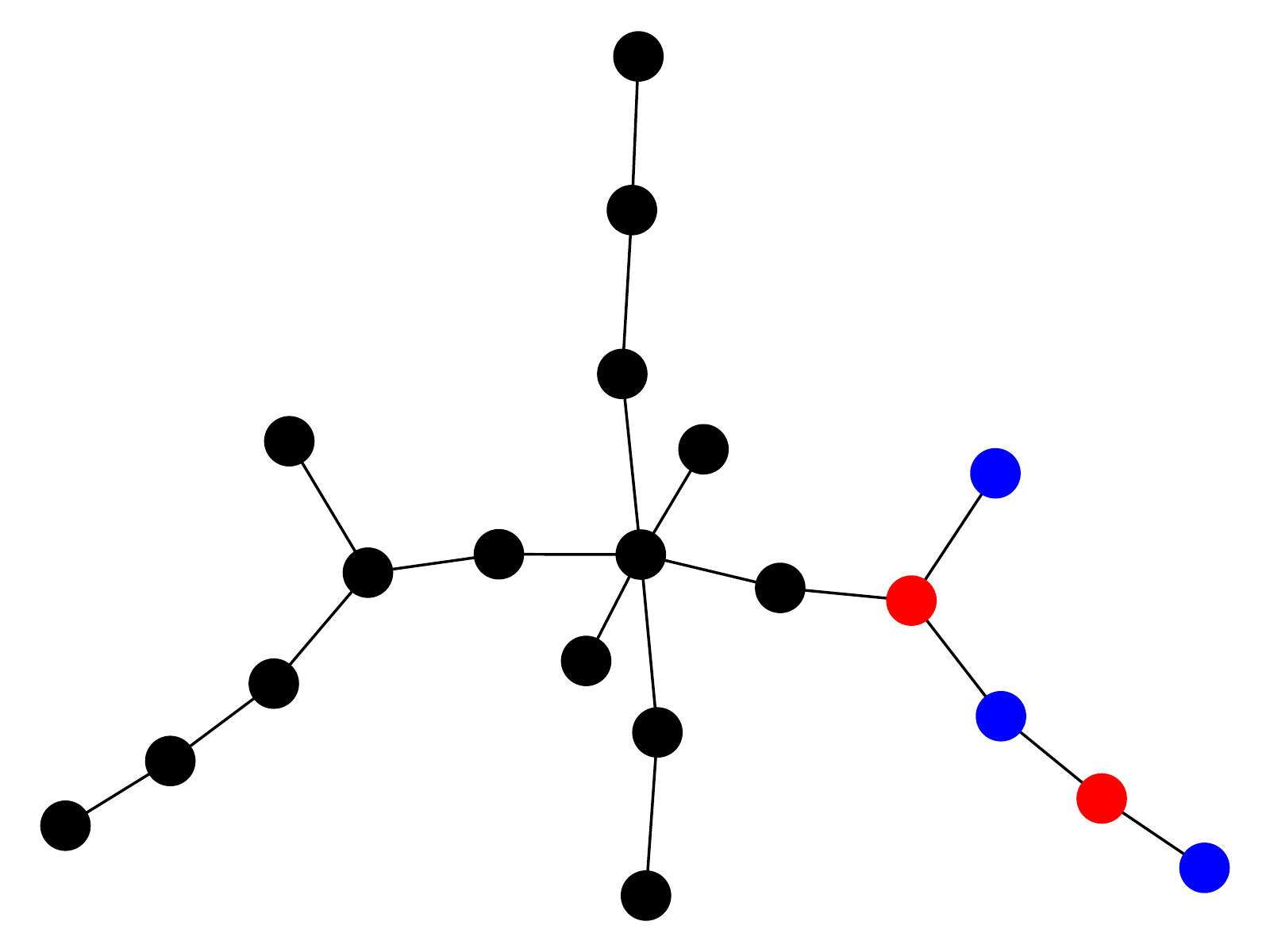} &
        \includegraphics[width=0.2\textwidth]{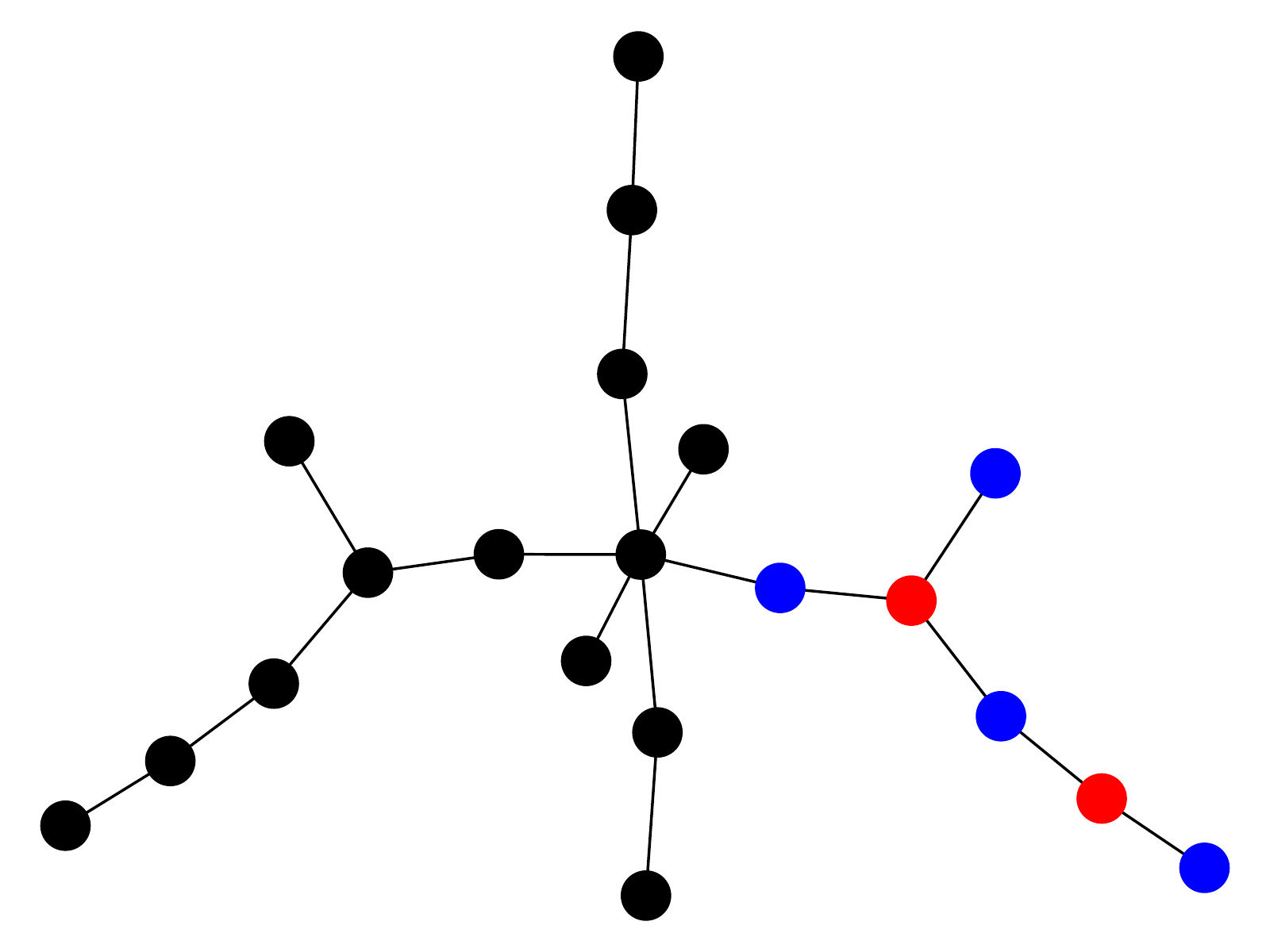} &
        \includegraphics[width=0.2\textwidth]{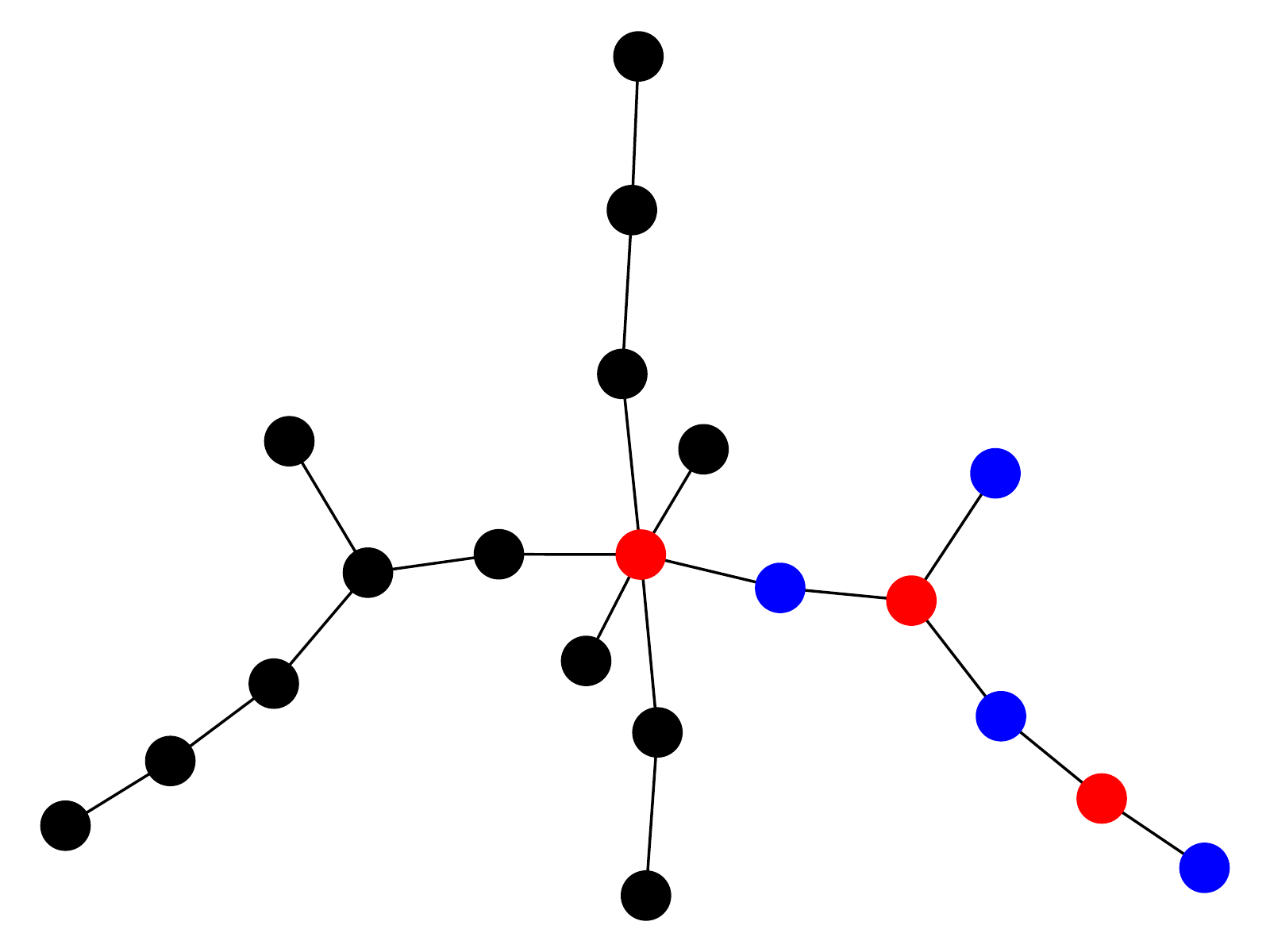} &
        \includegraphics[width=0.2\textwidth]{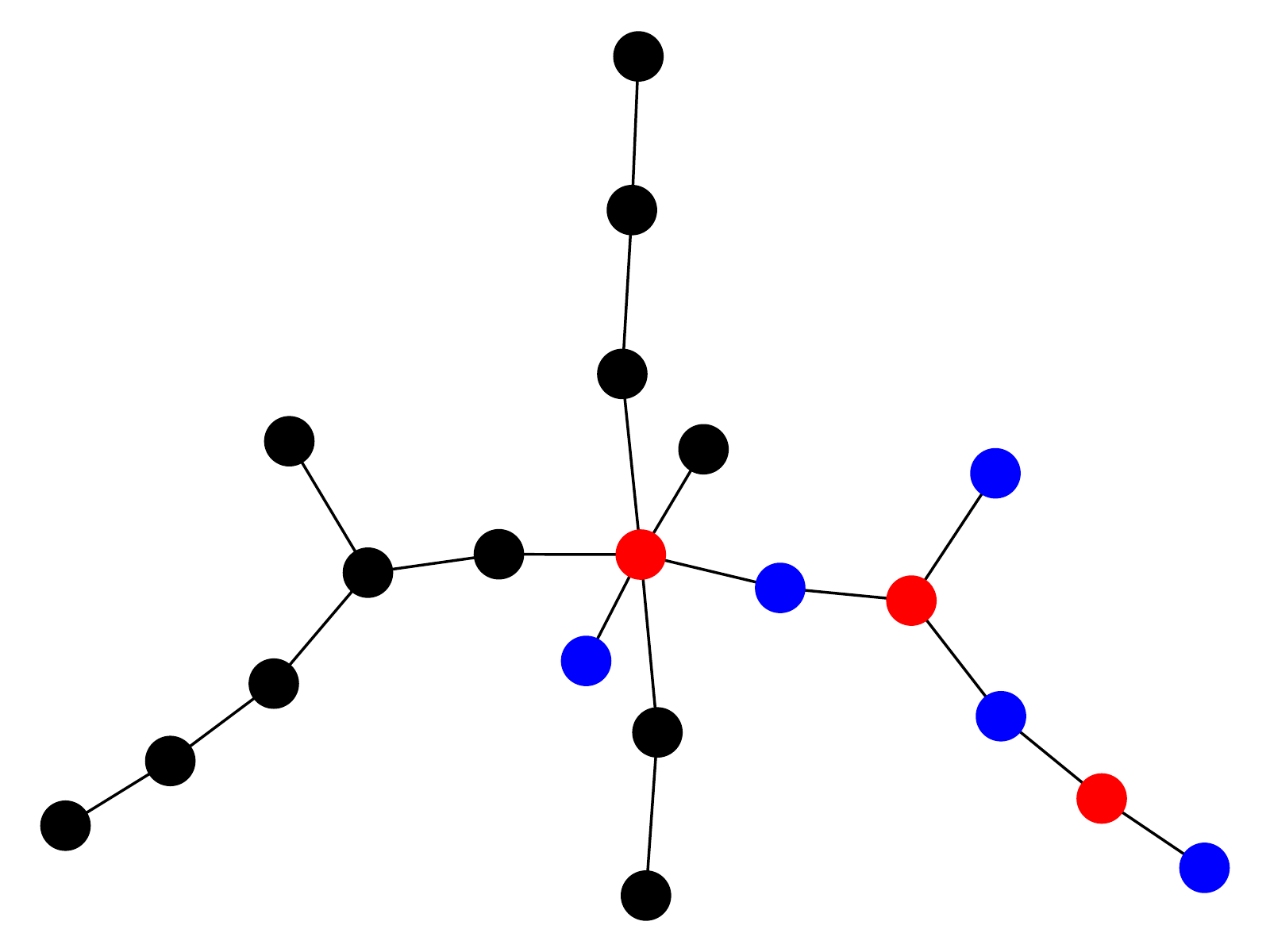} \\
        \includegraphics[width=0.2\textwidth]{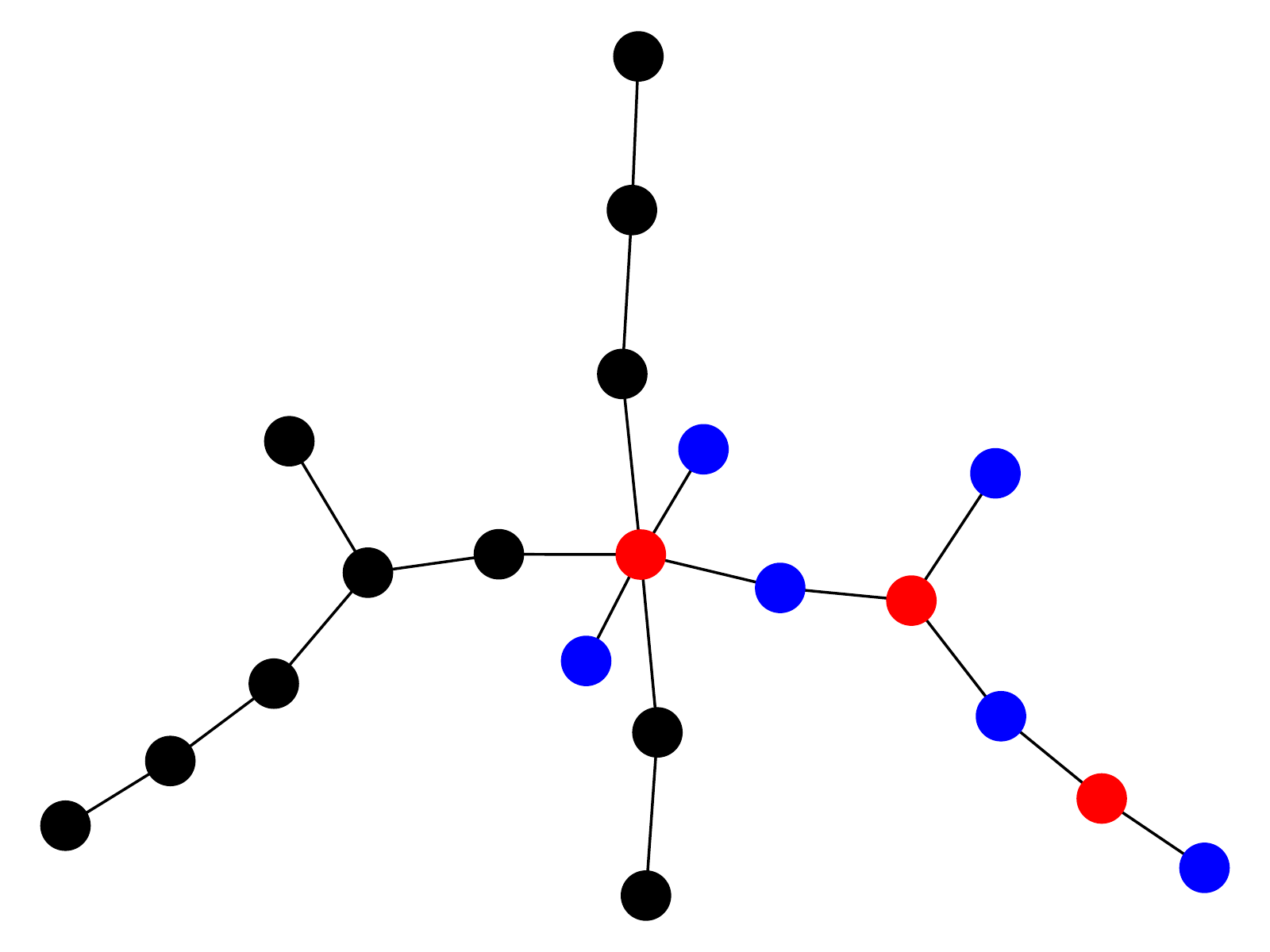} &
        \includegraphics[width=0.2\textwidth]{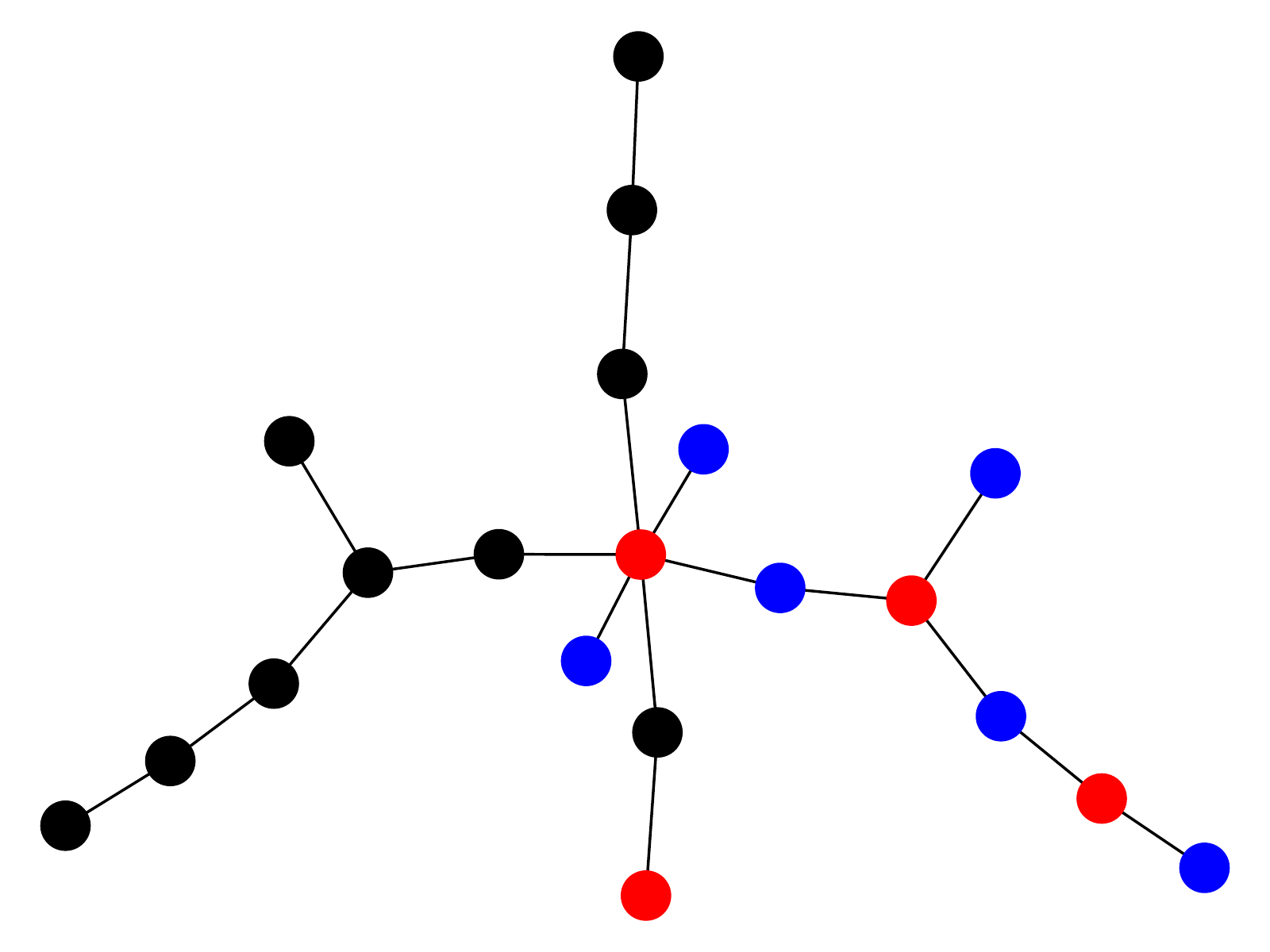} &
        \includegraphics[width=0.2\textwidth]{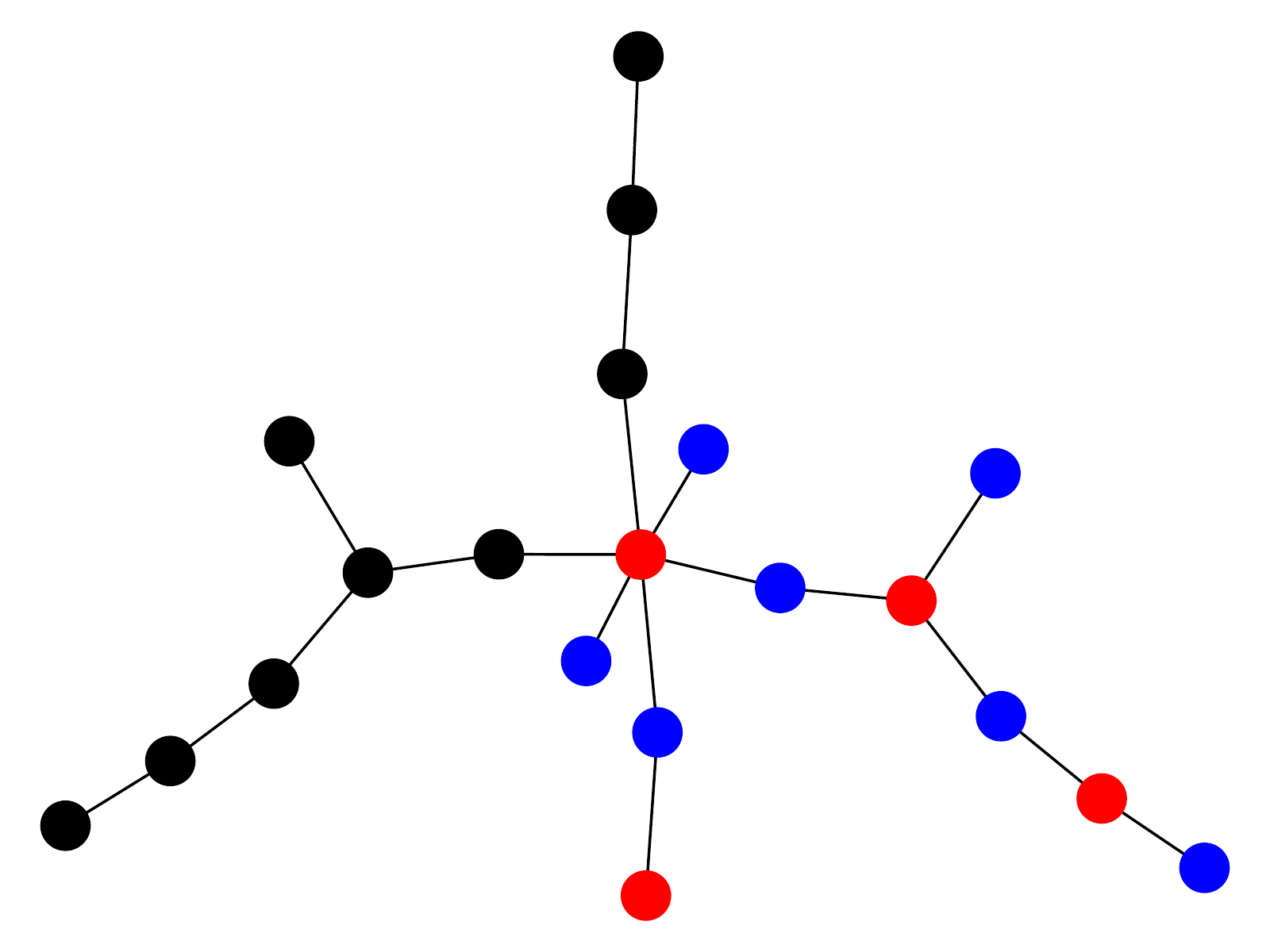} &
        \includegraphics[width=0.2\textwidth]{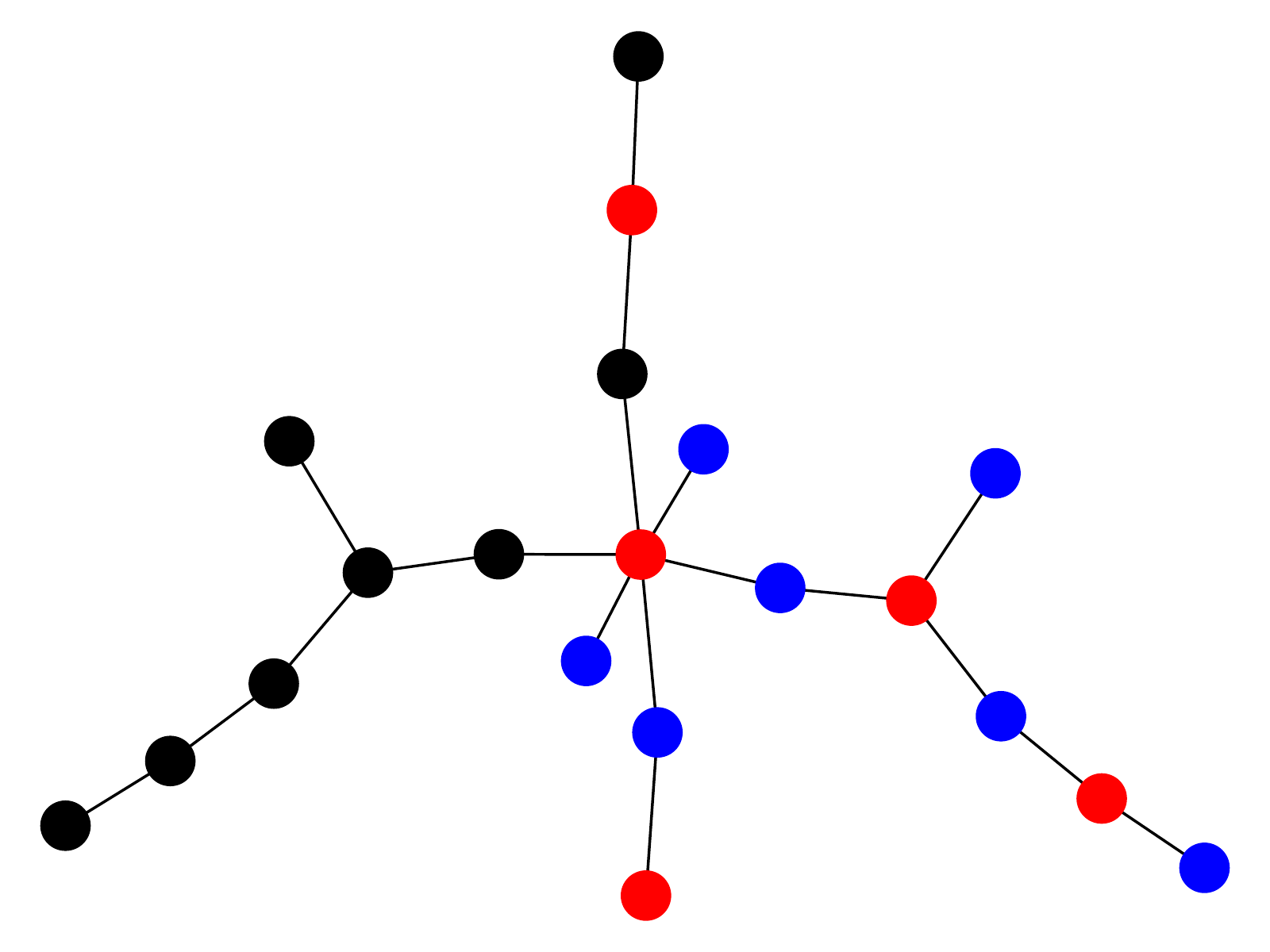} \\
        \includegraphics[width=0.2\textwidth]{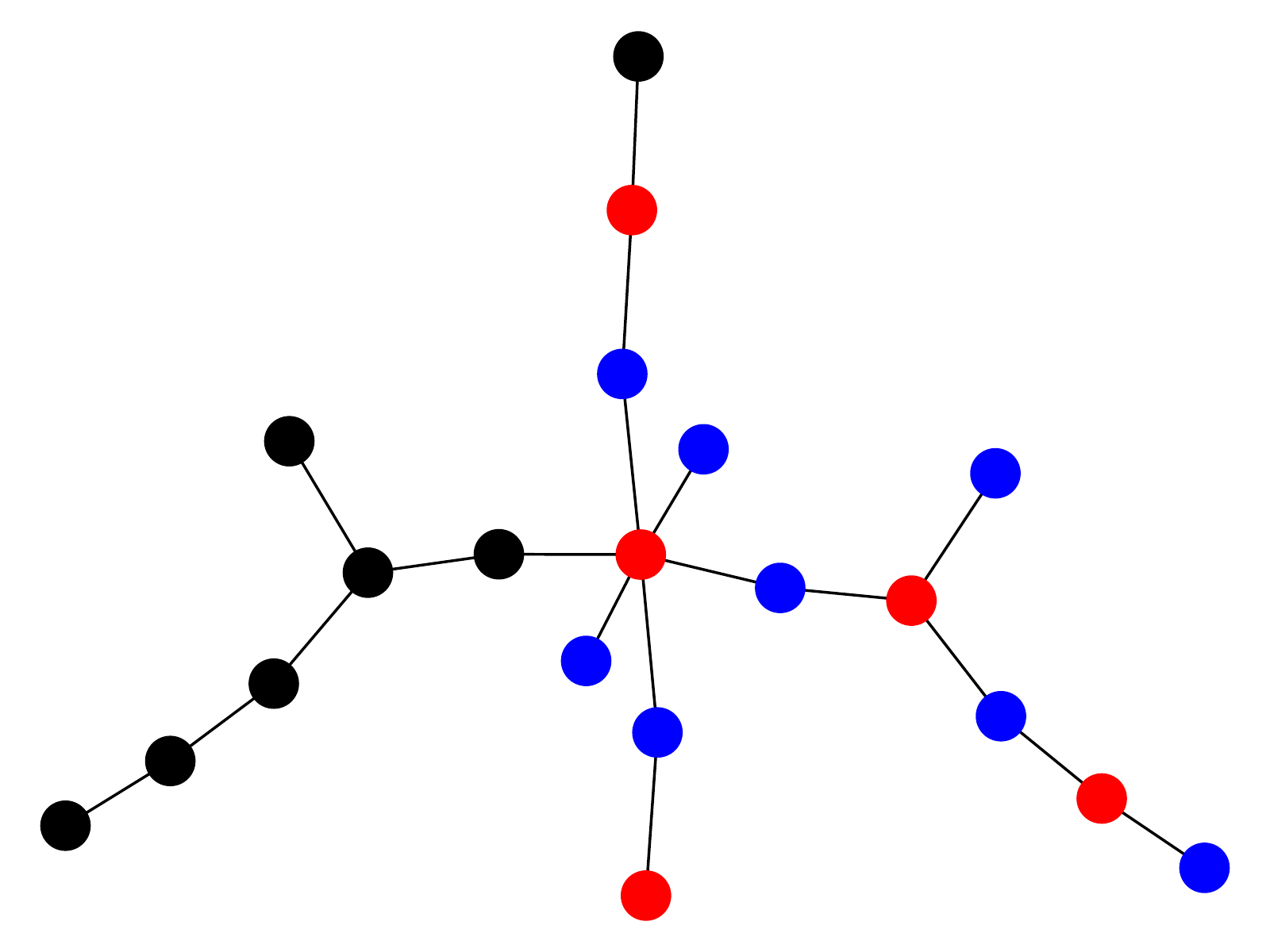} &
        \includegraphics[width=0.2\textwidth]{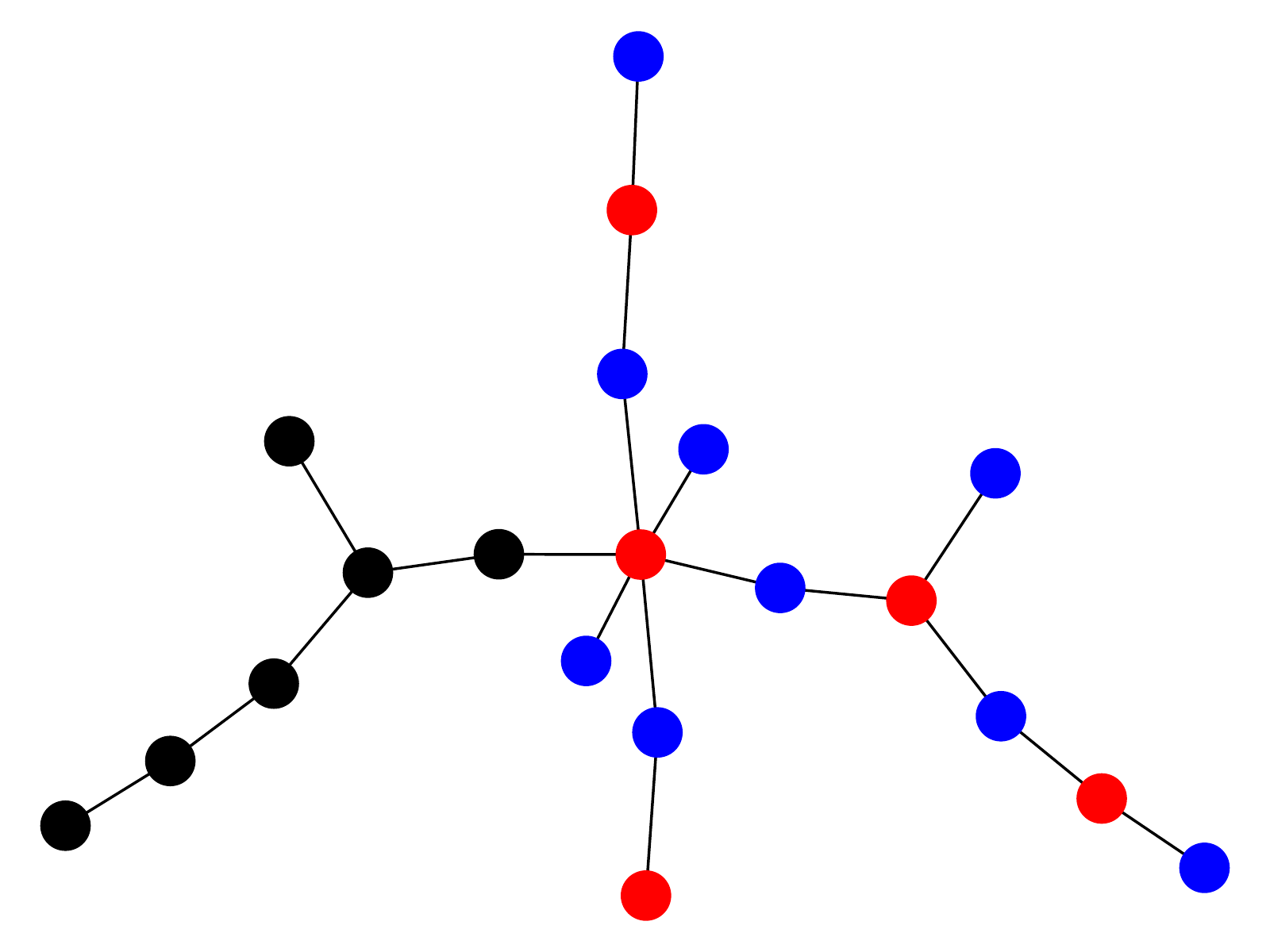} &
        \includegraphics[width=0.2\textwidth]{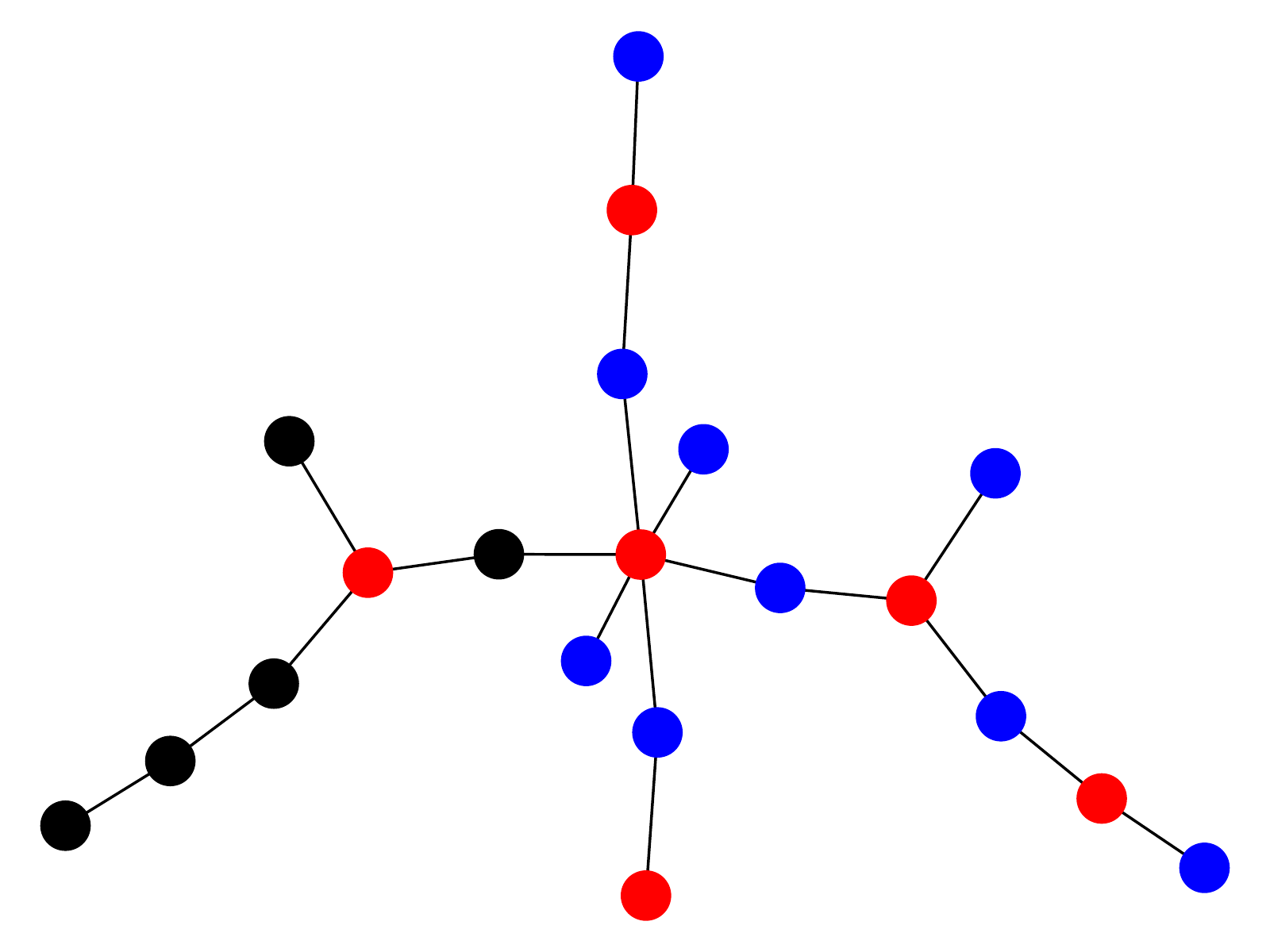} &
        \includegraphics[width=0.2\textwidth]{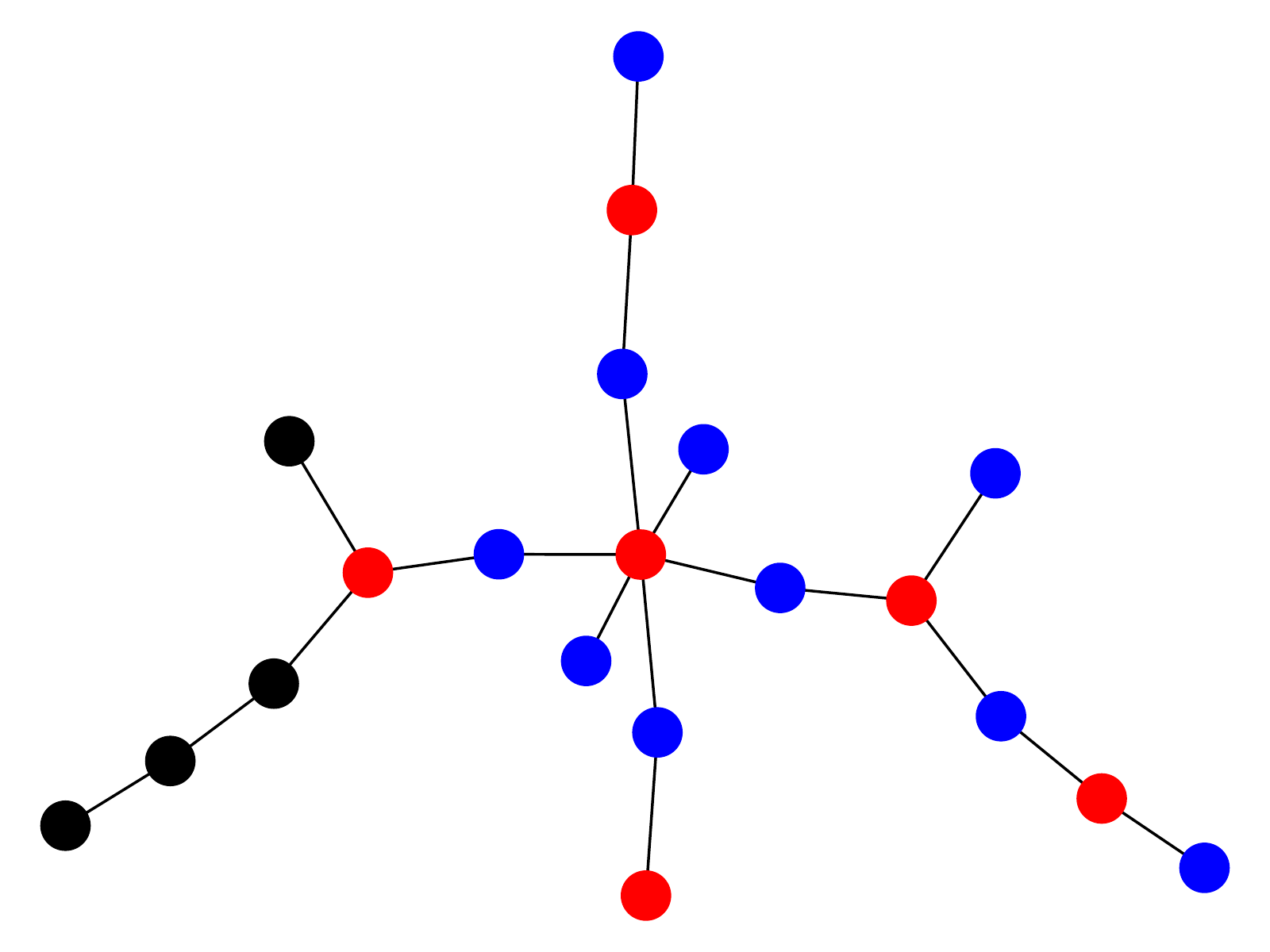} \\
        \includegraphics[width=0.2\textwidth]{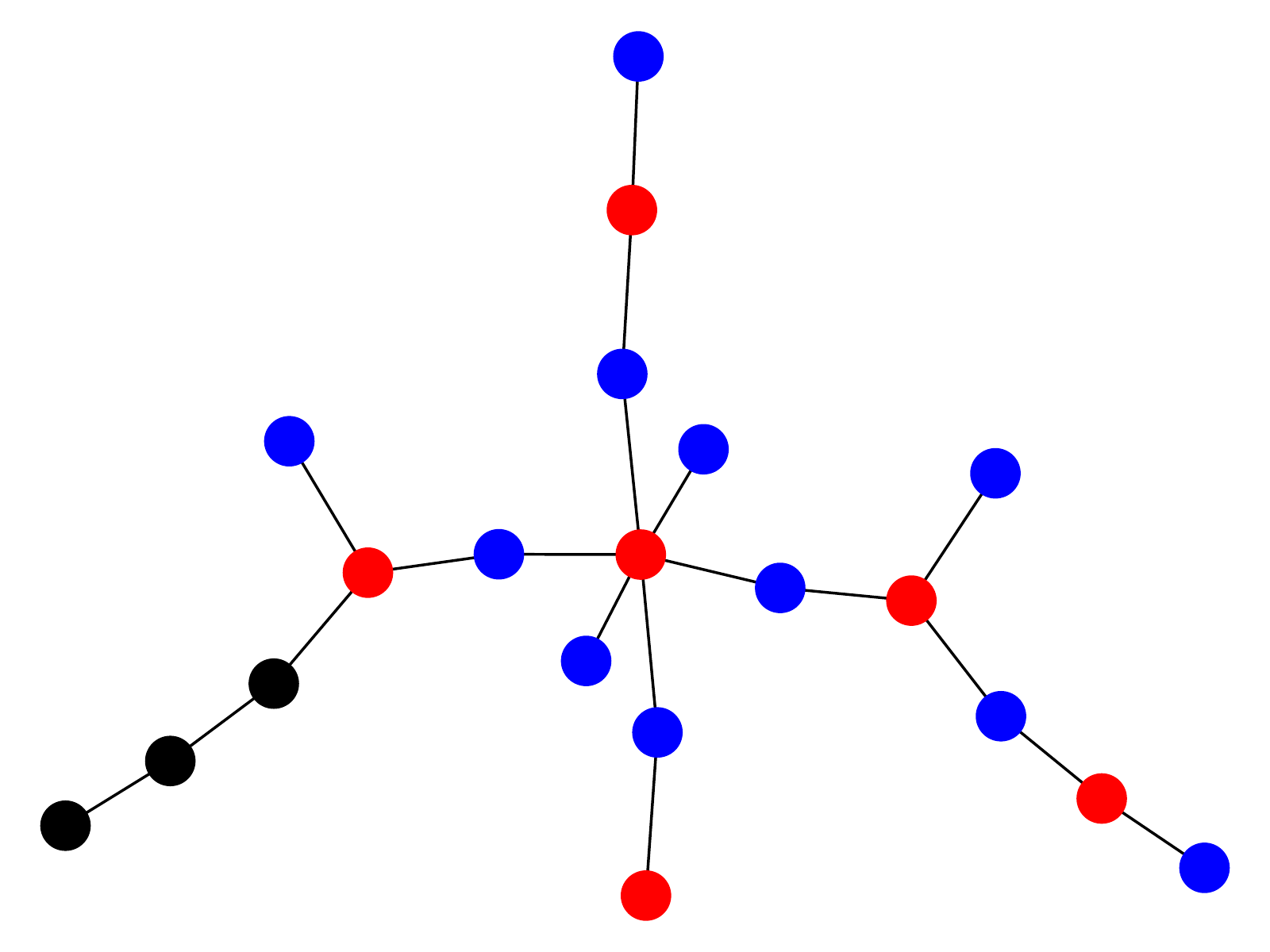} &
        \includegraphics[width=0.2\textwidth]{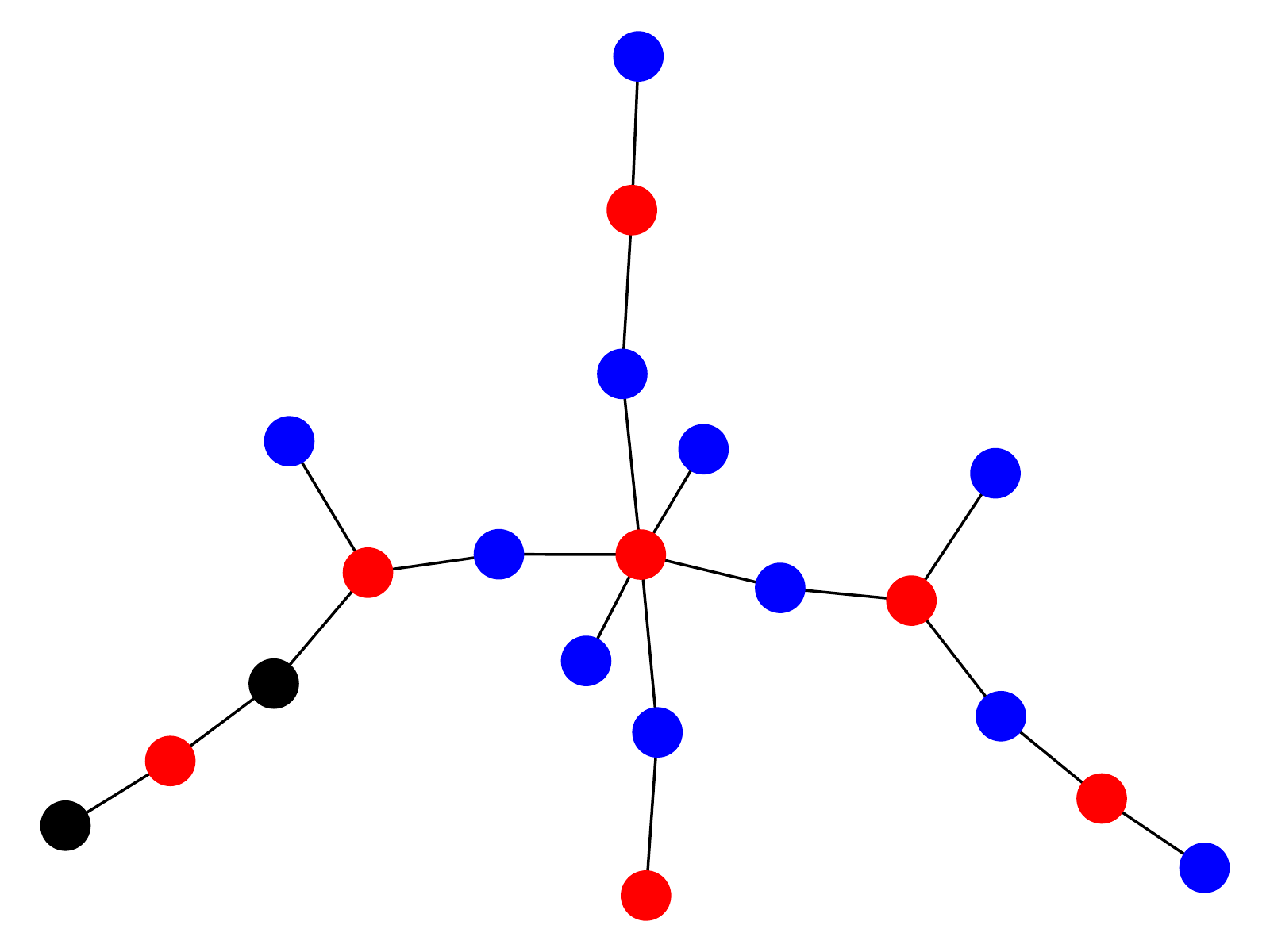} &
        \includegraphics[width=0.2\textwidth]{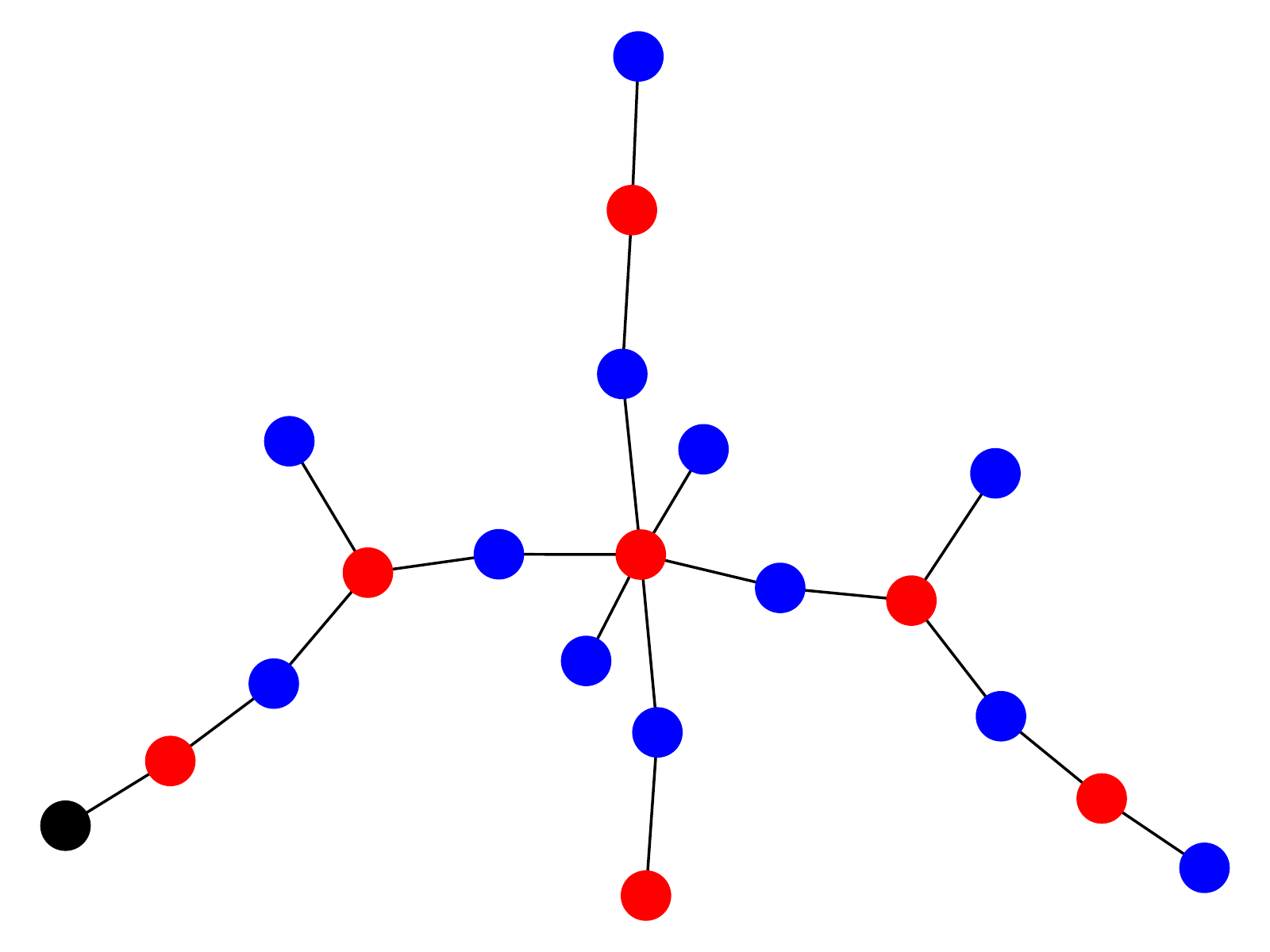} &
        \includegraphics[width=0.2\textwidth]{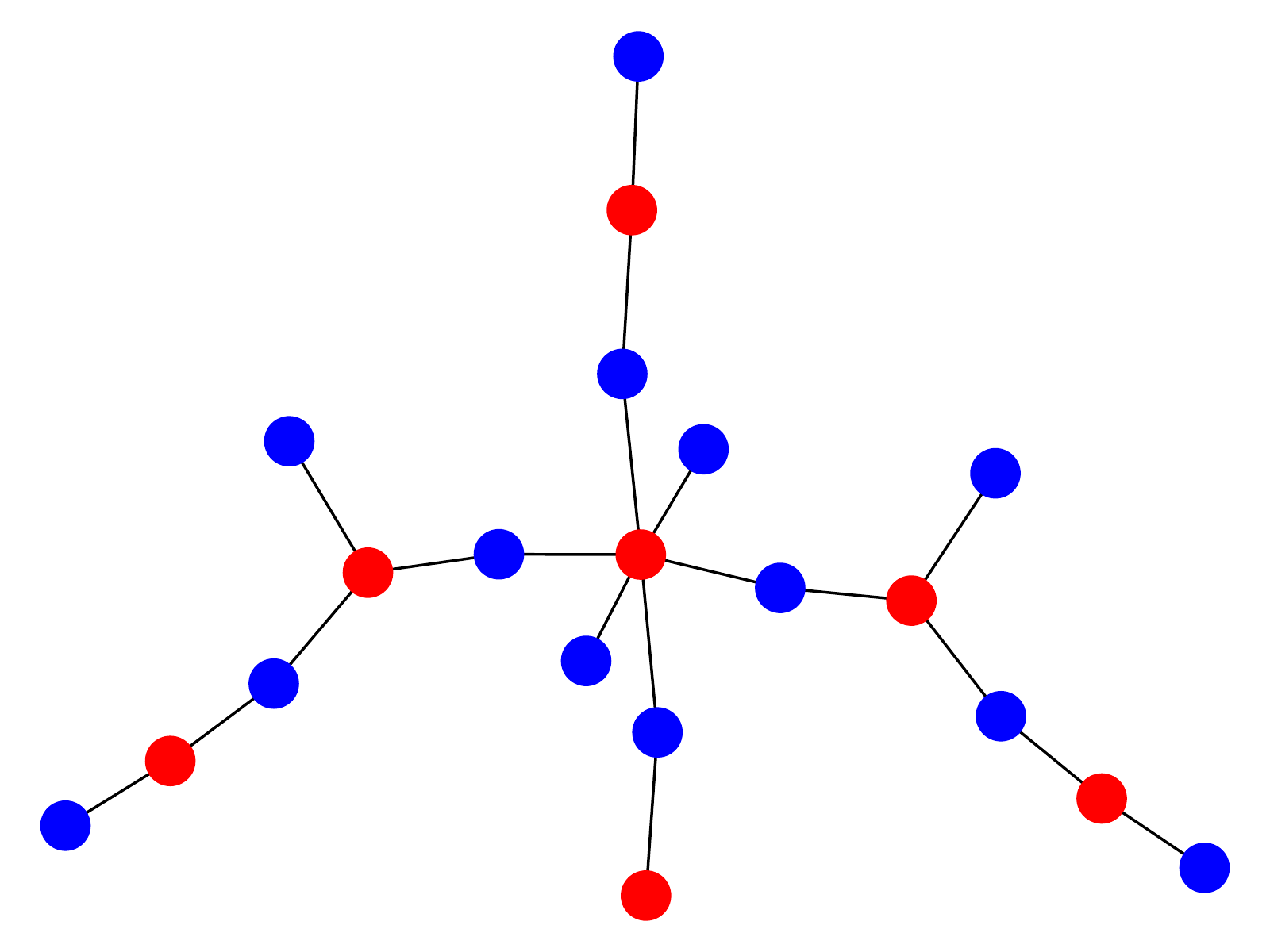} \\
    \end{tabular}
    \caption{\textbf{Order of actions of CombOpt Zero for {\sc MaxCut} on a tree.} It successfully found an optimal solution. The order of actions is similar to the order of visiting nodes in the depth-first-search.}
    \label{fig:maxcut-tree-vis}
\end{figure*}

\end{document}